\documentclass{article}

\usepackage[preprint]{neurips_2022}

\usepackage[utf8]{inputenc} 
\usepackage[T1]{fontenc}    
\usepackage{hyperref}       
\hypersetup{colorlinks,allcolors=black}
\usepackage{url}            
\usepackage{booktabs}       
\usepackage{amsfonts}       
\usepackage{nicefrac}       
\usepackage{microtype}      
\usepackage{xcolor}         

\usepackage{wrapfig}
\usepackage[american]{babel}
\usepackage{xcolor}
\usepackage{graphicx}
\usepackage{epstopdf}
\usepackage{comment}
\usepackage{amssymb}
\usepackage{algorithm}
\usepackage{algorithmic}
\usepackage{amsthm}
\usepackage[shortlabels]{enumitem}
\usepackage{bm}
\usepackage{mathtools} 
\DeclareMathOperator*{\argmin}{argmin} 
\DeclareMathOperator*{\argmax}{argmax} 
\newtheorem{definition}{Definition}
\newtheorem*{definition*}{Definition}
\newtheorem{lemma}{Lemma}
\newtheorem*{lemma*}{Lemma}
\newtheorem{remark}{Remark}
\newtheorem*{remark*}{Remark}

\newtheorem*{corollary*}{Corollary}
\newtheorem{theorem}{Theorem}
\newtheorem*{theorem*}{Theorem}
\newtheorem{prop}{Proposition}
\newtheorem*{prop*}{Proposition}
\newtheorem{assumption}{Assumption}
\DeclareMathOperator{\cA}{\mathcal{A}}
\DeclareMathOperator{\cS}{\mathcal{S}}
\DeclareMathOperator{\cN}{\mathcal{N}}
\DeclareMathOperator{\cF}{\mathcal{F}}
\DeclareMathOperator{\cI}{\mathcal{I}}
\DeclareMathOperator{\bx}{\mathbf{x}}
\DeclareMathOperator{\bg}{\mathbf{g}}
\DeclareMathOperator{\bW}{\mathbf{W}}
\DeclareMathOperator{\bH}{\mathbf{H}}
\DeclareMathOperator{\bw}{\mathbf{w}}
\DeclareMathOperator{\b0}{\mathbf{0}}
\DeclareMathOperator{\bI}{\mathbf{I}}
\DeclareMathOperator{\bZ}{\mathbf{Z}}

\DeclareMathOperator{\bh}{\mathbf{h}}
\DeclareMathOperator{\bb}{\mathbf{b}}
\DeclareMathOperator{\bG}{\mathbf{G}}

\DeclareMathOperator{\bU}{\mathbf{U}}
\DeclareMathOperator{\btheta}{\bm{\theta}}
\newcommand{\Eeleven}{E_{\ref{lemma:error between widetilde_theta_dagger and theta_0},5}t^{\frac{5}{3}}m^{-\frac{2}{3}}(\log{t})^{\frac{2}{3}}(\log{m})^{\frac{1}{2}}\lambda^{-\frac{7}{3}}L^{\frac{7}{2}} + E^J_{\ref{lemma:error between theta_tilde_dagger and theta},2} E_{\ref{lemma:error between theta_tilde_dagger and theta},3}\sqrt{\frac{t\log{t}}{m\lambda^2}}}
\newcommand{\hungyh}[1]{\textcolor{black}{#1}}

\title{Reward-Biased Maximum Likelihood Estimation for Neural Contextual Bandits}

\author{%
  Yu-Heng Hung \\
  Department of Computer Science\\
  National Yang Ming Chiao Tung University, Hsinchu, Taiwan \\
  \texttt{hungyh.cs08@nycu.edu.tw} \\
  \And
  Ping-Chun Hsieh \\
  Department of Computer Science\\
  National Yang Ming Chiao Tung University, Hsinchu, Taiwan \\
  \texttt{pinghsieh@nctu.edu.tw} \\
}

\begin{document}

\maketitle
\begin{abstract}
Reward-biased maximum likelihood estimation (RBMLE) is a classic principle in the adaptive control literature for tackling explore-exploit trade-offs. This paper studies the stochastic contextual bandit problem with general bounded reward functions and proposes NeuralRBMLE, which adapts the RBMLE principle by adding a bias term to the log-likelihood to enforce exploration. NeuralRBMLE leverages the representation power of neural networks and directly encodes exploratory behavior in the parameter space, without constructing confidence intervals of the estimated rewards. We propose two variants of NeuralRBMLE algorithms: The first variant directly obtains the RBMLE estimator by gradient ascent, and the second variant simplifies RBMLE to a simple index policy through an approximation. We show that both algorithms achieve $\widetilde{\mathcal{O}}(\sqrt{T})$ regret. Through extensive experiments, we demonstrate that the NeuralRBMLE algorithms achieve comparable or better empirical regrets than the state-of-the-art methods on real-world datasets with non-linear reward functions.
\end{abstract}

\section{Introduction}
\label{section:intro}

Efficient exploration has been a fundamental challenge in sequential decision-making in unknown environments.
As a classic principle originally proposed in the stochastic adaptive control literature for solving unknown Markov decision processes (MDPs), \textit{Reward-Biased Maximum Likelihood Estimation (RBMLE)} learns an optimal policy by alternating between estimating the unknown model parameters in an ``exploratory'' manner and applying the optimal control law based on the estimated parameters \citep{kumar1982new,borkar1990kumar,campi1998adaptive,prandini2000adaptive}. 
Specifically, to resolve the inherent issue of insufficient exploration of maximum likelihood estimation (MLE), RBMLE enforces exploration by incorporating into the likelihood function a bias term in favor of those model parameters that correspond to higher long-term average rewards.
This generic exploration scheme has been shown to asymptotically attain the optimal long-term average reward \citep{kumar1982new}. 

\hungyh{Recently, the RBMLE principle has been adapted to
optimize the regrets in stochastic bandit problems, including the classic non-contextual multi-armed bandit problems \cite{liu2020exploration}, the contextual bandit problems with generalized linear reward functions \cite{hung2020reward} and model-based reinforcement learning for finite MDPs \cite{mete2021reward}.}
Moreover, RBMLE has been shown to achieve order-optimal finite-time regret bounds and competitive empirical regret performance in the above settings.
Despite the recent progress, the existing RBMLE bandit algorithms, as well as their regret guarantees, rely heavily on the structural assumptions, such as the absence of contextual information in \cite{liu2020exploration} and linear realizability in \cite{hung2020reward}, and hence are not readily applicable to various real-world bandit applications with more complex reward structures, such as recommender systems and clinical trials.

Motivated by the competitive performance of the RBMLE principle in the bandit problems mentioned above, this paper takes one step further to study RBMLE in contextual bandits with general reward functions.
To unleash the full potential of RBMLE in contextual bandits, we leverage the representation power of neural networks to approximate the unknown reward function, as recently adopted by \cite{zhou2020neural,zhang2020neural}.
{Specifically, neural tangent kernel (NTK) \citep{jacot2018neural} is a technique to describe that a wide neural network is approximately linear in its parameters and has been adopted in several neural bandit algorithms \citep{zhou2020neural,zhang2020neural}.}
\hungyh{
While it appears natural to leverage the NTK technique in any bandit algorithm, RBMLE presents its own salient challenges to be tackled:
\begin{itemize}[leftmargin=*]
    \item \hungyh{\textit{How to handle the model misspecification issue?} -- RBMLE is a model-based approach that leverages the machinery of MLE. The use of MLE requires the knowledge of the likelihood function of the stochastic rewards, whose distributions are commonly unknown in practical bandit problems. Thus, RBMLE needs to rely on a \textit{surrogate} likelihood function that may not match the true likelihood. This naturally leads to a \textit{model misspecification} issue, which could have a significant effect on the regret.}
    \item \textit{How to address the compatibility of RBMLE with NTK?} -- Unlike the existing neural bandit algorithms (e.g., NeuralUCB) that leverage the least square estimation to learn a proper model parameter, RBMLE uses maximum likelihood estimation with a reward bias to obtain the arm-specific estimators, which do not necessarily satisfy the conditions required by the NTK regime. On top of that, the choice of the reward bias reflects an inherent \textit{three-way trade-off} in RBMLE: exploration, exploitation, and ability of NTK approximation.
    \item \textit{What is the effect of inexact maximizers of RBMLE on regret?} -- Different from RBMLE for linear bandits \cite{hung2020reward}, which relies on the exact maximizer of RBMLE, RBMLE for neural bandits uses gradient descent to approximate the true RBMLE estimator during training, and the resulting approximation error should be considered and precisely characterized in the regret analysis.
\end{itemize}
Moreover, the above three issues are tightly coupled with each other, and this renders the regret analysis of RBMLE for neural bandits even more challenging.
Despite the above challenges, this paper addresses all of the above issues rigorously by extending the RBMLE principle to the neural reward function approximation and proposing the first RBMLE bandit algorithm with regret guarantees for the contextual bandit problems without the linear realizability assumption.
}
\hungyh{We highlight the main contributions as follows:
\begin{itemize}[leftmargin=*]
    \item We propose NeuralRBMLE, which extends the RBMLE principle to the neural contextual bandit problem. We first present a prototypic NeuralRBMLE algorithm that enjoys an index form by incorporating a surrogate likelihood function and a proper reward-bias term. We then propose two practical approaches, namely NeuralRBMLE-GA and NeuralRBMLE-PC, to substantiate the prototypic NeuralRBMLE algorithm.
    \item We formally establish the regret bounds for the two practical NeuralRBMLE algorithms in the NTK regime. Through regret analysis, we provide an affirmative answer to the compatibility of RBMLE with NTK. Moreover, we fully resolve the model misspecification issue and validate the flexibility in using any exponential family distribution as the surrogate likelihood function for NeuralRBMLE-GA, thereby opening up a whole new family of neural bandit algorithms. This serves as an additional degree of algorithmic freedom in achieving low empirical regret. For NeuralRBMLE-PC, we characterize the interplay between the reward-bias term, the regret, and the condition of the parameters of the neural network.
    \item We evaluate NeuralRBMLE and other benchmark algorithms through extensive simulations on various benchmark real-world datasets. The simulation results show that NeuralRBMLE achieves comparable or better empirical regret performance than the benchmark methods. RBMLE also exhibits better robustness with a smaller standard deviation of final regret compared to other benchmark methods. Moreover, {unlike NeuralUCB and NeuralTS, NeuralRBMLE-GA does not require computing the inverse of a high-dimensional matrix, which is computationally expensive.}
\end{itemize}
}

\noindent\textbf{Notations.} Throughout this paper, for any positive integer $K$, we use $[K]$ as a shorthand for the set $\{1,\cdots,K\}$.
We use $\lVert \cdot\rVert_2$ to denote the $L_2$-norm of a vector.
We use boldface fonts for vectors and matrices throughout the paper.
Moreover, we use $\b0$ and $\bI$ to denote the zero matrices and the identity matrices, respectively.
\section{Problem Formulation}
\label{section:problem}

In this section, we formally describe the neural contextual bandit problem considered in this paper.
\subsection{Contextual Bandits with General Rewards}
We consider the stochastic $K$-armed contextual bandit problem, where the total number of rounds $T$ is known\footnote{The assumption of a known horizon is mild as one could apply the standard \textit{doubling trick} to convert a horizon-dependent algorithm to an anytime one \citep{lattimore2020bandit}.}. 
At each decision time $t\in [T]$, the $K$ context vectors $\left\{ \bx_{t,a}\in\mathbb{R}^{d}\; |\; a \in [K]\right\}$, which capture the feature information about the arms, are revealed to the learner. 
Without loss of generality, we assume that $\lVert \bx_{t,a}\rVert_2 \leq 1$, for all $t\in [T]$ and for all $a\in [K]$.
Given the contexts, the learner selects an arm $a_t\in [K]$ and obtains the corresponding random reward $r_{t,a_t}$.
For ease of notation, we define 
(i) $\bx_t := \bx_{t, a_t}$, 
(ii) $r_t := r_{t,a_t}$,
(iii) $\mathcal{F}_t := (\bx_1,a_1,r_1,\dots,\bx_t)$ as the observation history up to the beginning of time $t$, 
(iv) $a_t^{*}:=\argmax_{a\in [K]}E[r_{t,a}\rvert \mathcal{F}_t]$, and 
(v) $r^*_t := r_{t,a^*_t}$.
The goal of the learner is to minimize the \emph{pseudo regret} as
\begin{align}
    \mathcal{R}(T) := \mathbb{E}\bigg[ \sum^T_{t=1} (r^*_t - r_{t})\bigg] \label{def:regret}
\end{align}
In the neural contextual bandit problem, the random reward at each time $t$ takes the form of $r_t = h(\bx_t) + \epsilon_t$, where $h:\mathbb{R}^d\rightarrow [0,1]$ is an unknown reward function and $\epsilon_t$ is a $\nu$-sub-Gaussian noise conditionally independent of all the other rewards in the past given the context $\bx_t$ and satisfying $\mathbb{E}[\epsilon_t|\bx_t] = 0$, and the reward function $h(\cdot)$ is approximated by a neural network through training.
Compared to the generalized linear bandit model, here we assume no special structure (e.g., linearity or convexity) on the reward function, except that the reward signal is bounded in $[0,1]$.

\subsection{Neural Function Approximation for Rewards}
\label{section:nn}
In this paper, we leverage a neural network to approximate the reward function $h(\cdot)$. 
Let $L \geq 2$ be the depth of this neural network, $\sigma(\cdot) = \max\{\cdot,0\}$ be the Rectified Linear Unit (ReLU) activation function, and $m_l$ be the width of the $l$-th hidden layer, for $l\in [L-1]$. 
We also let $m_0=d$ and $m_{L}=1$.
Let $\bW_l \in \mathbb{R}^{m_{l}\times m_{l-1}}$ denote the weight matrix of the $l$-th layer of the neural network, for $l\in [L]$.
For ease of exposition, we focus on the case where $m_{l}=m$, for all $l\in [L-1]$.
For ease of notation, we define $\btheta := [\text{vec}(\bW_1)^\top,\dots, \text{vec}(\bW_L)^\top] \in {\mathbb{R}}^{p}$, where $p= m+md+m^2(L-1)$ denotes the total number of parameters of the neural network.
Let $f(\bx;\btheta)$ denote the output of the neural network with parameters $\btheta$ and input $\bx$, i.e.,
\begin{equation}
    f(\bx;\btheta) := \sqrt{m} \cdot \bW_L\sigma(\bW_{L-1}\sigma(\bW_{L-2}\dots \sigma(\bW_1 \bx))) \label{def:nn}.
\end{equation}
Let $\bg(\bx;\btheta) := \nabla_{\btheta} f(\bx;\btheta)$ be the gradient of $f(\bx;\btheta)$, and let $\btheta_0$ be the initial model parameters selected by the following random initialization steps\footnote{The initialization steps are the same as those in \cite{cao2019generalization,zhou2020neural}.}: 
(i) For $l \in [L-1]$, let $\bW_{l}$ take the form of $\bW_l = \begin{pmatrix}
        \bW & {\b0}\\
        {\b0} & \bW
    \end{pmatrix}$, where each entry of $\bW$ is drawn independently from $\cN(0,2/m)$.
(ii) For the output layer, let $\bW_{L}$ take the form of $\bW_L = (\bw^\top, -\bw^\top)$, where each entry of $\bw$ is drawn independently from $\cN(0,1/m)$. 

\section{An Overview of the RBMLE Principle}
\label{section:intro:RBMLE}
In this section, we review the generic RBMLE principle in the context of adaptive control for maximizing the long-term average reward of an unknown dynamic system. 
Consider a discrete-time MDP with a state space $\cS$, an action space $\cA$, and unknown transition dynamics as well as a reward function that are both dependent on the unknown true parameter $\btheta_{*}$ belonging to some known set $\Theta$.
For ease of notation, for any $\btheta\in\Theta$\footnote{To make the connection between RBMLE and neural bandits explicit, in Section 3 we slightly abuse the notation $\btheta$ to denote the parameters of the dynamical system.}, we denote the transition probabilities under $\btheta$ by $p(s_t,s_{t+1},a_t;\btheta):=\text{Prob}(s_{t+1}|s_t,a_t)$, where $p$ is a probability function parameterized by $\btheta$, $s_t \in \cS$ and $a_t \in \cA$ are the state and the action taken at time $t$. 
Let $J(\pi;\btheta)$ be the long-term average reward under policy $\pi: \cS \rightarrow \cA$. 
We let ${J}^{*}(\btheta):=\max_{\pi} J(\pi;\btheta)$ denote the optimal long-term average reward for any $\btheta\in\Theta$ and use $\pi^{*} := \argmax_{\pi} J(\pi;\btheta_{*})$ to denote an optimal policy for $\btheta_{*}$.
\begin{itemize}[leftmargin=*]
    \item \textbf{Closed-loop identification issue:} Originally proposed by \cite{mandl1974estimation}, the \textit{certainty equivalent} (CE) method addresses the optimal control of an unknown dynamic system by first finding the MLE of the true parameter and then following an optimal policy with respect to the MLE. Specifically, the MLE of the true parameter $\btheta_{*}$ at each time $t$ can be derived as
\begin{equation}
    \btheta^{\text{MLE}}_t := \argmax_{\btheta\in\Theta} \prod_{i=1}^{t-1} p(s_i,s_{i+1},a_i;\btheta).
\end{equation}
Let $\pi^{\text{MLE}}_t := \argmax_{\pi} J(\pi,\btheta^{\text{MLE}}_t)$ denote an optimal policy for the system with parameter $\btheta^{\text{MLE}}_t$.
Then, it was shown in \cite{kumar1982new} that under the sequence of policies $\{\pi^{\text{MLE}}_t\}$, the sequence of maximum likelihood estimates $\{\btheta^{\text{MLE}}_t\}$ converges to some estimate $\btheta_{\infty}^{\text{MLE}}$ in the limit such that for all pairs of $s,s' \in S$,
\begin{equation}
     p(s,s',\pi^{\text{MLE}}_\infty(s);\btheta^{\text{MLE}}_\infty) =  p(s,s',\pi^{\text{MLE}}_\infty(s);\btheta_{*}),\label{eq:closed-loop identification property of MLE}
\end{equation}
where $\pi^{\text{MLE}}_\infty:=\argmax_{\pi}J(\pi,\btheta^{\text{MLE}}_\infty)$ is an optimal policy for $\btheta^{\text{MLE}}_\infty$. 
Notably, (\ref{eq:closed-loop identification property of MLE}) is typically known as the ``closed-loop identification'' property, which  indicates that under the policy $\pi^{\text{MLE}}_{\infty}$, the transition probabilities can be correctly identified only in a ``closed-loop'' manner.
As a result, under the CE approach, it is \textit{not} guaranteed that all the transition probabilities are correctly estimated, and therefore the policy $\pi^{\text{MLE}}_\infty$ is not necessarily optimal for the true parameter $\btheta_{*}$.

\item \textbf{The inherent bias resulting from MLE:} The above key insight about the CE approach can be made more explicit \cite{kumar1982new} by
\begin{equation}
     J(\pi^{\text{MLE}}_\infty;\btheta^{\text{MLE}}_\infty) = J(\pi^{\text{MLE}}_\infty;\btheta_{*}) \leq J(\pi^{*};\btheta_{*}),\label{eq:sub-optimal reward under CE}
\end{equation}
where the first equality in (\ref{eq:sub-optimal reward under CE}) follows from (\ref{eq:closed-loop identification property of MLE}). As $J(\pi^{\text{MLE}}_\infty;\btheta^{\text{MLE}}_\infty)\equiv {J}^{*}(\btheta^{\text{MLE}}_\infty)$ and $J(\pi^{*};\btheta_{*})\equiv {J}^*(\btheta_{*})$, (\ref{eq:sub-optimal reward under CE}) indicates that the estimates under CE suffer from an inherent \textit{bias} that favors the parameters with \textit{smaller} optimal long-term average rewards than $\btheta_*$. 

\item \textbf{Adding a reward-bias term for correcting the inherent bias of MLE.}
To counteract this bias, \cite{kumar1982new} proposed the RBMLE approach, which directly multiplies the likelihood by an additional reward-bias term $J^*(\btheta)^{\alpha(t)}$ with $\alpha(t) > 0, \alpha(t)\rightarrow \infty, \alpha(t)=o(t)$, with the aim of encouraging exploration over those parameters $\btheta$ with a potentially larger optimal long-term average reward. 
That is, the parameter estimate under RBMLE is
\begin{equation}
    \btheta^{\text{RBMLE}}_t := \argmax_{\btheta\in\Theta} \Big\{ {J}^*(\btheta)^{\alpha(t)}\prod_{i=1}^{t-1} p(s_i,s_{i+1},a_i;\btheta)\Big\}. \label{eq: original RBMLE}
\end{equation}
Accordingly, the policy induced by RBMLE at each $t$ is $\pi^{\text{RBMLE}}_t := \argmax_{\pi} J(\pi;\btheta^{\text{RBMLE}}_t)$.
It has been shown in \cite{kumar1982new} that RBMLE successfully corrects the inherent bias and converges to the optimal policy $\pi^*$ through the following steps: (i) Since $\alpha(t)=o(t)$, the effect of the reward-bias term becomes negligible compared to the likelihood term for large $t$.
Hence, the sublinearity of $\alpha(t)$ leads to diminishing exploration and thereby preserves the convergence property similar to that of MLE. 
As a result, both the limits $\btheta^{\text{RBMLE}}_\infty:=\lim_{t\rightarrow \infty}\btheta^{\text{RBMLE}}_t$ and $\pi^{\text{RBMLE}}_\infty:=\lim_{t\rightarrow \infty}\pi^{\text{RBMLE}}_t$ exist, and the result similar to (\ref{eq:sub-optimal reward under CE}) still holds under RBMLE, i.e.,
\begin{equation}
    J(\pi^{\text{RBMLE}}_\infty;\btheta^{\text{RBMLE}}_\infty) \leq J(\pi^{*};\btheta_{*}). \label{eq:property of RBMLE-1}
\end{equation}
(ii) Given that $\alpha(t) \rightarrow \infty$, the reward-bias term $J^*(\btheta)^{\alpha(t)}$, which favors those parameters with higher rewards, remains large enough to undo the inherent bias of MLE.
As a result, RBMLE achieves
\begin{equation}
    J(\pi^{\text{RBMLE}}_\infty,\btheta^{\text{RBMLE}}_\infty) \geq J(\pi^{*},\btheta^{*}), \label{eq:property of RBMLE-2}
\end{equation}
as proved in \citep[Lemma 4]{kumar1982new}.
By (\ref{eq:property of RBMLE-1})-(\ref{eq:property of RBMLE-2}), we know that the delicate choice of $\alpha(t)$ ensures $\pi^{\text{RBMLE}}_\infty$ is an optimal policy for the true parameter $\btheta_{*}$. 
\end{itemize}
Note that the above optimality result implies that RBMLE achieves a sublinear regret, but without any further characterization of the regret bound and the effect of the bias term $\alpha(t)$.
In this paper, we adapt the RBMLE principle to neural contextual bandits and design bandit algorithms with regret guarantees.

\section{RBMLE for Neural Contextual Bandits}
\label{section:rbmle}
In this section, we present how to adapt the generic RBMLE principle described in Section \ref{section:intro:RBMLE} to the neural bandit problem and propose the NeuralRBMLE algorithms.
\subsection{A Prototypic NeuralRBMLE Algorithm}

By leveraging the RBMLE principle in (\ref{eq: original RBMLE}), we propose to adapt the parameter estimation procedure for MDPs in (\ref{eq: original RBMLE}) to neural bandits through the following modifications:
\begin{itemize}[leftmargin=*]
    \item \textbf{Likelihood functions via surrogate distributions:} For each time $t\in [T]$, let $\ell^{\dagger}(\cF_t;\btheta)$ denote the log-likelihood of the observation history $\cF_t$ under a neural network parameter $\btheta$. Notably, different from the likelihood of state transitions in the original RBMLE in (\ref{eq: original RBMLE}), here $\ell^{\dagger}(\cF_t;\btheta)$ is meant to capture the statistical behavior of the received rewards given the contexts.
    However, one main challenge is that the underlying true reward distributions may not have a simple parametric form and are unknown to the learner.
    To address this challenge, we use the log-likelihood of \textit{canonical exponential family distributions} as a surrogate for the true log-likelihood. 
    Specifically, the surrogate log-likelihood\footnote{For brevity, we ignore the normalization function of the canonical exponential families since this term depends only on $r_t$ and is independent of $\btheta$.} is chosen as $\log p(r_s\rvert \bx_s;\btheta)=r_s f(\bx_{s};\btheta)-b(f(\bx_s;\btheta))$, where $b(\cdot):\mathbb{R}\rightarrow \mathbb{R}$ is a known strongly convex and smooth function with $L_b \leq b''(z) \leq U_b$, for all $z\in \mathbb{R}$.
    Note that the above $\log p(r_s\rvert \bx_s;\btheta)$ is used only for arm selection under NeuralRBMLE, and we do not impose any distributional assumption on the rewards other than sub-Gaussianity.
    Hence, $\ell^{\dagger}(\cF_t;\btheta)$ can be written as
    \begin{equation}
        \ell^{\dagger}(\cF_t;\btheta) := \sum_{s=1}^{t-1} \big(r_s f(\bx_{s};\btheta)-b(f(\bx_s;\btheta))\big).\label{eq:exponential family log-likelihood}
    \end{equation}
    For example, one candidate choice for (\ref{eq:exponential family log-likelihood}) is using Gaussian likelihood with $b(z)=z^2/2$.
    \vspace{0mm}
    \item \textbf{Reward-bias term:} We consider one natural choice $(\max_{a\in [K]}f(\bx_{t,a};\btheta))^{\alpha(t)}$, which provides a bias in favor of the parameters $\btheta$ with larger estimated rewards.
    \vspace{0mm}
    \item \textbf{Regularization:} As the RBMLE procedure requires a maximization step, we also incorporate into the reward-biased likelihood a quadratic regularization term $\frac{m \lambda}{2}\lVert \btheta - \btheta_0\rVert_2^2$ with $\lambda>0$, as typically done in training neural networks. For ease of notation, we define
\begin{equation}
    \ell_{\lambda}^{\dagger}(\cF_t;\btheta):=\ell^{\dagger}(\cF_t;\btheta)-\frac{m \lambda}{2}\lVert \btheta - \btheta_0 \rVert^2_2.\label{eq:regularization}
\end{equation}
\end{itemize}

Based on the RBMLE principle in (\ref{eq: original RBMLE}) and the design in (\ref{eq:exponential family log-likelihood})-(\ref{eq:regularization}), at each time $t$, the learner under NeuralRBMLE selects an arm that maximizes $f(\bx_{t,a};\btheta_{t}^{\dagger})$, where
\begin{equation}
    \btheta_{t}^{\dagger} := \argmax_{\btheta} \big\{\ell_{\lambda}^{\dagger}(\mathcal{F}_t;\btheta) + \alpha(t)\max_{a\in [K]}f(\bx_{t,a};\btheta) \big\}.\label{eq:original theta_t of NeuralRBMLE}
\end{equation}

Inspired by \cite{hung2020reward}, we can further show that $\text{NeuralRBMLE}$ can be simplified to an \textit{index strategy} by interchanging the two max operations in (\ref{eq:original theta_t of NeuralRBMLE}).
Now we are ready to present the prototypic NeuralRBMLE algorithm as follows: At each time $t$, 
\begin{enumerate}[leftmargin=*]
    \item Define the \textit{arm-specific} RBMLE estimators
    \begin{align}
    \btheta_{t,a}^{\dagger} := \argmax_{\btheta} \big\{\ell_{\lambda}^{\dagger}(\mathcal{F}_t;\btheta) + \alpha(t)f(\bx_{t,a};\btheta) \big\}. \label{def:RBMLE arm-specific}
\end{align}
\item Accordingly, for each arm, we construct an index as
\begin{align}
    \mathcal{I}^{\dagger}_{t,a} := \ell_{\lambda}^{\dagger}(\mathcal{F}_t;\btheta_{t,a}^{\dagger}) + \alpha(t)f(\bx_{t,a};\btheta_{t,a}^{\dagger}).  \label{eq:index}
\end{align}
\end{enumerate}


Then, it can be shown that the policy induced by the NeuralRBMLE in (\ref{eq:original theta_t of NeuralRBMLE}) is equivalent to an index strategy which selects an arm with the largest $\mathcal{I}^{\dagger}_{t,a}$ at each time $t$.
The proof is similar to that in \cite{hung2020reward} and provided in Appendix \ref{appendix:justification} for completeness.

\begin{remark}
\normalfont Note that in (\ref{eq:regularization})-(\ref{eq:index}) we use exponential family distributions as the surrogate likelihood functions for RBMLE.
When the reward distributions are unknown, one could simply resort to some commonly-used distributions, such as the Gaussian likelihood function. 
On the other hand, when additional structures of the reward distributions are known, this design also enables the flexibility of better matching the true likelihood and the surrogate likelihood. 
For example, in the context of logistic bandits, the rewards are known to be binary, and hence one can apply the Bernoulli likelihood to obtain the corresponding NeuralRBMLE estimator.
\end{remark}
\begin{remark}
\normalfont The surrogate likelihood function in (\ref{eq:exponential family log-likelihood}) follows the similar philosophy as that for the generalized linear bandits in \cite{hung2020reward}. 
Despite the similarity, one fundamental difference is that the objective function of NeuralRBMLE is no longer concave in $\btheta$.
In spite of this, in Section \ref{section:regret} we show that the practical algorithms derived from NeuralRBMLE still enjoy favorable regret bounds with the help of the neural tangent kernel.
\end{remark}
\vspace{-3mm}
\subsection{Practical NeuralRBMLE Algorithms}
\label{section:rbmle:alg}
\vspace{-2mm}
One major challenge in implementing the prototypic NeuralRBMLE algorithm in (\ref{def:RBMLE arm-specific})-(\ref{eq:index}) is that the exact maximizer $\btheta_{t,a}^{\dagger}$ in (\ref{def:RBMLE arm-specific}) can be difficult to obtain since the maximization problem of (\ref{def:RBMLE arm-specific}), which involves the neural function approximator $f(\bx;\btheta)$, is non-convex in $\btheta$. 
In this section, we proceed to present two practical implementations of NeuralRBMLE algorithm. 
\vspace{-1mm}
\hfill
\begin{figure}[t]
\begin{minipage}[t]{0.5\linewidth}
    \vspace{-3pt}
	\begin{algorithm}[H]
		\caption{NeuralRBMLE-GA}
	    \begin{algorithmic}[1]
	        \STATE {\bfseries Input:} $\alpha(t)$, $\zeta(t)$, $\lambda$, $f$, $\btheta_0$, $\eta$, $J$.
	        \STATE {\bfseries Initialization:} $\{\widetilde{\btheta}^{\dagger}_{0,i})\}_{i=1}^{K} \leftarrow \btheta_0$.
	        \vspace{2pt}
	        \FOR{$t=1,2,\cdots$}
	            \STATE Observe all the contexts $\{\bx_{t,a}\}_{1\leq a \leq K}$.
	            \FOR{$a=1,\cdots,K$}
	                \STATE Set $\widetilde{\btheta}^{\dagger}_{t,a}$ to be the output of $J$-step gradient ascent with step size $\eta$ for maximizing $\ell_{\lambda}^{\dagger}(\mathcal{F}_t;\btheta) + \alpha(t)f(\bx_{t,a};\btheta)$.
	            \ENDFOR
	            \STATE Choose $a_t=\argmax_{a} \big\{ \ell_{\lambda}^{\dagger}(\mathcal{F}_t;\widetilde{\btheta}^{\dagger}_{t,a}) + \alpha(t)\zeta(t)f(\bx_{t,a};\widetilde{\btheta}^{\dagger}_{t,a})\big\}$ and obtain reward $r_t$.
	        \ENDFOR
	    \end{algorithmic}
	    \label{alg:NeuralRBMLE-GA}
	\end{algorithm}
\end{minipage}
\hfill
\begin{minipage}[t]{0.49\linewidth}
\vspace{-3pt}
\begin{algorithm}[H]
\caption{NeuralRBMLE-PC}  
    \begin{algorithmic}[1]
        \STATE {\bfseries Input:} $\alpha(t)$,  $\lambda$, $f$, $\btheta_0$, $\eta$, $J$.
        \STATE {\bfseries Initialization:} $\textbf{Z}_0\leftarrow \lambda \textbf{I}$\;, $\widehat{\btheta}_0\leftarrow \btheta_0$.
        \FOR{$t=1,2,\cdots$}
            \STATE Observe all the contexts $\{\bx_{t,a}\}_{1\leq a \leq K}$.
            \FOR{$a=1,\cdots,K$}
            \STATE $ \Bar{\btheta}_{t,a} \leftarrow \widehat{\btheta}_{t} + \frac{\alpha(t)}{m}\cdot \textbf{Z}^{-1}_{t-1} \textbf{g}(\bx_{t,a};\widehat{\btheta}_{t}) $.
            \ENDFOR
            \STATE Choose  $a_t=\argmax_{a}\{f(\bx_{t,a};\Bar{\btheta}_{t,a})\}$ and obtain reward $r_t$.
            \STATE Set $\widehat{\btheta}_{t}$ as the output of $J$-step gradient ascent with step size $\eta$ for maximizing $\ell_{\lambda}^{\dagger}(\cF_t;\btheta)$.
            \STATE $\textbf{Z}_{t}\leftarrow \textbf{Z}_{t-1}+\textbf{g}(\bx_{t};\widehat{\btheta}_{t})\textbf{g}(\bx_{t};\widehat{\btheta}_{t})^\top/m$.
        \ENDFOR
    \end{algorithmic}
    \label{alg:NeuralRBMLE-PC}
\end{algorithm}
\end{minipage}
\vspace{-5mm}
\end{figure}
\begin{itemize}[leftmargin=*]
    \item \textbf{NeuralRBMLE by gradient ascent (NeuralRBMLE-GA):} To solve the optimization problem in (\ref{def:RBMLE arm-specific}), one natural approach is to apply gradient ascent to obtain an approximator $\widetilde{\btheta}^{\dagger}_{t,a}$ for $\btheta_{t,a}^{\dagger}$, for each arm.
    The pseudo code of NeuralRBMLE-GA is provided in Algorithm \ref{alg:NeuralRBMLE-GA}\footnote{Compared to the index in (\ref{eq:index}), one slight modification in Line 8 of Algorithm \ref{alg:NeuralRBMLE-GA} is the additional factor $\zeta(t)$, which is presumed to be a positive-valued and strictly increasing function (e.g., $\zeta(t)$ can be chosen as $1+\log t$). This modification was first considered by \citep{hung2020reward}. As shown in our regret analysis in Appendix \ref{appendix:Regret Bound of NeuralRBMLE-GA subsection 1}, the technical reason behind $\zeta(t)$ is only to enable the trick of completing the square, and $\zeta(t)$ does not affect the regret bound.}.
    Given the recent progress on the neural tangent kernel of neural networks \citep{jacot2018neural,cao2019generalization}, the estimator $\widetilde{\btheta}^{\dagger}_{t,a}$ serves as a good approximation for ${\btheta}_{t,a}^{\dagger}$ despite the non-concave objective function in (\ref{def:RBMLE arm-specific}). 
    This will be described in more detail in the regret analysis.
\item \textbf{NeuralRBMLE via reward-bias-guided parameter correction (NeuralRBMLE-PC):} 
Note that by (\ref{def:RBMLE arm-specific}), finding each $\btheta_{t,a}^{\dagger}$ originally involves solving an optimization problem for each arm.
To arrive at a more computationally efficient algorithm, we propose a surrogate index for the original index policy in (\ref{eq:index}) with Gaussian likelihood.
We observe that the main difference among the estimators $\btheta_{t,a}^{\dagger}$ of different arms lies in the reward-bias term $\alpha(t)f(\bx_{t,a};\btheta_{t,a}^{\dagger})$, as shown in (\ref{eq:index}).
Based on this observation, we propose to first (i) find a \textit{base estimator} $\widehat{\btheta}_t$ without any reward bias and then (ii) approximately obtain the arm-specific RBMLE estimators by involving the neural tangent kernel of neural networks. Define
\begin{align}
    \widehat{\btheta}^{\dagger}_{t} :&= \argmax_{\btheta} \big\{\ell_{\lambda}^{\dagger}(\mathcal{F}_t;\btheta) \big\}, \label{eq:UCB estimator}
\end{align}
Notably, $\widehat{\btheta}^{\dagger}_{t}$ can be viewed as the least squares estimate given $\cF_t$ for the neural network $f(\cdot;\btheta)$.  We apply $J$-step gradient ascent with step size $\eta$ to solve (\ref{eq:UCB estimator}), and denote $\widehat{\btheta}_{t}$ as the output of gradient ascent. Next, we define $\bZ_t := \lambda \bI + \frac{1}{m}\sum_{\tau=1}^{t} \bg(\bx_\tau;\widehat{\btheta}_\tau)\bg(\bx_\tau;\widehat{\btheta}_\tau)^\top$
and construct an approximate estimator for $\btheta_{t,a}^{\dagger}$ as
\begin{align}
    \Bar{\btheta}_{t,a} := \widehat{\btheta}_{t} + \frac{\alpha(t)}{m}\bZ^{-1}_{t-1} \bg(\bx_{t,a};\widehat{\btheta}_{t}), \label{def:theta_bar}
\end{align} 
where $\frac{\alpha(t)}{m}\bZ^{-1}_{t-1} \bg(\bx_{t,a};\widehat{\btheta}_{t})$ reflects the effect of the reward-bias term on the neural network parameter.
Then, we propose a surrogate index $\Bar{\mathcal{I}}_{t,a}$ for the index $\cI^{\dagger}_{t,a}$ as $\Bar{\mathcal{I}}_{t,a} := f(\bx_{t,a};\Bar{\btheta}_{t,a})$, for all $a \in [K]$ and for all $t\in [T]$.
The detailed derivation for the surrogate index $\Bar{\mathcal{I}}_{t,a}$ is provided in Appendix \ref{appendix:index of NeuralRBMLE-PC}.
The pseudo code of the NeuralRBMLE-PC of this surrogate index is provided in Algorithm \ref{alg:NeuralRBMLE-PC}.
The main advantage of NeuralRBMLE-PC is that at each time step $t$, the learner only needs to solve one optimization problem for the base estimator $\widehat{\btheta}_t$ and then follow the guidance of $\bZ_{t-1}^{-1}\bg(\bx_{t,a};\widehat{\btheta}_{t})$ that are readily available, instead of solving multiple optimization problems.
\end{itemize}

\section{Regret Analysis of NeuralRBMLE}
\label{section:regret}
In this section, we present the regret analysis of the NeuralRBMLE algorithms.
We denote $\bH$ as the NTK matrix and $\widetilde{d}$ as the effective dimension of $\bH$. The detailed definitions are in Appendix \ref{Appendix: Regret Analysis}. 
\begin{assumption}
$\textbf{H} \succeq \lambda \textbf{I}$, and for all $a \in [K], t \in [T], \lVert \textbf{x}_{t,a} \rVert_2 = 1$ and $[\textbf{x}]_{j} = [\textbf{x}]_{j+d/2}$.
\label{assumption:ntk}
\end{assumption}
\vspace{-2mm}
Based on this assumption, we can ensure that $\bH$ is a positive-definite matrix, and this can be satisfied if no two contexts are parallel. The second part ensures that $f(\bx_{t,a};\btheta_0) = 0$, which is mainly for the convenience of analysis. Next, we provide the regret bounds of NeuralRBMLE.
\begin{theorem}
\label{theorem:regret of NeuralRBMLE-GA}
Under NeuralRBMLE-GA in Algorithm \ref{alg:NeuralRBMLE-GA}, there exist positive constants $\{C_{\text{GA},i}\}^4_{i=1}$ such that for any $\delta \in (0,1)$, if $\alpha(t) = \Theta(\sqrt{t})$, $\eta \leq C_{\text{GA},1}(m\lambda+TmL)^{-1}$, $J \geq C_{\text{GA},2}\frac{TL}{\lambda}$, and 
\begin{align}
     m \geq C_{\text{GA},3} &\max\bigg\{T^{16}\lambda^{-7}L^{24}(\log{m})^3, T^6K^6L^6\lambda^{-\frac{1}{2}}\left(\log\left(TKL^2/\delta\right)\right)^{\frac{3}{2}}\bigg\}, 
\end{align}
then with probability at least $1-\delta$, the regret satisfies $\mathcal{R}(T) \leq C_{\text{GA},4}\sqrt{T}\cdot\widetilde{d}\log(1+TK/\lambda)$.
\end{theorem}
The detailed proof is provided in Appendix 
\ref{appendix:Regret Bound of NeuralRBMLE-GA}.

We highlight the technical novelty of the analysis for NeuralRBMLE-GA as follows. The challenges in establishing the regret bound in Theorem \ref{theorem:regret of NeuralRBMLE-GA} are mainly three-fold: 
\vspace{-3mm}
\begin{itemize}[leftmargin=*]
    \item (a) \textit{Model misspecification}: In NeuralRBMLE-GA, the surrogate likelihood is designed to be flexible and takes the general form of an exponential family distribution, instead of a Gaussian distribution as used in the existing neural bandit methods \citep{zhou2020neural,zhang2020neural}. As a result, we cannot directly exploit the special parametric form of the Gaussian likelihood in analyzing the RBMLE index. Therefore, to tackle the model misspecification problem, the proof is provided with no assumption on reward distribution and works for a surrogate likelihood of an exponential family distribution.   
    \vspace{-1mm}
    \item (b) \textit{Three-way trade-off}: In neural bandits, applying the RBMLE principle for low regret involves one inherent dilemma -- the reward bias and $\alpha(t)$ need to be large enough to achieve sufficient exploration, while $\alpha(t)$ needs to be sufficiently small to enable the NTK-based approximation for low regret (cf. Lemma \ref{appendix_lemma_1}). This is one salient difference from RBMLE for linear bandits \citep{hung2020reward}.
    \vspace{-1mm}
    \item (c) \textit{Compatibility of RBMLE with NTK}: As the additional reward-bias term is tightly coupled with the neural network in the objective (\ref{eq:original theta_t of NeuralRBMLE}) and the arm-specific estimators in (\ref{def:RBMLE arm-specific}), it is technically challenging to quantify the deviation of the learned policy parameter from $\btheta_0$. This is one salient difference between NeuralRBMLE-GA and the existing neural bandit algorithms \citep{zhou2020neural,zhang2020neural}. 
\end{itemize}
\vspace{-2mm}
 
Given the above, we address these issues as follows: (i) We address the issue (b) by carefully quantifying the distance between the learned policy parameters and the initial parameter (cf. Lemmas \ref{lemma:error between widetilde_theta and theta}-\ref{lemma:error between widetilde_theta_dagger and theta_0} along with the supporting Lemmas \ref{lemma:upper bound of r}-\ref{lemma:upper bound of b'}).
(ii) We tackle the issue (c) by providing bounds regarding the log-likelihood of the exponential family distributions in analyzing the arm-specific index (cf. Lemmas \ref{lemma:error between theta_tilde_dagger and theta}-\ref{lemma:index policy for NeuralRBMLE-GA}).
(iii) Finally, based on the above results, we address the issue (a) by choosing a proper $\alpha(t)$ that achieves low regret and enables NTK-based analysis simultaneously.


\begin{theorem}
\label{theorem:regret of NeuralRBMLE-PC}
Under NeuralRBMLE-PC in Algorithm \ref{alg:NeuralRBMLE-PC}, there exist positive constant $C_{\text{PC},1}$, $C_{\text{PC},2}$ and $C_{\text{PC},3}$ such that for any $\delta \in (0,1)$, if $\alpha(t) = \Theta(\sqrt{t})$, $\eta \leq C_{\text{PC},1}(m\lambda+TmL)^{-1}$, and 
\begin{align}
     m \geq C_{\text{PC},2} &\max\bigg\{T^{21}\lambda^{-7}L^{24}(\log{m})^3,T^6K^6L^6\lambda^{-\frac{1}{2}}\left(\log\left(TKL^2/\delta\right)\right)^{\frac{3}{2}}\bigg\}, 
\end{align}
the with probability at least $1-\delta$, the regret satisfies $\mathcal{R}(T) \leq C_{\text{PC},3}\sqrt{T}\cdot\widetilde{d}\log(1+TK/\lambda)$. 
\end{theorem}
The complete proof is provided in Appendix \ref{appendix:Regret Bound of NeuralRBMLE-PC}.

We highlight the challenge and technical novelty of the analysis for NeuralRBMLE-PC as follows:
Notably, the main challenge in establishing the regret bound in Theorem \ref{theorem:regret of NeuralRBMLE-PC} lies in that the bias term affects not only the \textit{explore-exploit trade-off} but also the \textit{approximation capability} of the NTK-based analysis. Such a {three-way trade-off} due to the reward-bias term serves as one salient feature of NeuralRBMLE-PC, compared to RBMLE for linear bandits \citep{hung2020reward} and other existing neural bandit algorithms \citep{zhou2020neural,zhang2020neural}.
Despite the above challenges, we are still able to (i) characterize the interplay between regret and the reward-bias $\alpha(t)$ (cf. (\ref{eq: NeuralRBMLE-PC regret 2})-(\ref{eq: NeuralRBMLE-PC regret 9}) in Appendix \ref{appendix: Regret Bound of NeuralRBMLE-PC subsection 2}), (ii) specify the distance between $\btheta_0$ and the learned policy parameter $\bar{\btheta}^{\dagger}_{t,a}$ induced by the bias term (cf. Lemma \ref{lemma:range of theta_bar}), and (iii) carefully handle each regret component that involves $\lVert \bg(\bx^{*}_{t};\widehat{\btheta}_t) \rVert_2$ in the regret bound by the technique of completing the square (cf. (\ref{eq: NeuralRBMLE-PC regret 10})-(\ref{eq: NeuralRBMLE-PC regret 21}) in Appendix \ref{appendix: Regret Bound of NeuralRBMLE-PC subsection 2}).

\begin{figure*}[!t]
\vspace{-8mm}
$\begin{array}{c c c}
    \multicolumn{1}{l}{\mbox{\bf }} & \multicolumn{1}{l}{\mbox{\bf }} & \multicolumn{1}{l}{\mbox{\bf }} \\ 
    \scalebox{0.33}{\includegraphics[width=\textwidth]{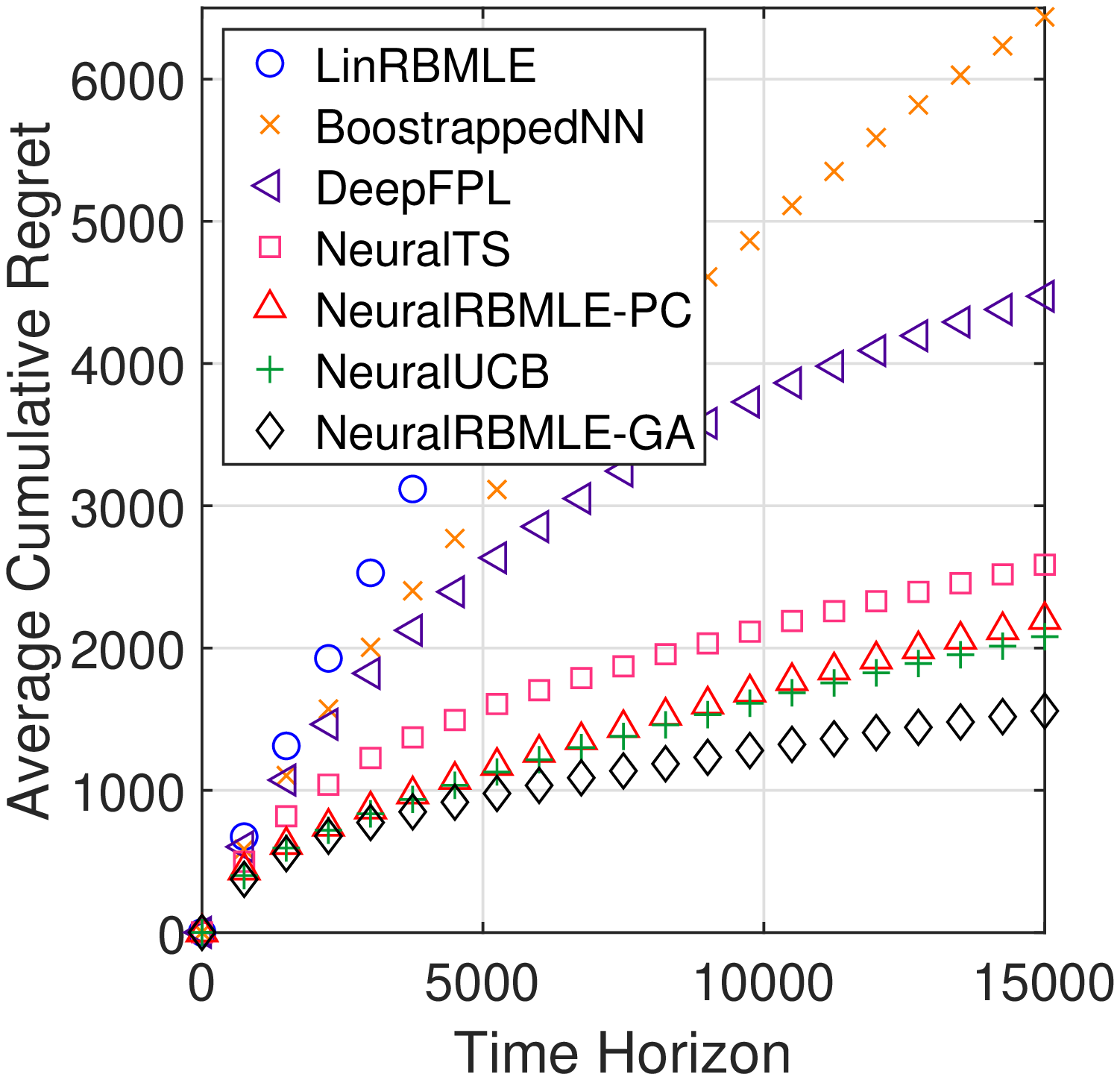}}  \label{fig:real_mnist} & \hspace{-3mm} \scalebox{0.33}{\includegraphics[width=\textwidth]{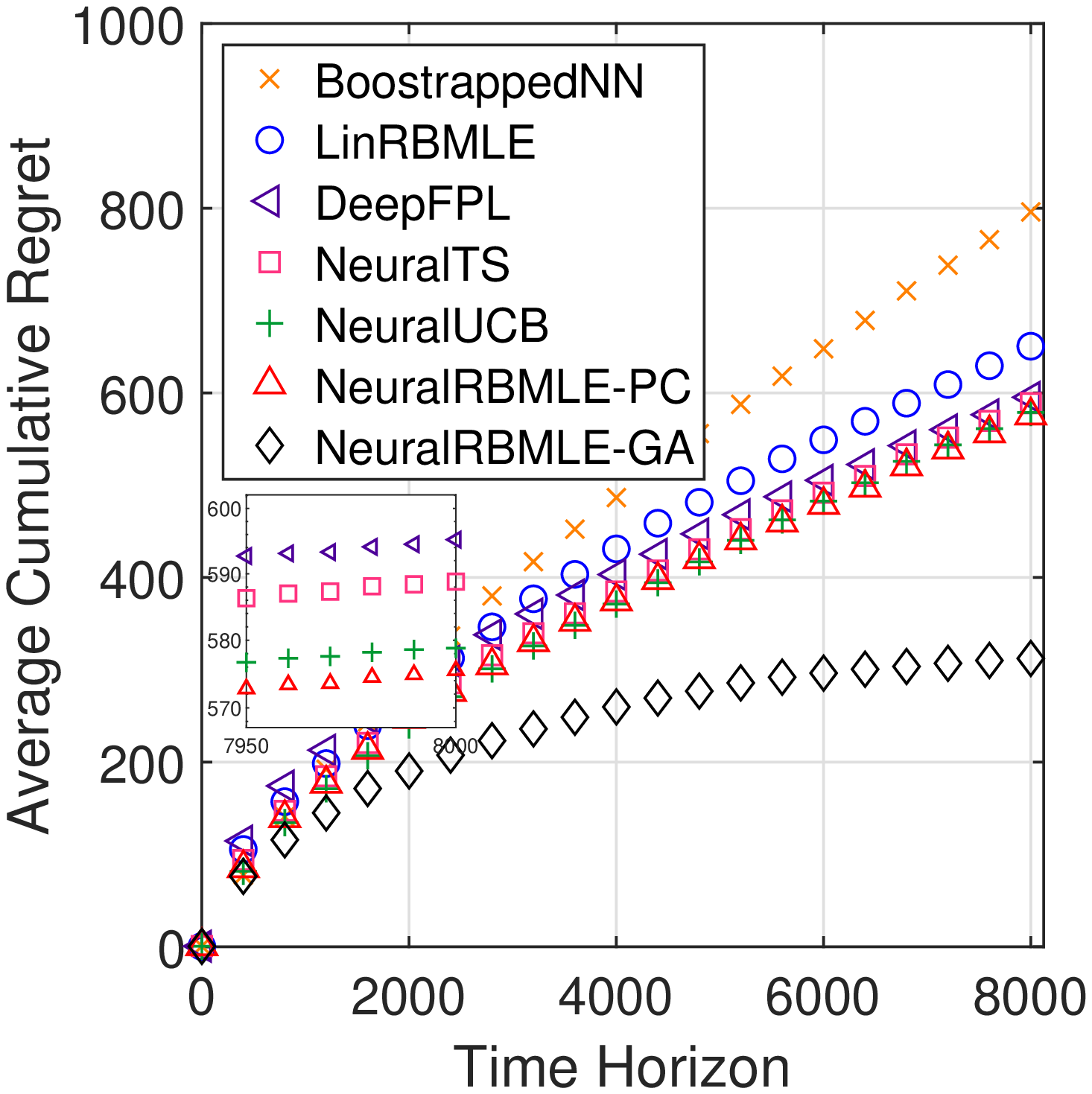}} \label{fig:real_mushroom} &  \scalebox{0.33}{\includegraphics[width=\textwidth]{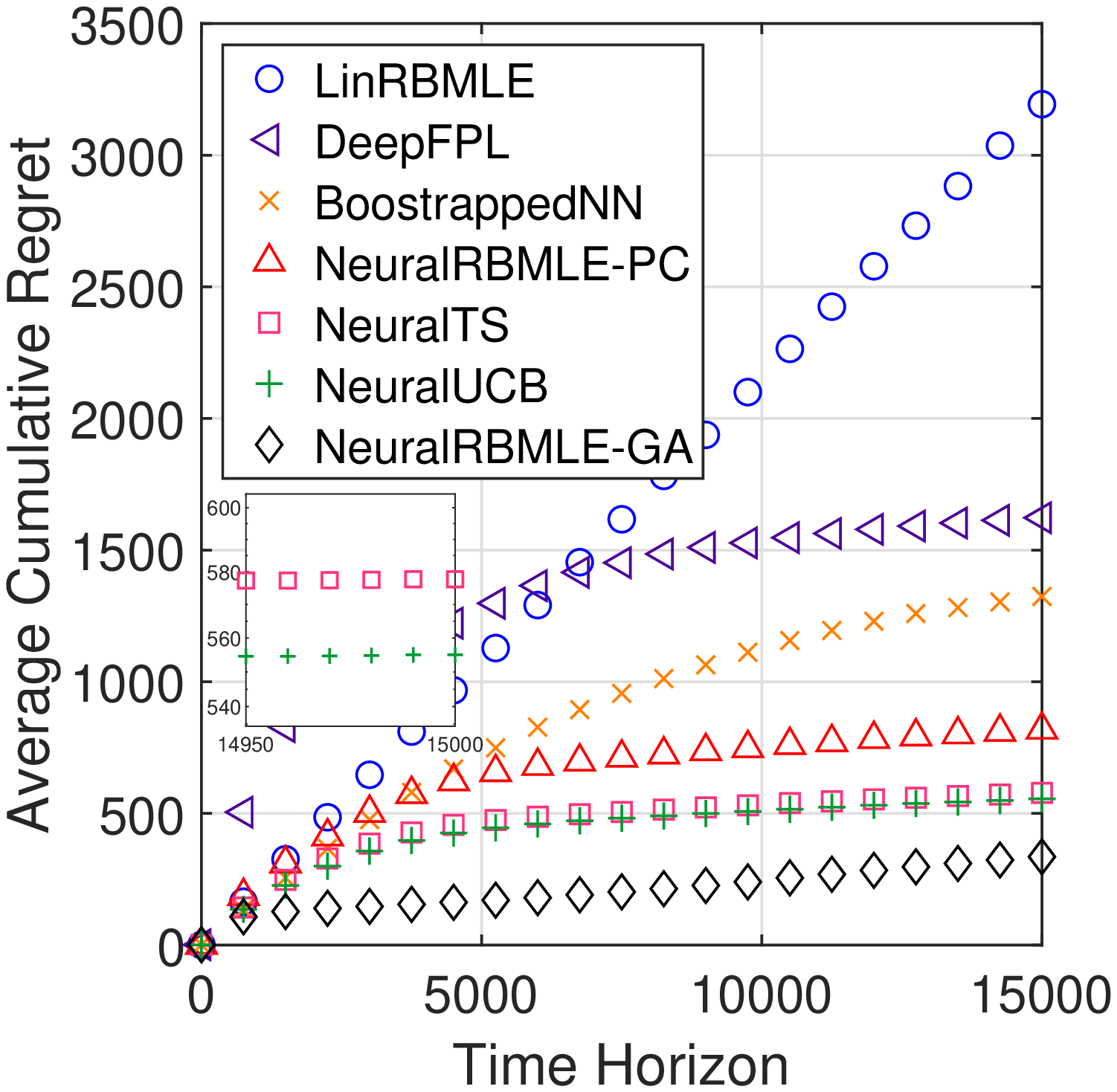}} \label{fig:real_shuttle} \\
    \mbox{(a) MNIST} &  \mbox{(b) Mushroom} &  \mbox{(c) Shuttle} 
\end{array}$
\vspace{-2mm}
\caption{Cumulative regret averaged over $10$ trials with $T = 1.5\times 10^4$.}
\label{fig:real_regret}
\vspace{-5mm}
\end{figure*}
\vspace{-3mm}
\section{Numerical Experiments}
\vspace{-2mm}
\label{section:exp}
We evaluate the performance of NeuralRBMLE against the popular benchmark methods through experiments on various real-world datasets, including Adult, Covertype, Magic Telescope, Mushroom, Shuttle~\citep{asuncion2007uci}, and MNIST~\citep{lecun2010mnist}.
\hungyh{The detailed configuration of the experiments is provided in Appendix \ref{appendix:hyp}.}
All the results reported in this section are the average over $10$ random seeds.
\begin{figure*}[!ht]
\vspace{-8mm}
    \begin{minipage}[b]{0.75\linewidth}
        \centering
        $\begin{array}{c c c}
            \multicolumn{1}{l}{\mbox{\bf }} \hspace{-0.5cm}& \multicolumn{1}{l}{\mbox{\bf }} \hspace{-0.5cm}& \multicolumn{1}{l}{\mbox{\bf }} \\ 
            \scalebox{0.33}{\includegraphics[width=\textwidth]{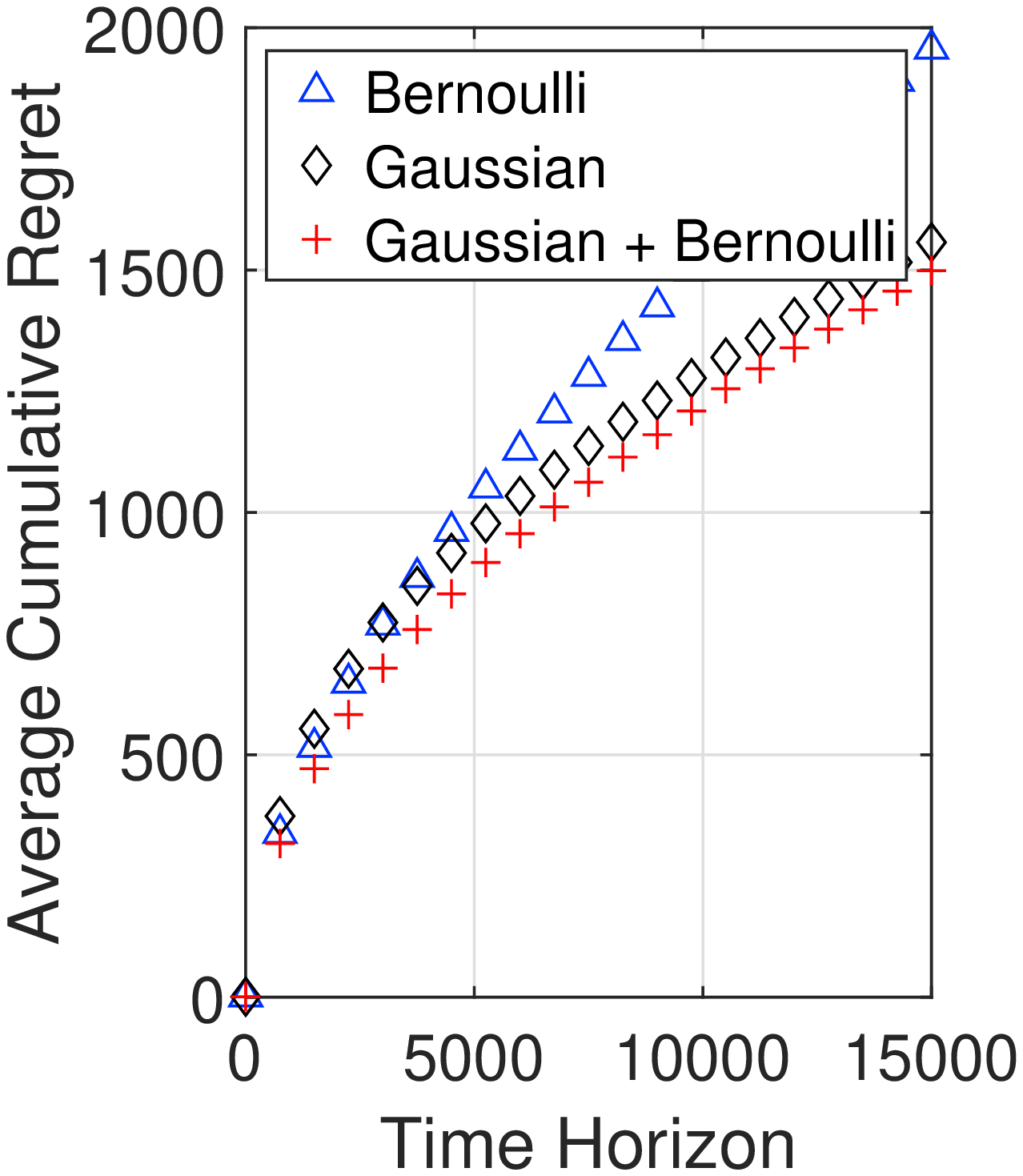}} \hspace{-0.5cm}& \scalebox{0.33}{\includegraphics[width=\textwidth]{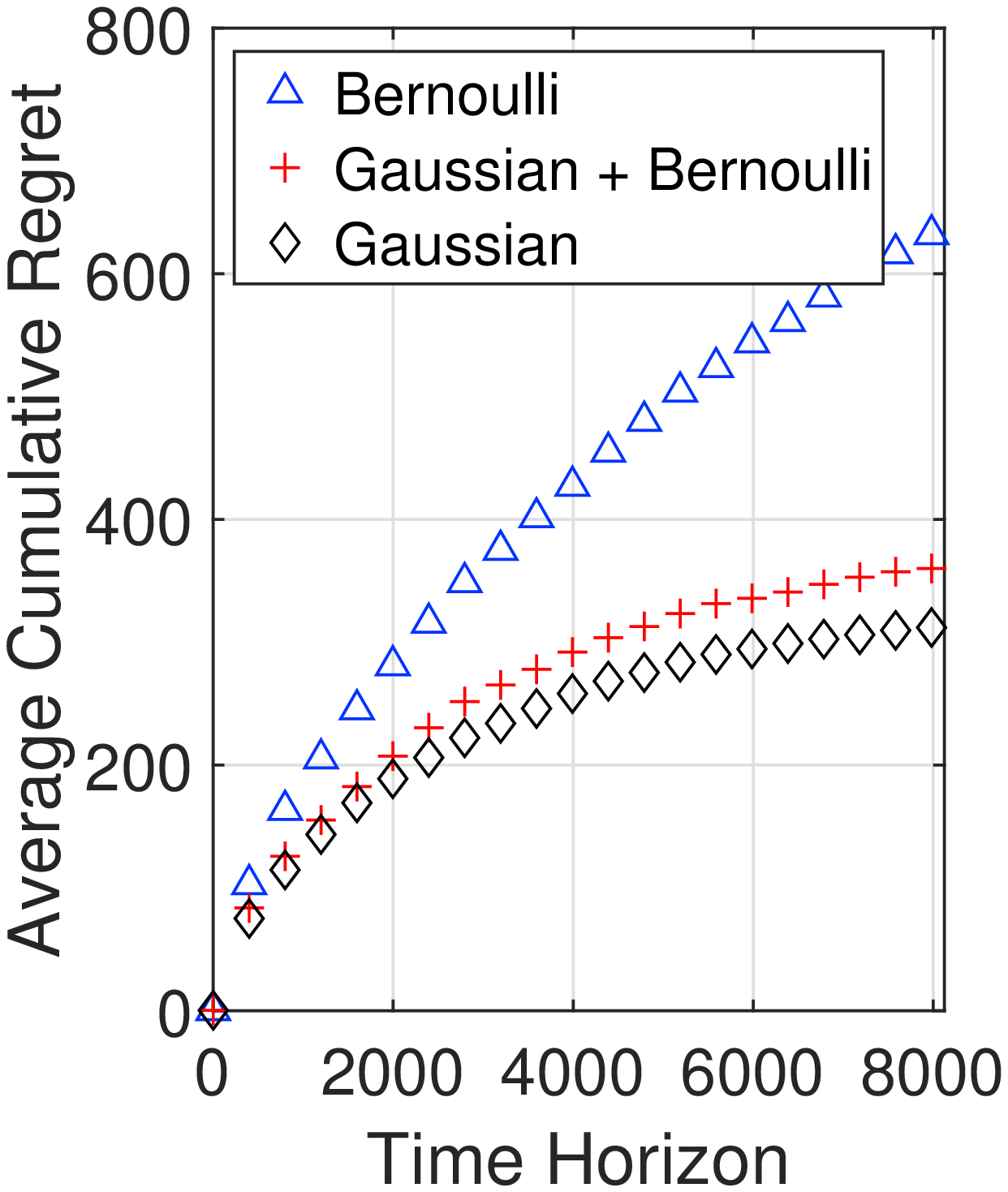}} \hspace{-0.5cm}&  \scalebox{0.33}{\includegraphics[width=\textwidth]{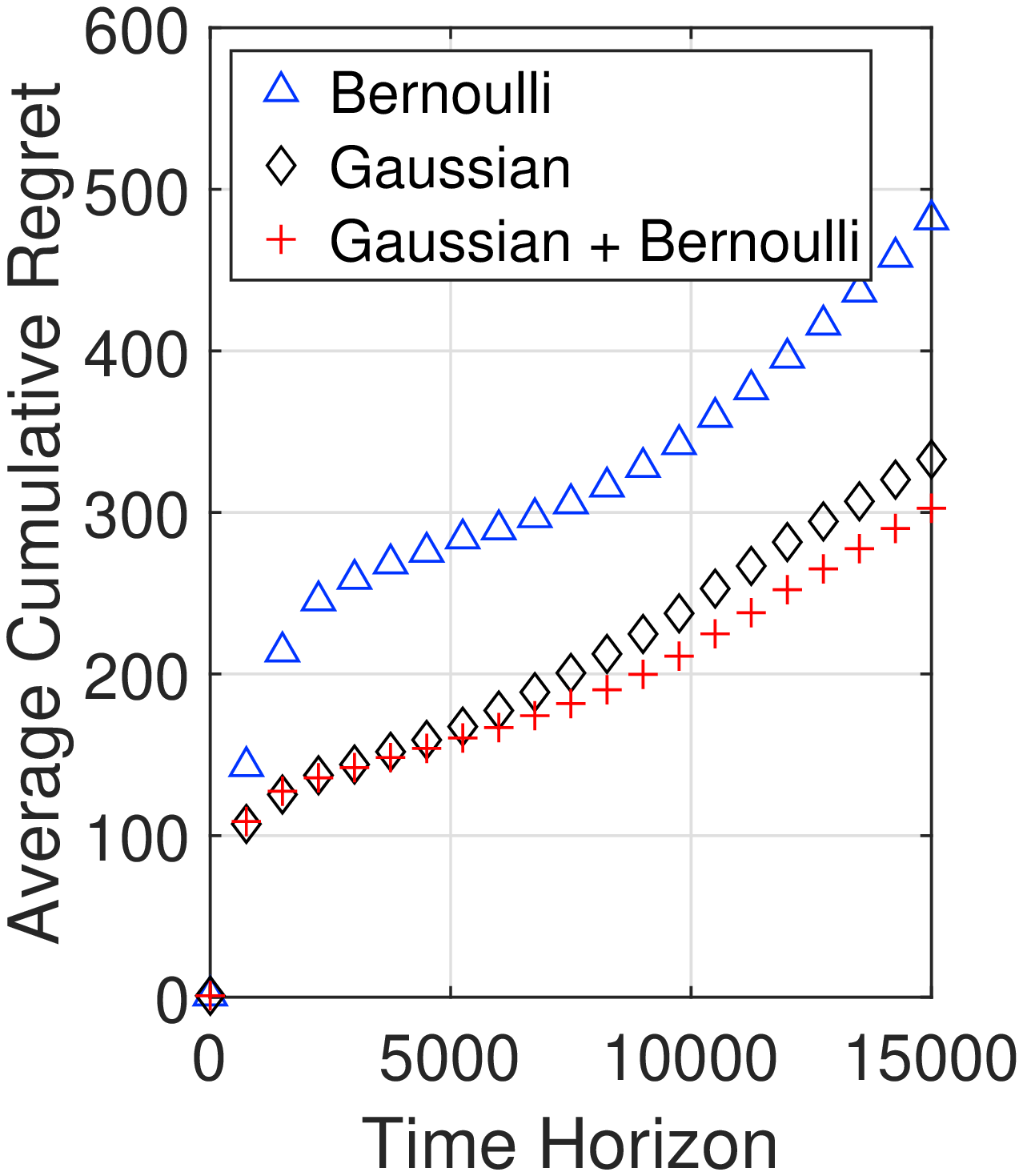}} \\
            \mbox{(a) MNIST} \hspace{-0.5cm}&  \mbox{(b) Mushroom} \hspace{-0.5cm}&  \mbox{(c) Shuttle}  
        \end{array}$
        \vspace{-2mm}
        \caption{Regret of the three variants of NeuralRBMLE-GA averaged\\ over $10$ trials with $T = 1.5\times 10^4$.}
        \vspace{-2mm}
        \label{fig:real_regret_3}
    \end{minipage}
    \hspace{-2mm}
    \begin{minipage}[b]{0.25\linewidth}
        \centering
        $\begin{array}{c} 
            \multicolumn{1}{l}{\mbox{\bf }} \\ \scalebox{0.99}{\includegraphics[width=\textwidth]{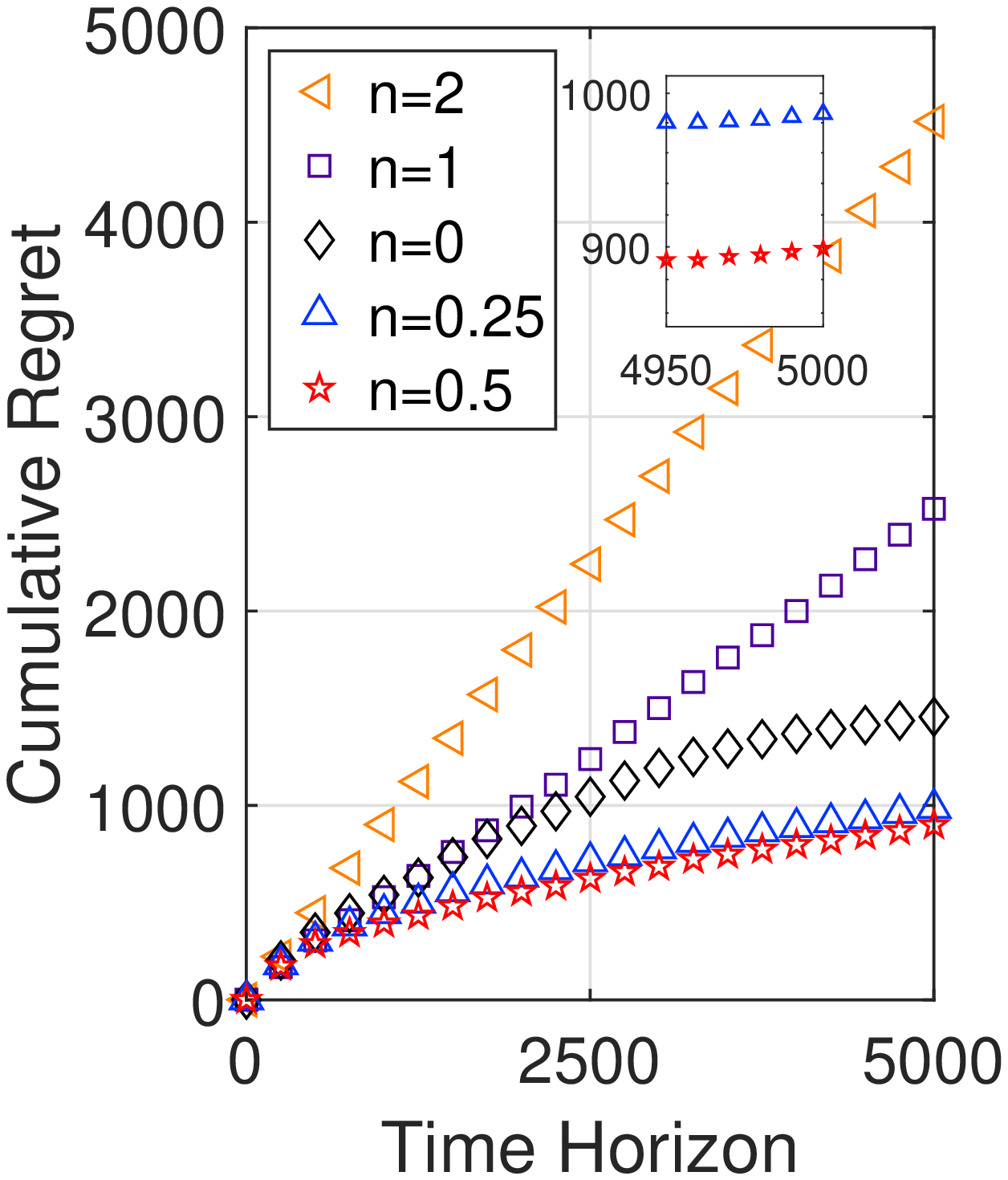}} \\  
            \mbox{(a) MNIST}  
        \end{array}$
        \vspace{-2mm}
        \caption{Regret under different $\alpha(t)$.}
        \vspace{-2mm}
        \label{fig:real_regret_3.5}
    \end{minipage}
\end{figure*}

\textbf{Effectiveness of NeuralRBMLE.} Figure \ref{fig:real_regret}, Table \ref{table:mean regret}, and Figure \ref{fig:real_regret_2} (cf. Appendix \ref{appendix:exp}) show the empirical regret performance of each algorithm. We observe that NeuralRBMLE-GA outperforms most of the other benchmark methods under the six real-world datasets, and NeuralRBMLE-PC achieves empirical regrets competitive to the other state-of-the-art neural bandit algorithms. 
Table \ref{table:std regret} in Appendix \ref{appendix:exp} shows the standard deviation of the regrets over random seeds. 
Notably, we observe that the performance of NeuralRBMLE and LinRBMLE appears more consistent across different random seeds than other benchmarks. This suggests that the reward-biased methods empirically enjoy better robustness across different sample paths.

\textbf{NeuralRBMLE-GA with different likelihood functions.} As described in Sections \ref{section:rbmle:alg} and \ref{section:regret}, NeuralRBMLE-GA opens up a whole new family of algorithms with provable regret bounds. We provide the experiment of NeuralRBMLE-GA with the three different likelihood functions (i) $-\frac{1}{2}(f(\bx_s;\btheta) - r_s)^2$, (ii) $r_sf(\bx_s;\btheta) - \log{(1+e^{f(\bx_s;\btheta)})}$ and (iii) $-\frac{1}{2}(f(\bx_s;\btheta) - r_s)^2 + r_s f(\bx_s;\btheta) - \log{(1+e^{f(\bx_s;\btheta)})}$ are Gaussian, Bernoulli, and the mixture of Gaussian and Bernoulli likelihood, respectively, and all of them belong to some canonical exponential family. Figure \ref{fig:real_regret_3} shows the empirical regrets of NeuralRBMLE-GA under different likelihood functions. We can observe that these three variants of NeuralRBMLE-GA have empirical regret competitive to other neural bandit algorithms like NeuralUCB, NeuralTS and DeepFPL. NeuralRBMLE-GA with Gaussian likelihood and the mixture of Gaussian and Bernoulli likelihood are the best in empirical regret. This corroborates the flexibility of NeuralRBMLE-GA in selecting the surrogate likelihood.

\hungyh{\textbf{Computational efficiency NeuralRBMLE-GA}
NeuralRBMLE-GA achieves the same $\tilde{\mathcal{O}}(\sqrt{T})$ regret as NeuralUCB and NeuralTS, but \textit{without requiring the inverse of $\bZ \in \mathbb{R}^{p \times p}$}, where $p= m+md+m^2(L-1)$. The computation of $\bZ^{-1}$ can take a substantial amount of time even if the context dimension $d$ is only moderately large. 
To further verify this, we measure the average per-step computation time and the time spent on computing $\bZ^{-1}$ of the three algorithms on the Covertype dataset (with $d=378$, $m=100$, and one hidden layer) and show the computational advantage of NeuralRBMLE-GA. The result is provided in Table \ref{table:computation time} in Appendix \ref{section:additional experiments}.
}

\hungyh{\textbf{Ablation study for $\alpha(t)$.}
We provide an ablation study for $\alpha(t)$ by evaluating the empirical regret on the MNIST dataset for NeuralRBMLE-GA with $\alpha(t) = t^{n}$, where $n \in \{0,0.25,0.5,1,2\}$. The results are provided in Figure \ref{fig:real_regret_3.5} and corroborate the choice $\alpha(t)=\Theta(\sqrt{t})$ provided by the theory.
}

\vspace{-3mm}
\section{Related Work}
\vspace{-2mm}
To relax the linear realizability assumption, contextual bandits have been studied from the perspectives of using known general kernels \cite{valko2013finite} and the surrogate model provided by Gaussian processes \cite{chowdhury2017kernelized}. 
Moreover, one recent attempt is to leverage the representation power of deep neural networks to learn the unknown bounded reward functions \cite{gu2021batched,kveton2020randomized,riquelme2018deep,zahavy2019deep,zhou2020neural,zhang2020neural,zhu2021pure,xu2020neural,anonymous2022learning}.
Among the above, the prior works most related to this paper are NeuralUCB \cite{zhou2020neural} and NeuralTS \cite{zhang2020neural}, which incorporate the construction of confidence sets of UCB and the posterior sampling technique of TS into the training of neural networks, respectively.
Another two recent preprints propose Neural-LinUCB \citep{xu2020neural} and NPR \citep{anonymous2022learning}, which incorporate shallow exploration in the sense of LinUCB and the technique of reward perturbation into neural bandits, respectively.
Different from the above neural bandit algorithms, NeuralRBMLE achieves efficient exploration in a fundamentally different manner by following the guidance of the reward-bias term in the parameter space, instead of using posterior sampling, reward perturbation, or confidence bounds derived from concentration inequalities.

\vspace{-3mm}
\section{Conclusion}
\vspace{-2mm}
\label{section:conclusion}
This paper presents NeuralRBMLE, the first bandit algorithm that extends the classic RBMLE principle in adaptive control to contextual bandits with general reward functions. Through regret analysis and extensive simulations on real-world datasets, we show that NeuralRBMLE achieves competitive regret performance compared to the state-of-the-art neural bandit algorithms.
One limitation of this work is that our analysis is built upon NTK, which typically requires a large hidden layer size.
In practice, NeuralRBMLE performs well with a relatively small neural network.
On the other hand, our model does not consider fairness in learning the unknown parameters. 
NeuralRBMLE may consistently discriminate against some specific groups of users based on their social, economic, racial, and sexual characteristics during the learning process. These open problems serve as critical research directions for future investigations.


\bibliography{reference}
\bibliographystyle{icml2022}

\appendix
\section*{Appendix}
\section{Supporting Lemmas for the Regret Analysis}
\label{Appendix: Regret Analysis}
\begin{definition} \normalfont \cite{jacot2018neural}.
For $i,j \in [TK]$, define
\begin{align}
    &{\bf\Sigma}^1_{i,j} := \langle \bx_i,\bx_j \rangle,\quad
    \textbf{A}^l_{i,j} := \begin{pmatrix}
        {\bf\Sigma}^l_{i,j} & {\bf\Sigma}^l_{i,j}\\
        {\bf\Sigma}^l_{i,j} & {\bf\Sigma}^l_{i,j}
    \end{pmatrix}\\
    &{\bf\Sigma}^{l+1}_{i,j} := 2\cdot \mathbb{E}_{(u,v) \sim \mathcal{N}({\bf 0},\textbf{A}^l_{i,j})}[\sigma(u)\sigma(v)]\\
    &\widetilde{\textbf{H}}_{i,j}^{1} := {\bf\Sigma}^1_{i,j}\\
    &\widetilde{\textbf{H}}_{i,j}^{l+1} := 2\widetilde{\textbf{H}}_{i,j}^{l} \mathbb{E}_{(u,v) \sim \mathcal{N}(0,\textbf{A}^l_{i,j})}[\dot{\sigma}(u)\dot{\sigma}(v)] + {\bf\Sigma}^{l+1}_{i,j},
\end{align}
where $\dot{\sigma}(\cdot)$ denotes the derivative of $\sigma(\cdot)$. Then, we define the neural tangent kernel matrix of an L-layer ReLU network on the training inputs $\{\bx_i\}_{i=1}^{TK}$ as
\begin{equation}
    \textbf{H} := \left[(\widetilde{\textbf{H}}_{i,j}^{L}+{\bf\Sigma}^L_{i,j})/2\right]_{TK\times TK}. \label{def:ntk}
\end{equation}
\end{definition}
Then, for $i \in [TK]$, we define the effective dimension $\widetilde{d}$ of the NTK matrix on contexts $\{\bx_i\}$ as \begin{align}
    \widetilde{d} := \frac{\log \text{det}(\textbf{I}+ \textbf{H}/\lambda)}{\log (1+TK/\lambda)}\label{def:effective dimension}.
\end{align}
The effective dimension $\widetilde{d}$ is a data-dependent measure first introduced by \cite{valko2013finite} to measure the actual underlying dimension of the set of observed contexts and later adapted to neural bandits \citep{zhou2020neural,zhang2020neural}, and $\widetilde{d}$ can be upper bounded in various settings \citep{chowdhury2017kernelized,zhang2020neural}.

\begin{lemma} \citep[Lemma 5.1]{zhou2020neural}
\label{lemma:theta*}
There exists a positive constant $E_{\ref{lemma:theta*}}$ and $\btheta^{*} \in \mathbb{R}^p$ such that for any $\delta \in (0,1)$ and all observed context vectors $\bx$, with probability at least $1-\delta$, under $m \geq E_{\ref{lemma:theta*}}T^4K^4L^6\log(T^2K^2L/\delta^4)$, we have
\begin{align}
    &h(\bx) = \langle \bg(\bx;\btheta_0) , \btheta^{*} - \btheta_0 \rangle \label{eq:theta* eq-1}\\
    &\sqrt{m}\lVert \btheta^{*} - \btheta_0 \rVert_2 \leq \sqrt{2\bh^\top\bH^{-1}\bh} \leq S, \label{eq:theta* eq-2}
\end{align}
where $\bH$ is defined in (\ref{def:ntk}), and $\bh := [h(\bx_{t,a})]_{a \in [K], t \in [T]}$. According to \cite{zhou2020neural}, $S$ is a constant when the reward function $h(\cdot)$ belongs to $\mathcal{H}$, the RKHS function space induced by NTK, and $\Vert h \Vert_{\mathcal{H}} \leq \infty$.
\end{lemma}

\begin{lemma}
\label{appendix_lemma_1}
\citep[Lemma 4.1]{cao2019generalization} There exist positive constants $E_{\ref{appendix_lemma_1}}$, $E_{\ref{appendix_lemma_1},1}$ and $E_{\ref{appendix_lemma_1},2}$ such that for any $\delta \in (0,1)$ and all observed context vectors $\bx$, with probability at least $1-\delta$, if $\tau$ satisfies that 
\begin{align}
E_{\ref{appendix_lemma_1},1}m^{-\frac{3}{2}}L^{-\frac{3}{2}}\left[\log\left(TKL^2/\delta\right)\right]^{\frac{3}{2}} \leq \tau \leq E_{\ref{appendix_lemma_1},2} L^{-6}(\log{m})^{-\frac{3}{2}},
\end{align}
then for all $\btheta$ that satisfies $\left\Vert \btheta - \btheta_0 \right\Vert_2 \leq \tau$, we have
\begin{equation} 
    \big|f(\bx;\btheta) - \langle \bg(\bx; \btheta)\,,\btheta - \btheta_0 \rangle\big| \leq E_{\ref{appendix_lemma_1}} \tau^{\frac{4}{3}}L^3\sqrt{m\log{m}}.
\end{equation}
Notice that under the initialization steps in Section \ref{section:nn}, we have $f(\bx;\btheta_0) = 0$ for all observed context vectors $\bx$.  
\end{lemma}

\begin{lemma}\label{appendix_lemma_2} \citep[Theorem 5]{allen2019convergence} There exist positive constants $E_{\ref{appendix_lemma_2}}$, $E_{\ref{appendix_lemma_2},1}$ and $E_{\ref{appendix_lemma_2},2}$ such that for any $\delta \in (0,1)$ and all observed context vectors $\bx$, with probability at least $1-\delta$, if $\tau$ satisfies that \begin{equation}               
    E_{\ref{appendix_lemma_2},1}m^{-\frac{3}{2}}L^{-\frac{3}{2}}\max\left\{\left(\log{m}\right)^{-\frac{3}{2}},\log\left(TK/\delta\right)^\frac{3}{2} \right\} \leq \tau \leq E_{\ref{appendix_lemma_2},2}L^{-\frac{9}{2}}(\log{m})^{-3},
\end{equation}
then for all $\btheta$ that satisfies $\left\Vert \btheta - \btheta_0 \right\Vert_2 \leq \tau$, we have
\begin{align}
    \left\Vert \bg(\bx; \btheta) - \bg(\bx; \btheta_0) \right\Vert_2 &\leq E_{\ref{appendix_lemma_2}}\sqrt{\log m}\tau^\frac{1}{3}L^3\left\Vert \bg(\bx; \btheta_0)\right\Vert_2.
\end{align}
\end{lemma}

\begin{lemma} \label{lemma:gradient of f}
\citep[Lemma B.3]{cao2019generalization} There exist positive constants $E_{\ref{lemma:gradient of f}}$, $E_{\ref{lemma:gradient of f},1}$ and $E_{\ref{lemma:gradient of f},2}$ for any $\delta \in (0,1)$ and all observed context vectors $\bx$, with probability at least $1-\delta$, if $\tau$ satisfies that
\begin{equation} 
    E_{\ref{lemma:gradient of f},1}m^{-\frac{3}{2}}L^{-\frac{3}{2}}[\log{TKL^2/\delta}]^{\frac{3}{2}} \leq \tau \leq E_{\ref{lemma:gradient of f},2}L^{-6}[\log{m}]^{-\frac{3}{2}},
\end{equation}
then for all $\btheta$ that satisfies $\left\Vert \btheta - \btheta_0 \right\Vert_2 \leq \tau$, we have
$\left\Vert \bg(\bx;\btheta) \right\Vert_{2} \leq E_{\ref{lemma:gradient of f}} \sqrt{mL}$.
\end{lemma}

\begin{lemma}
\label{Hoeffding's nequality}
(Hoeffding's inequality) For a sequence of independent, zero-mean, sub-Gaussian random variables $X_1,\dots,X_t$, we have
\begin{align}
    \mathbb{P}\left\{ \frac{1}{t}\bigg| \sum_{s = 1}^{t} X_s \bigg| \geq \delta_1 \right\} \leq 2\exp(-E_{\ref{Hoeffding's nequality}}\delta_1^2),
\end{align}
where $E_{\ref{Hoeffding's nequality}}$ is a constant that is independent of $\delta_1$ but dependent on the sub-Gaussianity parameters.
\begin{proof}
This is the direct result of Theorem 2.6.2 in \cite{vershynin2018high}.
\end{proof}
\end{lemma}

\begin{lemma}
\label{Bernstein's nequality}
(Bernstein's inequality) For a sequence of independent, zero-mean, sub-exponential random variables $X_1,\dots,X_t$, we have
\begin{align}
    \mathbb{P}\left\{ \frac{1}{t}\sum_{s = 1}^{t} X_s \geq \delta_2 \right\} \leq 2\exp(-\min\{E_{\ref{Bernstein's nequality},1}\delta_2,E_{\ref{Bernstein's nequality},2}\delta^2_2\}).
\end{align}
where $E_{\ref{Bernstein's nequality},1}$, $E_{\ref{Bernstein's nequality},2}$ are constants that are independent of $\delta_2$ but dependent on the sub-exponentiality parameters.
\begin{proof}
This is the direct result of Theorem 2.8.1 in \cite{vershynin2018high}.
\end{proof}
\end{lemma}

\begin{lemma}
\label{lemma:upper bound of r}
For all $t \in [T_0,T]$, with probability at least $1 - \frac{1}{t^2}$, we have
\begin{align}
    \bigg|\sum_{s = 1}^{t} r_s\bigg| &\leq E_{\ref{lemma:upper bound of r},1}t\sqrt{\log{t}}, \label{eq:upper bound of r eq-1}\\ 
    \sum_{s=1}^{t} r_s^2 &\leq E_{\ref{lemma:upper bound of r},2}t\log{t}, \label{eq:upper bound of r eq-2}
\end{align}
where $T_0 = \inf\limits_{s} \left\{ s \in \mathbb{N}:E_{\ref{Bernstein's nequality},1}\bigg(\frac{2\log (2s)}{E_{\ref{Bernstein's nequality},1}}\bigg) \leq E_{\ref{Bernstein's nequality},2}\bigg(\frac{2\log (2s)}{E_{\ref{Bernstein's nequality},1}}\bigg)^2, s \geq 2 \right\}$ and $E_{\ref{lemma:upper bound of r},1},E_{\ref{lemma:upper bound of r},2}$ are absolute constants.
\vspace{-5mm}
\begin{proof}
By Lemma \ref{Hoeffding's nequality} and  letting $\delta_1 = \sqrt{\frac{2\log{(2t)}}{E_{\ref{Hoeffding's nequality}}}}$, for the sub-Gaussian noise $\{\epsilon_s\}_{s=1}^{t}$ which are assumed to be conditionally independent given the contexts $\{\bx_1,\cdots, \bx_t\}$, we have
\begin{align}
    \mathbb{P}\left\{\frac{1}{t}\bigg|\sum_{s = 1}^{t} \epsilon_s \bigg|\leq \sqrt{\frac{2\log{(2t)}}{E_{\ref{Hoeffding's nequality}}}} \right\} &\geq 1 - 2\text{exp}(-E_{\ref{Hoeffding's nequality}}\frac{2\log{(2t)}}{E_{\ref{Hoeffding's nequality}}}) \label{eq:upper bound of r eq-3}\\
    & = 1 - \frac{1}{2t^2} \label{eq:upper bound of r eq-4}.
\end{align}
Then, choosing $E_{\ref{lemma:upper bound of r},1} = \sqrt{\frac{4}{E_{\ref{Hoeffding's nequality}}}}$ and $T_0 \geq 2 $, we have proved (\ref{eq:upper bound of r eq-1}).
Similarly, letting $\delta_2 = \frac{2\log (2t)}{E_{\ref{Bernstein's nequality},1}}$, for all $t$ that satisfies \begin{equation}
    E_{\ref{Bernstein's nequality},1}\bigg(\frac{2\log (2t)}{E_{\ref{Bernstein's nequality},1}}\bigg) \leq E_{\ref{Bernstein's nequality},2}\bigg(\frac{2\log (2t)}{E_{\ref{Bernstein's nequality},1}}\bigg)^2, \label{eq:upper bound of r eq-5}
\end{equation}
we have
\begin{align}
    \mathbb{P}\left\{ \frac{1}{t}\sum_{s = 1}^{t} \epsilon^2_s \leq \delta_2 \right\} \geq & 1 -  2\text{exp}(-\min\{E_{\ref{Bernstein's nequality},1}\delta_2, E_{\ref{Bernstein's nequality},2}\delta^2_2\}) \label{eq:upper bound of r eq-6}\\
    = & 1 - 2\text{exp}(-E_{\ref{Bernstein's nequality},1}\frac{2\log(2t)}{E_{\ref{Bernstein's nequality},1}}) \label{eq:upper bound of r eq-7} \\
    = & 1 - \frac{1}{2t^2} \label{eq:upper bound of r eq-8},
\end{align}
where (\ref{eq:upper bound of r eq-6}) holds by Lemma \ref{Bernstein's nequality} with $\{\epsilon_s^2\}^t_{s=1}$ are sub-exponential random variables. (\ref{eq:upper bound of r eq-7}) holds by (\ref{eq:upper bound of r eq-5}). Then, choosing $E_{\ref{lemma:upper bound of r},2 } = \frac{4}{E_{\ref{Bernstein's nequality},1}}$ and $T_0 \geq 2$, we have proved (\ref{eq:upper bound of r eq-2}). By combining (\ref{eq:upper bound of r eq-4}) and (\ref{eq:upper bound of r eq-8}), we complete the proof by the boundedness of $h(\cdot)$.
\end{proof}

\end{lemma}

\section{Supporting Lemmas for the Proof of Theorem \ref{theorem:regret of NeuralRBMLE-GA}}
Define 
\begin{equation}
    \mathcal{L}^{\dagger}_{t,a}(\btheta) := \sum_{s=1}^{t-1} \Big( b(f(\bx_s;\btheta)) - r_sf(\bx_s;\btheta) \Big) + \frac{m\lambda}{2} \left\Vert \btheta - \btheta_{0} \right\Vert_2^2 - \alpha(t)f(\bx_{t,a};\btheta) \label{eq:loss function of GA}
\end{equation}
as the loss function we used in Algorithm $\ref{alg:NeuralRBMLE-GA}$. Then,
according to the update rule of gradient descent with step size $\eta$ and total number of steps $J$ used in Algorithm \ref{alg:NeuralRBMLE-GA}, we formally define the parameter $\widetilde{\btheta}^{\dagger}_{t,a}$ here. Define 
\begin{align}
    \widetilde{\btheta}^{\dagger}_{t,a} &= \widetilde{\btheta}^{\dagger(J)}_{t,a} \\
    \widetilde{\btheta}^{\dagger(j)}_{t,a} &= \widetilde{\btheta}^{\dagger(j-1)}_{t,a} - \eta\nabla_{\btheta}\mathcal{L}^{\dagger}_{t,a}(\widetilde{\btheta}^{\dagger(j-1)}_{t,a}) \label{eq: gd rule of theta tilde dagger}\\
    \widetilde{\btheta}^{\dagger(0)}_{t,a} &= \btheta_{0}.
\end{align}
Then, the following lemma shows the upper bound of the log-likelihood under $\widetilde{\btheta}^{\dagger(j)}_{t,a}$.

\begin{lemma}
\label{lemma:upper bound of likelihood}
There exist a positive constant $E_{\ref{lemma:upper bound of likelihood},1}$ and $E_{\ref{lemma:upper bound of likelihood},2}$, which are independent to $m$ and $t$, such that for all $j \in [J]$, $a \in [K]$, $t \in [T]$ and under
\begin{align}
     m \geq  E_{\ref{appendix_lemma_1}}^{6} E_{\ref{lemma:error between widetilde_theta and theta}}^6 T^7\lambda^{-8}L^{18}(\log{T})^{-4}(\log m)^3, \label{eq:upper bound of likelihood eq-1}
\end{align}
if $\alpha(t) = \sqrt{t}$ and $\widetilde{\btheta}^{\dagger (j)}_{t,a}$ that satisfies 
\begin{align}
    \left\Vert \widetilde{\btheta}^{\dagger (j)}_{t,a} - \btheta_0 \right\Vert_{2} \leq E_{\ref{lemma:upper bound of likelihood},1}\sqrt{\frac{t\log{t}}{m\lambda^2}}, \label{eq:upper bound of likelihood eq-2}
\end{align}
then we have 
\begin{equation}
    \sum_{s=1}^{t} \Big( b(f(\bx_s;\widetilde{\btheta}^{\dagger (j)}_{t,a})) - r_sf(\bx_s;\widetilde{\btheta}^{\dagger (j)}_{t,a}) \Big)\leq E_{\ref{lemma:upper bound of likelihood},2}t\sqrt{\log{t}} \label{eq:upper bound of likelihood eq-3}
\end{equation}
\begin{proof}
According to the gradient update rule, for all $j \in [J]$, we have
\begin{align}
    \mathcal{L}^{\dagger}_{t,a}(\widetilde{\btheta}^{\dagger (j)}_{t,a}) - \frac{m\lambda}{2} \left\Vert \widetilde{\btheta}^{\dagger (j)}_{t,a} - \btheta_0 \right\Vert_2^2 
    \leq  \mathcal{L}^{\dagger}_{t,a}(\widetilde{\btheta}^{\dagger (j)}_{t,a})
    \leq  \mathcal{L}^{\dagger}_{t,a}(\btheta_0) 
    =  \sum_{s=1}^{t-1}  b(0)  \label{eq:upper bound of likelihood eq-4},
\end{align}
where the equality holds due to $f(\bx_{t,a};\btheta_0) = 0$ for all $a \in [K]$, $t \in [T]$.
Rearranging (\ref{eq:upper bound of likelihood eq-4}), we have
\begin{align}
    & \;\sum_{s=1}^{t-1} \Big( b(f(\bx_s;\widetilde{\btheta}^{\dagger (j)}_{t,a})) - r_sf(\bx_s;\widetilde{\btheta}^{\dagger (j)}_{t,a}) \Big) \nonumber\\
    \leq & \;\sum_{s=1}^{t-1}  b(0) + \alpha(t)\Big( |\bg(\bx_{t,a};\btheta_0)^\top(\widetilde{\btheta}^{\dagger (j)}_{t,a} -\btheta_0)| + E_{\ref{appendix_lemma_1}} \Big( E_{\ref{lemma:error between widetilde_theta and theta}}\sqrt{\frac{t\log{t}}{m\lambda^2}}\Big)^{\frac{4}{3}}L^3\sqrt{m\log{m}}\Big) \label{eq:upper bound of likelihood eq-5} \\
    \leq & \;\sum_{s=1}^{t-1}  b(0) + \alpha(t) E_{\ref{lemma:gradient of f}}\sqrt{mL} E_{\ref{lemma:error between widetilde_theta and theta}}\sqrt{\frac{t\log{t}}{m\lambda^2}} + \alpha(t)E_{\ref{appendix_lemma_1}} \Big( E_{\ref{lemma:error between widetilde_theta and theta}}\sqrt{\frac{t\log{t}}{m\lambda^2}}\Big)^{\frac{4}{3}}L^3\sqrt{m\log{m}}\label{eq:upper bound of likelihood eq-6} \\
    = &\; (t-1)b(0) + \alpha(t) E_{\ref{lemma:gradient of f}}E_{\ref{lemma:error between widetilde_theta and theta}}\sqrt{\frac{Lt\log{t}}{\lambda^2}} + \alpha(t)E_{\ref{appendix_lemma_1}}\Big( E_{\ref{lemma:error between widetilde_theta and theta}}\sqrt{\frac{t\log{t}}{\lambda^2}}\Big)^{\frac{4}{3}}L^3m^{-\frac{1}{6}}\sqrt{\log m} \label{eq:upper bound of likelihood eq-7},
\end{align}
where (\ref{eq:upper bound of likelihood eq-5}) holds by substituting the reward-biased term by Lemma \ref{appendix_lemma_1}, and (\ref{eq:upper bound of likelihood eq-6}) holds due to Lemma \ref{lemma:gradient of f} and Cauchy–Schwarz inequality. Notice that (\ref{eq:upper bound of likelihood eq-1}) satisfies the conditions of $m$ in both Lemma \ref{appendix_lemma_1} and Lemma \ref{lemma:gradient of f}. Then, by letting $\alpha(t) = \sqrt{t}$, $m \geq  E_{\ref{appendix_lemma_1}}^{6} E_{\ref{lemma:error between widetilde_theta and theta}}^6 T^7\lambda^{-8}L^{18}(\log{T})^{-4}(\log m)^3$ and $E_{\ref{lemma:upper bound of likelihood},2} = 3\max\{ b(0), 2E_{\ref{lemma:gradient of f}}E_{\ref{lemma:error between widetilde_theta and theta}}\sqrt{\frac{L}{\lambda^2}},1\}$, we complete the proof.
\end{proof}
\end{lemma}

\begin{lemma}
\label{lemma:upper bound of b'}
For all $j \in [J]$, $a \in [K]$, $t \in [T]$ and $m$ satisfies that $m \geq  \widehat{C}_1 T^7\lambda^{-7}L^{21}(\log m)^3$, we have 
\begin{equation}
    \sum_{s=1}^{t} | b'(f(\bx_s;\widetilde{\btheta}^{\dagger(j)}_{t,a})) | \leq E_{\ref{lemma:upper bound of b'}}t\sqrt{\log{t}} \label{eq:upper bound of b' eq-1}
\end{equation}
where $E_{\ref{lemma:upper bound of b'}}$ is an absolute constant.
\begin{proof}
\normalfont
First, we start by obtaining an upper bound of $\sum_{s=1}^{t} |f(\bx_s;\widetilde{\btheta}^{\dagger(j)}_{t,a}) |$. By Lemma \ref{lemma:upper bound of likelihood}, we have 
\begin{align}
    E_{\ref{lemma:upper bound of likelihood}}t\sqrt{\log{t}} \geq & \sum_{s=1}^{t} \Big( b(f(\bx_s;\widetilde{\btheta}^{\dagger(j)}_{t,a})) - r_sf(\bx_s;\widetilde{\btheta}^{\dagger(j)}_{t,a}) \Big)   \label{eq:upper bound of b' eq-2} \\
    \geq &\sum_{s=1}^{t} \Big( b\big(f(\bx_s;\btheta_0)\big) + b'\big(f(\bx_s;\btheta_0)\big)\big(f(\bx_s;\widetilde{\btheta}^{\dagger(j)}_{t,a}) - f(\bx_s;\btheta_0)\big) \nonumber \\ 
    & + \frac{L_b}{2}\big(f(\bx_s;\widetilde{\btheta}^{\dagger(j)}_{t,a}) - f(\bx_s;\btheta_0)\big)^2- r_sf(\bx_s;\widetilde{\btheta}^{\dagger(j)}_{t,a}) \Big)  \label{eq:upper bound of b' eq-3} \\
    = &\sum_{s=1}^{t} \Big( b(0) + b'(0)f(\bx_s;\widetilde{\btheta}^{\dagger(j)}_{t,a}) + \frac{L_b}{2}f(\bx_s;\widetilde{\btheta}^{\dagger(j)}_{t,a})^2- r_sf(\bx_s;\widetilde{\btheta}^{\dagger(j)}_{t,a}) \Big) \label{eq:upper bound of b' eq-4}  \\
    = &\sum_{s=1}^{t} \Big( \frac{L_b}{2} \big(f(\bx_s;\widetilde{\btheta}^{\dagger(j)}_{t,a}) + \frac{b'(0) - r_s}{L_b}\big)^2 + b(0) - \frac{(b'(0) - r_s)^2}{2L_b} \Big) \label{eq:upper bound of b' eq-5} ,
\end{align}
where (\ref{eq:upper bound of b' eq-3}) holds because $b'(\cdot)$ is $L_{b}$-strongly convex, (\ref{eq:upper bound of b' eq-4}) holds by $f(\bx;\btheta_0) = 0$ for all vector $\bx$.
Then, reorganizing the terms in (\ref{eq:upper bound of b' eq-5}) and applying Lemma \ref{lemma:upper bound of r} to find bounds for the terms that involve $r_s$, with probability at least $1-\frac{1}{t^2}$, we have
\begin{align}
    E_{\ref{lemma:upper bound of b'},1}t\log{t} \geq \sum_{s=1}^{t}  \big(f(\bx_s;\widetilde{\btheta}^{\dagger(j)}_{t,a}) + \frac{b'(0) - r_s}{L_b}\big)^2, \label{eq:upper bound of b' eq-6} 
\end{align}
where $ E_{\ref{lemma:upper bound of b'},1} = E_{\ref{lemma:upper bound of likelihood}} + b(0) + \frac{b'(0)^{2}}{L_b} + \frac{E_{\ref{lemma:upper bound of r},1}|b'(0)|}{L_b} + \frac{E_{\ref{lemma:upper bound of r},2}}{2L_b}$. Applying Cauchy–Schwarz inequality to (\ref{eq:upper bound of b' eq-6}), we have
\begin{align}
E_{\ref{lemma:upper bound of b'},1}t\log{t} \cdot \sum_{s=1}^{t} 1^2  &\geq   \sum_{s=1}^{t}  \big(f(\bx_s;\widetilde{\btheta}^{\dagger(j)}_{t,a}) + \frac{b'(0) - r_s}{L_b}\big)^2 \cdot \sum_{s=1}^{t} 1^2 \label{eq:upper bound of b' eq-7} \\
& \geq \Big( \sum_{s=1}^{t}  \big(f(\bx_s;\widetilde{\btheta}^{\dagger(j)}_{t,a}) + \frac{b'(0) - r_s}{L_b}\big) \Big)^2 \label{eq:upper bound of b' eq-8}.
\end{align}
Reorganizing (\ref{eq:upper bound of b' eq-8}) and applying Lemma \ref{lemma:upper bound of r}, we have
\begin{align}
    \sum_{s=1}^{t} |f(\bx_s;\widetilde{\btheta}^{\dagger(j)}_{t,a})| \leq E_{\ref{lemma:upper bound of b'},2}t \sqrt{\log{t}} \label{eq:upper bound of b' eq-9},
\end{align}
where $E_{\ref{lemma:upper bound of b'},2} = E_{\ref{lemma:upper bound of b'},1} + \frac{b'(0)}{L_b} + \frac{E_{\ref{lemma:upper bound of r},1}}{L_b}$ is a positive constant. Then, by Taylor's Theorem and the fact that $b''(z) \leq U_b$, for all $z$, we have 
\begin{align}
    \sum_{s=1}^{t} | b'(f(\bx_s;\widetilde{\btheta}^{\dagger(j)}_{t,a})) | &\leq \sum_{s=1}^{t} | b'(f(\bx_s;\btheta_0)) | + \sum_{s=1}^{t}  U_b | f(\bx_s;\widetilde{\btheta}^{\dagger(j)}_{t,a})  - f(\bx_s; \btheta_0)| \label{eq:upper bound of b' eq-10} \\
    &= \sum_{s=1}^{t} | b'(0) |+ \sum_{s=1}^{t}  U_b|f(\bx_s;\widetilde{\btheta}^{\dagger(j)}_{t,a}) | \label{eq:upper bound of b' eq-11}\\
    & \leq E_{\ref{lemma:upper bound of b'}}t\sqrt{\log{t}} \label{eq:upper bound of b' eq-12},
\end{align}
where (\ref{eq:upper bound of b' eq-11}) hold by $f(\bx; \btheta_0) = 0$ for all vector $\bx$, (\ref{eq:upper bound of b' eq-12}) holds by (\ref{eq:upper bound of b' eq-9}) and choosing $E_{\ref{lemma:upper bound of b'}} = b'(0) + U_bE_{\ref{lemma:upper bound of b'},2}$.
Then, we complete the proof.
\end{proof}
\end{lemma}

Define
\begin{align}
    \ell_{\lambda}(\mathcal{F}_t;\btheta) &:= \sum_{s=1}^{t-1} \Big( r_s\langle \bg(\bx_s;\btheta_0), \btheta - \btheta_0 \rangle - b(\langle \bg(\bx_s;\btheta_0), \btheta - \btheta_0 \rangle) \Big) - \frac{m\lambda}{2}\left\Vert \btheta - \btheta_0 \right\Vert_2^2 \label{def:log-likeihood under ntk}\\
    \mathcal{L}_{t,a}(\btheta) &:=  -\ell_{\lambda}(\mathcal{F}_t;\btheta) -  \alpha(t)\langle \bg(\bx_{t,a};\btheta_0), \btheta - \btheta_0 \rangle \label{def:loss function under ntk}
\end{align}
as the loss function under the approximation in the NTK regime. Then, given the update rule of gradient descent for $\mathcal{L}_{t,a}(\btheta)$ with step size $\eta$ and total number of steps $J$, we define the useful parameters as follows:
\begin{align}
    \widetilde{\btheta}_{t,a} &= \widetilde{\btheta}^{(J)}_{t,a} \\
    \widetilde{\btheta}^{(j)}_{t,a} &= \widetilde{\btheta}^{(j-1)}_{t,a} - \eta\nabla_{\btheta}\mathcal{L}_{t,a}(\widetilde{\btheta}^{(j-1)}_{t,a}) \label{eq: gd rule of theta tilde}\\
    \widetilde{\btheta}^{(0)}_{t,a} &= \btheta_0
\end{align}

The next lemma given an upper bound of the term $\left\Vert \widetilde{\btheta}^{(j)}_{t,a} - \btheta_0 \right\Vert_2$ for all $j \in [J]$.
\begin{lemma}
\label{lemma:error between widetilde_theta and theta}
For all $j \in [J]$, $t \in [T]$ and $a \in [K]$, we have
\begin{align}
    \left\Vert \widetilde{\btheta}^{(j)}_{t,a}  - \btheta_0 \right\Vert_2 \leq E_{\ref{lemma:error between widetilde_theta and theta}}\sqrt{\frac{t\log{t}}{m\lambda^2}}
\end{align}
\end{lemma}
\begin{proof}
We start by considering a collection of scalar variables $\{x_s\}_{s=1}^{t}$ and studying the function $ \sum_{s=1}^{t-1} \big( b(x_s) -r_sx_s\big)$.
For each $x_s$, define $x_s^*:=\argmin_x \{b(x)-r_s\}$. Due to the strong convexity of the log-likelihood of an exponential family, it is easy to show that
\begin{align}
    x_s^{*} = \argmin_x \{b(x) - r_sx\} = b'^{-1}(r_s).
    \label{eq:error between widetilde_theta and theta eq-1}
\end{align}
Then, for all $x_s \in \{x_s\}_{s=1}^{t-1}$, we have
\begin{align}
    \sum_{s=1}^{t-1} \Big( b(x_s) -r_sx_s\Big) & \geq \sum_{s=1}^{t-1} \Big( b(x^{*}_s) -r_sx^{*}_s\Big) \label{eq:error between widetilde_theta and theta eq-2}\\
    & = \sum_{s=1}^{t-1} \Big( b(b'^{-1}(r_s)) -r_sb'^{-1}(r_s)\Big) \label{eq:error between widetilde_theta and theta eq-3}\\
    & \geq (t-1)  b\Big(\frac{1}{t-1} \sum_{s=1}^{t-1}b'^{-1}(r_s)\Big) - \sum_{s=1}^{t-1}r_sb'^{-1}(r_s) \label{eq:error between widetilde_theta and theta eq-4} \\
    & \geq E_{\ref{lemma:error between widetilde_theta and theta},1}t - \frac{1}{L_b}E_{\ref{lemma:upper bound of r},2}t\log{t}  \label{eq:error between widetilde_theta and theta eq-5} \\
    & \geq -E_{\ref{lemma:error between widetilde_theta and theta},2}t\log{t}, \label{eq:error between widetilde_theta and theta eq-6}
\end{align}
where (\ref{eq:error between widetilde_theta and theta eq-2})-(\ref{eq:error between widetilde_theta and theta eq-3}) hold by (\ref{eq:error between widetilde_theta and theta eq-1}), (\ref{eq:error between widetilde_theta and theta eq-4}) holds by Jensen's inequality, (\ref{eq:error between widetilde_theta and theta eq-5}) holds by the fact that $b(\cdot)$ is a strongly convex function, Lemma \ref{lemma:upper bound of r} and that $E_{\ref{lemma:error between widetilde_theta and theta},2}$ is a positive constant. Based on this, we can derive an upper bound for the regularization term of $\mathcal{L}_{t,a}(\widetilde{\btheta}_{t,a}^{(j)})$, 
\begin{align}
    \frac{m\lambda}{2}\left\Vert \widetilde{\btheta}^{(j)}_{t,a} - \btheta_0 \right\Vert_2^2 \leq & \sum_{s=1}^{t-1} \Big( b(\langle \bg(\bx_s;\btheta_0), \widetilde{\btheta}_{t,a}^{(j)} - \btheta_0 \rangle) - r_s\langle \bg(\bx_s;\btheta_0), \widetilde{\btheta}_{t,a}^{(j)} - \btheta_0 \rangle \Big)  \nonumber \\
    &+ \frac{m\lambda}{2}\left\Vert \widetilde{\btheta}^{(j)}_{t,a} - \btheta_0 \right\Vert_2^2 - \alpha(t)\langle \bg(\bx_{t,a};\btheta_0),\widetilde{\btheta}_{t,a}^{(j)} - \btheta_0 \rangle \nonumber\\
    & + E_{\ref{lemma:error between widetilde_theta and theta},2}t\log{t} + \alpha(t)\left\Vert \bg(\bx_{t,a};\btheta_0) \right\Vert_2 \cdot \left\Vert \widetilde{\btheta}_{t,a}^{(j)} - \btheta_0 \right\Vert_2
    \label{eq:error between widetilde_theta and theta eq-7} \\
    \leq & \;\mathcal{L}_{t,a}(\widetilde{\btheta}_{t,a}^{(j)}) + E_{\ref{lemma:error between widetilde_theta and theta},2}t\log{t} +  E_{\ref{lemma:gradient of f}}\sqrt{tmL}\left\Vert \widetilde{\btheta}_{t,a}^{(j)} - \btheta_0 \right\Vert_2 \label{eq:error between widetilde_theta and theta eq-8} \\
    \leq & \;\mathcal{L}_{t,a}(\btheta_0) + E_{\ref{lemma:error between widetilde_theta and theta},2}t\log{t} +  E_{\ref{lemma:gradient of f}}\sqrt{tmL}\left\Vert \widetilde{\btheta}_{t,a}^{(j)} - \btheta_0 \right\Vert_2 \label{eq:error between widetilde_theta and theta eq-9},
\end{align}
where (\ref{eq:error between widetilde_theta and theta eq-7}) holds by (\ref{eq:error between widetilde_theta and theta eq-6}) and  Cauchy–Schwarz inequality, (\ref{eq:error between widetilde_theta and theta eq-8}) holds by Lemma \ref{lemma:gradient of f}, and (\ref{eq:error between widetilde_theta and theta eq-9}) holds by the rule of gradient descent. Then, by solving the quadratic equation induced by (\ref{eq:error between widetilde_theta and theta eq-9}), we have
\begin{align}
    \left\Vert \widetilde{\btheta}_{t,a}^{(j)} - \btheta_0 \right\Vert_2 &  \leq \frac{E_{\ref{lemma:gradient of f}}\sqrt{tmL}}{m\lambda} + \sqrt{\frac{2\mathcal{L}_{t,a}(\btheta_0)}{m\lambda}+\frac{2E_{\ref{lemma:error between widetilde_theta and theta},2}t\log{t}}{m\lambda} + \frac{E^2_{\ref{lemma:gradient of f}}tL}{m\lambda^2}} \\
    & \leq E_{\ref{lemma:error between widetilde_theta and theta}}\sqrt{\frac{t\log{t}}{m\lambda^2}},
\end{align}

where $E_{\ref{lemma:error between widetilde_theta and theta}} = 4\max\{E_{\ref{lemma:gradient of f}}\sqrt{L},2b(0),2E_{\ref{lemma:error between widetilde_theta and theta},2},E^2_{\ref{lemma:gradient of f}}L\}$ is a positive constant.
\end{proof}
The next lemma gives an upper bound of the term $\left\Vert \widetilde{\btheta}^{\dagger(j)}_{t,a} - \btheta_0 \right\Vert_2$, for each $j \in [J]$.
\begin{lemma}
\label{lemma:error between widetilde_theta_dagger and theta_0}
For all $j \in [J]$, $t \in [T]$ and $a \in [K]$, we have
\begin{align}
    \left\Vert \widetilde{\btheta}^{\dagger(j)}_{t,a}  - \btheta_0 \right\Vert_2 \leq E_{\ref{lemma:error between widetilde_theta_dagger and theta_0}}\sqrt{\frac{t\log{t}}{m\lambda^2}} \label{eq:error between widetilde_theta_dagger and theta_0 eq-0}
\end{align}
where $E_{\ref{lemma:error between widetilde_theta_dagger and theta_0}} = 2E_{\ref{lemma:error between widetilde_theta and theta}}$ is a positive constant. 
\end{lemma}
\begin{proof}
Let's prove this lemma by induction. It is straightforward to verify that (\ref{eq:error between widetilde_theta_dagger and theta_0 eq-0}) holds if $j=0$. Next, we assume that $\left\Vert \widetilde{\btheta}^{\dagger(j)}_{t,a}  - \btheta_0 \right\Vert_2 \leq E_{\ref{lemma:error between widetilde_theta_dagger and theta_0}}\sqrt{\frac{t\log{t}}{m\lambda^2}}$ holds. By plugging (\ref{eq: gd rule of theta tilde dagger}) and (\ref{eq: gd rule of theta tilde}), we have
\begin{align}
    & \left\Vert \widetilde{\btheta}^{\dagger(j+1)}_{t,a}  - \widetilde{\btheta}^{(j+1)}_{t,a}\right\Vert_2 \nonumber\\
    = & \big\Vert \widetilde{\btheta}^{\dagger(j)}_{t,a}  - \widetilde{\btheta}^{(j)}_{t,a} - \eta \sum_{s=1}^{t-1} \bg(\bx_s;\widetilde{\btheta}^{\dagger(j)}_{t,a})\Big(b'(f(\bx_s;\widetilde{\btheta}^{\dagger(j)}_{t,a})) - r_s \Big)\nonumber\\
    & + \eta \sum_{s=1}^{t-1} \bg(\bx_s;\btheta_0)\Big(b'(\langle \bg(\bx_s;\btheta_0), \widetilde{\btheta}^{(j)}_{t,a} - \btheta_0 \rangle{}{}) - r_s \Big) \nonumber\\ 
    & - \eta m \lambda(\widetilde{\btheta}^{\dagger(j)}_{t,a}  - \widetilde{\btheta}^{(j)}_{t,a}) + \eta\alpha(t)\Big(\bg(\bx_{t,a};\widetilde{\btheta}^{\dagger(j)}_{t,a}) - \bg(\bx_{t,a};\btheta_0)\Big)
    \big\Vert_2 \label{eq:error between widetilde_theta_dagger and theta_0 eq-1} \\
     = & \big\Vert (1-\eta m \lambda)(\widetilde{\btheta}^{\dagger(j)}_{t,a}  - \widetilde{\btheta}^{(j)}_{t,a}) - \eta \sum_{s=1}^{t-1}\Big( \bg(\bx_s;\widetilde{\btheta}^{\dagger(j)}_{t,a}) - \bg(\bx_s;\btheta_0)\Big) \Big(b'(f(\bx_s;\widetilde{\btheta}^{\dagger(j)}_{t,a})) - r_s \Big)\nonumber\\
    & - \eta \sum_{s=1}^{t-1} \bg(\bx_s;\btheta_0)\Big(b'(f(\bx_s;\widetilde{\btheta}^{\dagger(j)}_{t,a})) - b'(\langle \bg(\bx_s;\btheta_0), \widetilde{\btheta}^{(j)}_{t,a} - \btheta_0 \rangle)\Big) \nonumber\\ 
    &  + \eta\alpha(t)(\bg(\bx_{t,a};\widetilde{\btheta}^{\dagger(j)}_{t,a}) - \bg(\bx_{t,a};\btheta_0)) 
    \big\Vert_2 \label{eq:error between widetilde_theta_dagger and theta_0 eq-2}
\end{align}
For the term $b'(f(\bx_s;\widetilde{\btheta}^{\dagger(j)}_{t,a})) - b'(\langle \bg(\bx_s;\btheta_0), \widetilde{\btheta}^{(j)}_{t,a} - \btheta_0 \rangle)$, by mean-value theorem, there exists a set of positive constants $\{E_{\ref{lemma:error between widetilde_theta_dagger and theta_0},1,s}\}_{s=1}^{t-1}$ with $E_{10,1,s} \in [L_b,U_b]$ for every $s\in \{1,\cdots, t-1\}$, such that
\begin{align}
    & b'(f(\bx_s;\widetilde{\btheta}^{\dagger(j)}_{t,a})) - b'(\langle \bg(\bx_s;\btheta_0), \widetilde{\btheta}^{(j)}_{t,a} - \btheta_0 \rangle) = E_{\ref{lemma:error between widetilde_theta_dagger and theta_0},1,s}\Big( f(\bx_s;\widetilde{\btheta}^{\dagger(j)}_{t,a}) - \langle \bg(\bx_s;\btheta_0), \widetilde{\btheta}^{(j)}_{t,a} - \btheta_0 \rangle \Big), \label{eq:error between widetilde_theta_dagger and theta_0 eq-3}
\end{align}
where (\ref{eq:error between widetilde_theta_dagger and theta_0 eq-3}) hold by $L_b \leq b''(\cdot) \leq U_b$. Then, plugging (\ref{eq:error between widetilde_theta_dagger and theta_0 eq-3}) into (\ref{eq:error between widetilde_theta_dagger and theta_0 eq-2}), we have
\begin{align}
    &\left\Vert \widetilde{\btheta}^{\dagger(j+1)}_{t,a}  - \widetilde{\btheta}^{(j+1)}_{t,a}\right\Vert_2 \nonumber \\
    = & \big\Vert (1-\eta m \lambda)(\widetilde{\btheta}^{\dagger(j)}_{t,a}  - \widetilde{\btheta}^{(j)}_{t,a}) - \eta \sum_{s=1}^{t-1}\Big( \bg(\bx_s;\widetilde{\btheta}^{\dagger(j)}_{t,a}) - \bg(\bx_s;\btheta_0)\Big) \Big(b'(f(\bx_s;\widetilde{\btheta}^{\dagger(j)}_{t,a})) - r_s \Big)\nonumber\\
    & - \eta \sum_{s=1}^{t-1}E_{\ref{lemma:error between widetilde_theta_dagger and theta_0},1,s} \bg(\bx_s;\btheta_0)\Big(f(\bx_s;\widetilde{\btheta}^{\dagger(j)}_{t,a}) - \langle \bg(\bx_s;\btheta_0), \widetilde{\btheta}^{\dagger(j)}_{t,a} - \btheta_0\rangle+ \langle \bg(\bx_s;\btheta_0), \widetilde{\btheta}^{\dagger(j)}_{t,a} - \widetilde{\btheta}^{(j)}_{t,a}\rangle \Big) \nonumber\\ 
    &  + \eta\alpha(t)(\bg(\bx_{t,a};\widetilde{\btheta}^{\dagger(j)}_{t,a}) - \bg(\bx_{t,a};\btheta_0) \big\Vert_2 \label{eq:error between widetilde_theta_dagger and theta_0 eq-4} \\ 
    \leq & \underbrace{\left\Vert [\bI - \eta(m\lambda\bI + \sum_{s=1}^{t-1}E_{\ref{lemma:error between widetilde_theta_dagger and theta_0},1,s}\bg(\bx_s;\btheta_0)\bg(\bx_s;\btheta_0)^\top)](\widetilde{\btheta}^{\dagger(j)}_{t,a}  - \widetilde{\btheta}^{(j)}_{t,a}) \right\Vert_2}_{:=I_1} \nonumber\\
    & + \underbrace{\eta\left\Vert \sum_{s=1}^{t-1}\Big( \bg(\bx_s;\widetilde{\btheta}^{\dagger(j)}_{t,a}) - \bg(\bx_s;\btheta_0)\Big) \Big(b'(f(\bx_s;\widetilde{\btheta}^{\dagger(j)}_{t,a})) - r_s \Big) \right\Vert_2}_{:=I_2} \nonumber\\
    & + \underbrace{\eta \left\Vert \sum_{s=1}^{t-1} \bg(\bx_s;\btheta_0)\Big(E_{\ref{lemma:error between widetilde_theta_dagger and theta_0},1,s}(f(\bx_s;\widetilde{\btheta}^{\dagger(j)}_{t,a}) - \langle \bg(\bx_s;\btheta_0), \widetilde{\btheta}^{\dagger(j)}_{t,a} - \btheta_0\rangle)\Big)\right\Vert_2}_{:=I_3} \nonumber \\
    & + \underbrace{\eta\alpha(t) \left\Vert \bg(\bx_{t,a};\widetilde{\btheta}^{\dagger(j)}_{t,a}) - \bg(\bx_{t,a};\btheta_0)
    \right\Vert_2 }_{:=I_4} , \label{eq:error between widetilde_theta_dagger and theta_0 eq-5}
\end{align}

where (\ref{eq:error between widetilde_theta_dagger and theta_0 eq-5}) holds by triangle inequality. For $I_1$, we have
\begin{align}
    I_1 = & \left\Vert [\bI - \eta(m\lambda\bI + \sum_{s=1}^{t-1}E_{\ref{lemma:error between widetilde_theta_dagger and theta_0},1,s}\bg(\bx_s;\btheta_0)\bg(\bx_s;\btheta_0)^\top)](\widetilde{\btheta}^{\dagger(j)}_{t,a}  - \widetilde{\btheta}^{(j)}_{t,a}) \right\Vert_2 \nonumber\\
    \leq & \left\Vert \bI - \eta(m\lambda\bI + \sum_{s=1}^{t-1}E_{\ref{lemma:error between widetilde_theta_dagger and theta_0},1,s}\bg(\bx_s;\btheta_0)\bg(\bx_s;\btheta_0)^\top) \right\Vert_2 \cdot \left\Vert \widetilde{\btheta}^{\dagger(j)}_{t,a}  - \widetilde{\btheta}^{(j)}_{t,a} \right\Vert_2 \label{eq:error between widetilde_theta_dagger and theta_0 eq-6} \\
    \leq & \left\Vert \bI - \eta(m\lambda\bI + U_btE_{\ref{lemma:gradient of f}}^2mL) \right\Vert_2 \cdot \left\Vert \widetilde{\btheta}^{\dagger(j)}_{t,a}  - \widetilde{\btheta}^{(j)}_{t,a} \right\Vert_2 \label{eq:error between widetilde_theta_dagger and theta_0 eq-7} \\
    \leq & (1 - \eta m\lambda) \left\Vert \widetilde{\btheta}^{\dagger(j)}_{t,a}  - \widetilde{\btheta}^{(j)}_{t,a} \right\Vert_2 \label{eq:error between widetilde_theta_dagger and theta_0 eq-8},
\end{align}
where (\ref{eq:error between widetilde_theta_dagger and theta_0 eq-6}) holds by the definition of spectral norm,  (\ref{eq:error between widetilde_theta_dagger and theta_0 eq-7}) holds by Lemma \ref{lemma:gradient of f},  $E_{\ref{lemma:error between widetilde_theta_dagger and theta_0},1,s} \leq U_b$ for all $s \in [t]$, and eigenvalue property of the spectral norm. (\ref{eq:error between widetilde_theta_dagger and theta_0 eq-8}) holds by the choice of $\eta$ with $\eta \leq (m\lambda + \widehat{C}_3tmL)^{-1}$ and $\widehat{C}_3 = E^2_{\ref{lemma:gradient of f}}U_b$. For $I_2$, we have
\begin{align}
    I_2 = & \; \eta\left\Vert \sum_{s=1}^{t-1}\Big( \bg(\bx_s;\widetilde{\btheta}^{\dagger(j)}_{t,a}) - \bg(\bx_s;\btheta_0)\Big) \Big(b'(f(\bx_s;\widetilde{\btheta}^{\dagger(j)}_{t,a})) - r_s \Big) \right\Vert_2 \label{eq:error between widetilde_theta_dagger and theta_0 eq-9}\\
    \leq & \eta \sum_{s=1}^{t-1} \Big|b'(f(\bx_s;\widetilde{\btheta}^{\dagger(j)}_{t,a})) - r_s \Big| \cdot \left\Vert  \bg(\bx_s;\widetilde{\btheta}^{\dagger(j)}_{t,a}) - \bg(\bx_s;\btheta_0)\right\Vert_2 \label{eq:error between widetilde_theta_dagger and theta_0 eq-10}\\
    \leq & \;\eta \sum_{s=1}^{t-1} \Big|b'(f(\bx_s;\widetilde{\btheta}^{\dagger(j)}_{t,a})) - r_s \Big|E_{\ref{appendix_lemma_2}}\sqrt{m\log m}\Big(E_{\ref{lemma:error between widetilde_theta_dagger and theta_0}}\sqrt{\frac{t\log{t}}{m\lambda^2}} \Big)^\frac{1}{3}L^{\frac{7}{2}} \label{eq:error between widetilde_theta_dagger and theta_0 eq-11}\\
    \leq & \eta E_{\ref{lemma:error between widetilde_theta_dagger and theta_0},2}t^{\frac{7}{6}}m^{\frac{1}{3}}(\log{t})^{\frac{2}{3}}(\log{m})^{\frac{1}{2}}\lambda^{-\frac{1}{3}}L^{\frac{7}{2}} \label{eq:error between widetilde_theta_dagger and theta_0 eq-12},
\end{align}
where (\ref{eq:error between widetilde_theta_dagger and theta_0 eq-11}) holds by Lemma \ref{appendix_lemma_2}, (\ref{eq:error between widetilde_theta_dagger and theta_0 eq-12}) holds by Lemma \ref{lemma:upper bound of r}, Lemma \ref{lemma:upper bound of b'} and $E_{\ref{lemma:error between widetilde_theta_dagger and theta_0},2} := E_{\ref{lemma:error between widetilde_theta_dagger and theta_0}}^{\frac{1}{3}}E_{\ref{appendix_lemma_2}}(E_{\ref{lemma:upper bound of r},1} + E_{\ref{lemma:upper bound of b'}})$.
For $I_3$, we have 
\begin{align}
    I_3 = & \eta  \left\Vert \sum_{s=1}^{t-1} \bg(\bx_s;\btheta_0)\Big(E_{\ref{lemma:error between widetilde_theta_dagger and theta_0},1,s}(f(\bx_s;\widetilde{\btheta}^{\dagger(j)}_{t,a}) - \langle \bg(\bx_s;\btheta_0), \widetilde{\btheta}^{\dagger(j)}_{t,a} - \btheta_0\rangle)\Big)\right\Vert_2 \label{eq:error between widetilde_theta_dagger and theta_0 eq-13}\\
    \leq & \eta  U_b \left\Vert\sum_{s=1}^{t-1}\Big|f(\bx_s;\widetilde{\btheta}^{\dagger(j)}_{t,a}) - \langle \bg(\bx_s;\btheta_0), \widetilde{\btheta}^{\dagger(j)}_{t,a} - \btheta_0\rangle\Big|\cdot \bg(\bx_s;\btheta_0)\right\Vert_2 \label{eq:error between widetilde_theta_dagger and theta_0 eq-14}\\
    \leq & \eta U_b E_{\ref{appendix_lemma_1}} \Big(E_{\ref{lemma:error between widetilde_theta_dagger and theta_0}}\sqrt{\frac{t\log{t}}{m\lambda^2}} \Big)^{\frac{4}{3}} L^3\sqrt{m\log m} \cdot\sum_{s=1}^{t-1}\left\Vert\bg(\bx_s;\btheta_0)\right\Vert_2 \label{eq:error between widetilde_theta_dagger and theta_0 eq-15}\\
    \leq & \eta U_b E_{\ref{appendix_lemma_1}} \Big(E_{\ref{lemma:error between widetilde_theta_dagger and theta_0}}\sqrt{\frac{t\log{t}}{m\lambda^2}} \Big)^{\frac{4}{3}} L^3\sqrt{m\log m}\cdot E_{\ref{lemma:gradient of f}}t\sqrt{mL}\label{eq:error between widetilde_theta_dagger and theta_0 eq-16}\\
    \leq & \eta E_{\ref{lemma:error between widetilde_theta_dagger and theta_0},3}t^{\frac{5}{3}}m^{\frac{1}{3}}(\log{t})^{\frac{2}{3}}(\log{m})^{\frac{1}{2}}\lambda^{-\frac{4}{3}}L^{\frac{7}{2}} \label{eq:error between widetilde_theta_dagger and theta_0 eq-17},
\end{align}
where (\ref{eq:error between widetilde_theta_dagger and theta_0 eq-14}) holds by $E_{\ref{lemma:error between widetilde_theta_dagger and theta_0},1,s} \leq U_b$ for all $s \in [t]$, (\ref{eq:error between widetilde_theta_dagger and theta_0 eq-15}) holds by Lemma \ref{appendix_lemma_1} and triangle inequality, (\ref{eq:error between widetilde_theta_dagger and theta_0 eq-16}) holds by Lemma \ref{lemma:gradient of f}, and (\ref{eq:error between widetilde_theta_dagger and theta_0 eq-17}) holds by $E_{\ref{lemma:error between widetilde_theta_dagger and theta_0},3} := E_{\ref{lemma:error between widetilde_theta_dagger and theta_0}}^{\frac{4}{3}}U_bE_{\ref{appendix_lemma_1}}E_{\ref{lemma:gradient of f}}$.
For $I_4$, we have
\begin{align}
    I_4 = & \eta\alpha(t) \left\Vert \bg(\bx_{t,a};\widetilde{\btheta}^{\dagger(j)}_{t,a}) - \bg(\bx_{t,a};\btheta_0)
    \right\Vert_2 \\
    \leq & \eta\alpha(t)E_{\ref{appendix_lemma_2}}\sqrt{m\log m}\Big(E_{\ref{lemma:error between widetilde_theta_dagger and theta_0}}\sqrt{\frac{t\log{t}}{m\lambda^2}} \Big)^\frac{1}{3}L^{\frac{7}{2}} \label{eq:error between widetilde_theta_dagger and theta_0 eq-18} \\
    = & \eta E_{\ref{lemma:error between widetilde_theta_dagger and theta_0},4}t^{\frac{2}{3}}m^{\frac{1}{3}}(\log{t})^{\frac{1}{6}}(\log{m})^{\frac{1}{2}}\lambda^{-\frac{1}{3}}L^{\frac{7}{2}} \label{eq:error between widetilde_theta_dagger and theta_0 eq-19},
\end{align}
where (\ref{eq:error between widetilde_theta_dagger and theta_0 eq-18}) holds by Lemma \ref{appendix_lemma_2}, (\ref{eq:error between widetilde_theta_dagger and theta_0 eq-19}) holds by $\alpha(t) = \sqrt{t}$, and $E_{\ref{lemma:error between widetilde_theta_dagger and theta_0},4} := E_{\ref{lemma:error between widetilde_theta_dagger and theta_0}}^{\frac{1}{3}}E_{\ref{appendix_lemma_2}}$. Therefore, combining (\ref{eq:error between widetilde_theta_dagger and theta_0 eq-8}), (\ref{eq:error between widetilde_theta_dagger and theta_0 eq-12}), (\ref{eq:error between widetilde_theta_dagger and theta_0 eq-17}), and (\ref{eq:error between widetilde_theta_dagger and theta_0 eq-19}), we have 
\begin{align}
    \big\Vert & \widetilde{\btheta}^{\dagger(j+1)}_{t,a}  - \widetilde{\btheta}^{(j+1)}_{t,a}\big\Vert_2 \nonumber \\
    \leq & (1 - \eta m\lambda) \big\Vert \widetilde{\btheta}^{\dagger(j)}_{t,a}  - \widetilde{\btheta}^{(j)}_{t,a} \big\Vert_2 + \sum_{s = 2}^{4}I_s \label{eq:error between widetilde_theta_dagger and theta_0 eq-20}\\
    \leq &  E_{\ref{lemma:error between widetilde_theta_dagger and theta_0},5}t^{\frac{5}{3}}m^{-\frac{2}{3}}(\log{t})^{\frac{2}{3}}(\log{m})^{\frac{1}{2}}\lambda^{-\frac{7}{3}}L^{\frac{7}{2}}, \label{eq:error between widetilde_theta_dagger and theta_0 eq-21}
\end{align}
where $E_{\ref{lemma:error between widetilde_theta_dagger and theta_0},5} = 3\max\{E_{\ref{lemma:error between widetilde_theta_dagger and theta_0},2},E_{\ref{lemma:error between widetilde_theta_dagger and theta_0},3},E_{\ref{lemma:error between widetilde_theta_dagger and theta_0},4}\}$. Notice that (\ref{eq:error between widetilde_theta_dagger and theta_0 eq-21}) holds by the formula of geometric series and the fact that $\lambda \leq 1$. Then, by triangle inequality, we have 
\begin{align}
     \left\Vert \widetilde{\btheta}^{\dagger(j+1)}_{t,a}  - \btheta_0 \right\Vert_2 & \leq \left\Vert \widetilde{\btheta}^{\dagger(j+1)}_{t,a}  - \widetilde{\btheta}^{(j+1)}_{t,a}\right\Vert_2 + \left\Vert \widetilde{\btheta}^{(j+1)}_{t,a}  - \btheta_0 \right\Vert_2 \label{eq:error between widetilde_theta_dagger and theta_0 eq-22}\\
    & \leq  E_{\ref{lemma:error between widetilde_theta_dagger and theta_0},5}t^{\frac{5}{3}}m^{-\frac{2}{3}}(\log{t})^{\frac{2}{3}}(\log{m})^{\frac{1}{2}}\lambda^{-\frac{7}{3}}L^{\frac{7}{2}} + E_{\ref{lemma:error between widetilde_theta and theta}}\sqrt{\frac{t\log{t}}{m\lambda^2}} \label{eq:error between widetilde_theta_dagger and theta_0 eq-23} \\
    & \leq  2E_{\ref{lemma:error between widetilde_theta and theta}}\sqrt{\frac{t\log{t}}{m\lambda^2}} \label{eq:error between widetilde_theta_dagger and theta_0 eq-24} \\
    & =  E_{\ref{lemma:error between widetilde_theta_dagger and theta_0}}\sqrt{\frac{t\log{t}}{m\lambda^2}}, \label{eq:error between widetilde_theta_dagger and theta_0 eq-25}
\end{align}
where (\ref{eq:error between widetilde_theta_dagger and theta_0 eq-23}) holds by (\ref{eq:error between widetilde_theta_dagger and theta_0 eq-21}) and Lemma \ref{lemma:error between widetilde_theta and theta}, (\ref{eq:error between widetilde_theta_dagger and theta_0 eq-24}) holds by choosing an $m$ that satisfies 
\begin{align}
    m \geq E_{\ref{lemma:error between widetilde_theta and theta}}^{-6}E^6_{\ref{lemma:error between widetilde_theta_dagger and theta_0},5}T^{7}(\log{T})(\log{m})^{3}\lambda^{-8}L^{21}, \label{eq:error between widetilde_theta_dagger and theta_0 eq-26}
\end{align}
and (\ref{eq:error between widetilde_theta_dagger and theta_0 eq-25}) holds by letting $E_{\ref{lemma:error between widetilde_theta_dagger and theta_0}}:= 2E_{\ref{lemma:error between widetilde_theta and theta}}$. By mathematical induction, we complete the proof.
\end{proof}

\section{Regret Bound of NeuralRBMLE-GA}
\label{appendix:Regret Bound of NeuralRBMLE-GA}
Recall that
\begin{align}
    \ell_{\lambda}(\mathcal{F}_t;\btheta) &= \sum_{s=1}^{t-1} \bigg( r_s\langle \bg(\bx_s;\btheta_0), \btheta - \btheta_0 \rangle - b(\langle \bg(\bx_s;\btheta_0), \btheta - \btheta_0 \rangle) \bigg) - \frac{m\lambda}{2}\left\Vert \btheta - \btheta_0 \right\Vert_2^2 \\
    \Bar{\bZ}_t & = \lambda\cdot \bI + \frac{1}{m}\sum_{s=1}^{t-1} \bg(\bx_s;\btheta_0)\bg(\bx_s;\btheta_0)^\top \\
    \btheta_{t,a} &:= \argmax_{\btheta} \left\{ \ell_{\lambda}(\mathcal{F}_t;\btheta) + \alpha(t)\langle \bg(\bx_{t,a};\btheta_0), \btheta - \btheta_0 \rangle \right\} \\
    \mathcal{I}_{t,a} &:=  \ell_{\lambda}(\mathcal{F}_t;\btheta_{t,a}) +  \alpha(t)\zeta(t)\langle \bg(\bx_{t,a};\btheta_0), \btheta_{t,a} - \btheta_0 \rangle
\end{align}

To begin with, we introduce the following several useful lemmas specific to the regret bound of NeuralRBMLE-GA. The following lemma shows the difference between the model parameter $\widetilde{\btheta}^{\dagger}_{t,a}$,  which we actually used in index comparison, and $\btheta_{t,a}$, which is the optimal solution to $\min\{\mathcal{L}_{t,a}(\btheta)\}$.
\begin{lemma}
\label{lemma:error between theta_tilde_dagger and theta}
For all $a\in [K], t\in[T]$, we have
\begin{align}
    \left\Vert \widetilde{\btheta}^{\dagger}_{t,a} - \btheta_{t,a} \right\Vert_2 \leq E_{\ref{lemma:error between widetilde_theta_dagger and theta_0},5}t^{\frac{5}{3}}m^{-\frac{2}{3}}(\log{t})^{\frac{2}{3}}(\log{m})^{\frac{1}{2}}\lambda^{-\frac{7}{3}}L^{\frac{7}{2}} + E^J_{\ref{lemma:error between theta_tilde_dagger and theta},2} E_{\ref{lemma:error between theta_tilde_dagger and theta},3}\sqrt{\frac{t\log{t}}{m\lambda^2}} \label{eq:error between theta_tilde_dagger and theta eq-0}
\end{align}
\begin{proof}
Note that $\mathcal{L}_{t,a}(\btheta)$ is $m\lambda$-strongly convex and $E_{\ref{lemma:error between theta_tilde_dagger and theta},1}(tmL+m\lambda)$-smooth for an absolute constant $E_{\ref{lemma:error between theta_tilde_dagger and theta},6}$. 
By the convergence theorem of gradient descent for a strongly convex and smooth function (Theorem 5, \cite{nesterov2013gradient}), if $\eta \leq \frac{2}{m\lambda + E_{\ref{lemma:error between theta_tilde_dagger and theta},1}(tmL+m\lambda)}$, we have 
\begin{align}
    & \; \left\Vert \widetilde{\btheta}_{t,a} - \btheta_{t,a} \right\Vert_2 \nonumber \\
    \leq & \; \bigg(\frac{\frac{m\lambda}{E_{\ref{lemma:error between theta_tilde_dagger and theta},1}(tmL+m\lambda)} - 1}{\frac{m\lambda}{E_{\ref{lemma:error between theta_tilde_dagger and theta},1}(tmL+m\lambda)} + 1}\bigg)^J \left\Vert \btheta_{t,a} - \btheta_0 \right\Vert_2 \label{eq:error between theta_tilde_dagger and theta eq-1}\\
    = & \; E^J_{\ref{lemma:error between theta_tilde_dagger and theta},2} \left\Vert \btheta_{t,a} - \btheta_0 \right\Vert_2 , \label{eq:error between theta_tilde_dagger and theta eq-2}
\end{align}
where $E_{\ref{lemma:error between theta_tilde_dagger and theta},2} := \bigg(\frac{\frac{m\lambda}{E_{\ref{lemma:error between theta_tilde_dagger and theta},6}(tmL+m\lambda)} - 1}{\frac{m\lambda}{E_{\ref{lemma:error between theta_tilde_dagger and theta},6}(tmL+m\lambda)} + 1}\bigg) \leq 1$.
For the term $\left\Vert \btheta_{t,a} - \btheta_0 \right\Vert_2$, we have
\begin{align}
    \frac{m\lambda}{2}\left\Vert \btheta_{t,a} - \btheta_0 \right\Vert^2_2 & \leq \mathcal{L}_{t,a}(\btheta_{t,a})  + E_{\ref{lemma:error between widetilde_theta and theta},2}t\log{t} + \alpha(t)\left\Vert\bg(\bx_{t,a};\btheta_0)\right\Vert_2 \cdot \left\Vert \btheta_{t,a} - \btheta_0 \right\Vert_2 \label{eq:error between theta_tilde_dagger and theta eq-3} \\
    &  \leq \mathcal{L}_{t,a}(\btheta_0)  + E_{\ref{lemma:error between widetilde_theta and theta},2}t\log{t} + E_{\ref{lemma:gradient of f}}\sqrt{tmL} \left\Vert \btheta_{t,a} - \btheta_0 \right\Vert_2 \label{eq:error between theta_tilde_dagger and theta eq-4}\\
    &  \leq E_{\ref{lemma:error between widetilde_theta and theta},3}t\log{t} + E_{\ref{lemma:gradient of f}}\sqrt{tmL} \left\Vert \btheta_{t,a} - \btheta_0 \right\Vert_2 \label{eq:error between theta_tilde_dagger and theta eq-5},
\end{align}
where (\ref{eq:error between theta_tilde_dagger and theta eq-3}) holds by (\ref{eq:error between widetilde_theta and theta eq-6}) and Cauchy–Schwarz inequality, (\ref{eq:error between theta_tilde_dagger and theta eq-5}) holds by the rule of gradient descent, Lemma \ref{lemma:gradient of f}, $\alpha(t) = \sqrt{t}$, and (\ref{eq:error between theta_tilde_dagger and theta eq-5}) holds by the definition $E_{\ref{lemma:error between widetilde_theta and theta},3} := 2\max\{b(0), E_{\ref{lemma:error between widetilde_theta and theta},2} \}$. By solving (\ref{eq:error between theta_tilde_dagger and theta eq-5}), we have
\begin{align}
    \left\Vert \btheta_{t,a} - \btheta_0 \right\Vert_2 \leq E_{\ref{lemma:error between theta_tilde_dagger and theta},3}\sqrt{\frac{t\log{t}}{m\lambda^2}} \label{eq:error between theta_tilde_dagger and theta eq-6}
\end{align}
Combining (\ref{eq:error between theta_tilde_dagger and theta eq-2}) and (\ref{eq:error between theta_tilde_dagger and theta eq-5}), we have
\begin{align}
    \left\Vert \widetilde{\btheta}_{t,a} - \btheta_{t,a} \right\Vert_2 \leq E^J_{\ref{lemma:error between theta_tilde_dagger and theta},2} E_{\ref{lemma:error between theta_tilde_dagger and theta},3}\sqrt{\frac{t\log{t}}{m\lambda^2}} \label{eq:error between theta_tilde_dagger and theta eq-7}
\end{align}
Then, by triangle inequality, we have
\begin{align}
    \left\Vert \widetilde{\btheta}^{\dagger}_{t,a} - \btheta_{t,a} \right\Vert_2 & \leq \left\Vert \widetilde{\btheta}^{\dagger}_{t,a} - \widetilde{\btheta}_{t,a} \right\Vert_2 + \left\Vert \widetilde{\btheta}_{t,a} - \btheta_{t,a} \right\Vert_2 \label{eq:error between theta_tilde_dagger and theta eq-8}\\
    & \leq E_{\ref{lemma:error between widetilde_theta_dagger and theta_0},5}t^{\frac{5}{3}}m^{-\frac{2}{3}}(\log{t})^{\frac{2}{3}}(\log{m})^{\frac{1}{2}}\lambda^{-\frac{7}{3}}L^{\frac{7}{2}} + E^J_{\ref{lemma:error between theta_tilde_dagger and theta},2} E_{\ref{lemma:error between theta_tilde_dagger and theta},3}\sqrt{\frac{t\log{t}}{m\lambda^2}}, \label{eq:error between theta_tilde_dagger and theta eq-9},
\end{align}
where (\ref{eq:error between theta_tilde_dagger and theta eq-9}) holds by (\ref{eq:error between widetilde_theta_dagger and theta_0 eq-21}) and (\ref{eq:error between theta_tilde_dagger and theta eq-7}). The proof is complete.
\end{proof}
\end{lemma}

\begin{lemma}
\label{lemma:index policy for NeuralRBMLE-GA}
For all $a\in [K], t\in[T]$, we have
\begin{align}
    &\bigg| \mathcal{I}_{t,a} - \ell^{\dagger}_\lambda(\mathcal{F}_t;\widetilde{\btheta}^{\dagger}_{t,a}) - \alpha(t)\zeta(t)f(\bx_{t,a};\widetilde{\btheta}^{\dagger}_{t,a}) \bigg| \nonumber\\
    \leq & \left(E_{\ref{lemma:index policy for NeuralRBMLE-GA},1}t\log{t} + \alpha(t)\zeta(t)\right)\Bigg( E_{\ref{appendix_lemma_1}} \left(E_{\ref{lemma:error between theta_tilde_dagger and theta},3}\sqrt{\frac{t\log{t}}{m\lambda^2}}\right)^{\frac{4}{3}}L^3\sqrt{m\log{m}}  \nonumber \\
    & \;\;\;\;+ E_{\ref{lemma:gradient of f}}\sqrt{mL}\left( \Eeleven \right) \Bigg) \nonumber \\
    & + U_bt\Bigg( E_{\ref{appendix_lemma_1}} (E_{\ref{lemma:error between theta_tilde_dagger and theta},3}\sqrt{\frac{t\log{t}}{m\lambda^2}})^{\frac{4}{3}}L^3\sqrt{m\log{m}}  \nonumber \\
    & \;\;\;\;+ E_{\ref{lemma:gradient of f}}\sqrt{mL}\left( \Eeleven \right) \Bigg)^2 \nonumber\\
    & + E_{\ref{lemma:index policy for NeuralRBMLE-GA},2}\frac{m\lambda}{2} \sqrt{\frac{t\log{t}}{m\lambda^2}}\Bigg( \Eeleven \Bigg) := E_{\ref{lemma:index policy for NeuralRBMLE-GA}}(m,t) \label{eq:index policy for NeuralRBMLE-GA eq-1} 
\end{align}
\begin{proof}
\normalfont
By the definition of $\mathcal{I}_{t, a}$, we have
\begin{align}
    & \bigg| \mathcal{I}_{t,a} - \ell^{\dagger}_\lambda(\mathcal{F};\widetilde{\btheta}^{\dagger}_{t,a}) - \alpha(t)\zeta(t)f(\bx_{t,a};\widetilde{\btheta}^{\dagger}_{t,a}) \bigg| \nonumber\\
    \leq & \underbrace{\bigg| \ell_{\lambda}(\mathcal{F}_t;\btheta_{t,a}) -  \ell^{\dagger}_\lambda(\mathcal{F};\widetilde{\btheta}^{\dagger}_{t,a})\bigg|}_{:=K_1} + \underbrace{\alpha(t)\zeta(t)\bigg| f(\bx_{t,a};\widetilde{\btheta}^{\dagger}_{t,a}) - \langle \bg(\bx_{t,a};\btheta_0), \btheta_{t,a} - \btheta_0 \rangle \bigg|}_{:=K_2} \label{eq:index policy for NeuralRBMLE-GA eq-2},
\end{align}
where (\ref{eq:index policy for NeuralRBMLE-GA eq-2}) holds by triangle inequality. For $K_1$, we have
\begin{align}
    K_1 = & \bigg|  \ell_{\lambda}(\mathcal{F}_t;\btheta_{t,a}) -  \ell^{\dagger}_\lambda(\mathcal{F};\widetilde{\btheta}^{\dagger}_{t,a})\bigg|  \label{eq:index policy for NeuralRBMLE-GA eq-3} \\
    \leq & \bigg| \ell(\mathcal{F}_t;\btheta_{t,a}) -  \ell^{\dagger}(\mathcal{F};\widetilde{\btheta}^{\dagger}_{t,a}) \bigg| + \frac{m\lambda}{2} \bigg| \left\Vert \btheta_{t,a} - \btheta_0 \right\Vert^2_2 -  \left\Vert \widetilde{\btheta}^{\dagger}_{t,a} - \btheta_0 \right\Vert^2_2 \bigg| \label{eq:index policy for NeuralRBMLE-GA eq-4}
    ,
\end{align}
where (\ref{eq:index policy for NeuralRBMLE-GA eq-4}) holds by triangle inequality. Then, for $\bigg| \ell(\mathcal{F}_t;\btheta_{t,a}) -  \ell^{\dagger}(\mathcal{F};\widetilde{\btheta}^{\dagger}_{t,a}) \bigg|$, we have
\begin{align}
    & \bigg| \ell(\mathcal{F}_t;\btheta_{t,a}) -  \ell^{\dagger}(\mathcal{F};\widetilde{\btheta}^{\dagger}_{t,a}) \bigg| \nonumber\\
    = & \bigg| \sum_{s=1}^{t-1} \bigg( b(\bg(\bx_{s};\btheta_{0})^\top(\btheta_{t,a} - \btheta_0)) - b(f(\bx_s;\widetilde{\btheta}^{\dagger}_{t,a})) - r_s(\bg(\bx_{s};\btheta_{0})^\top(\btheta_{t,a} - \btheta_0) - f(\bx_s;\widetilde{\btheta}^{\dagger}_{t,a}))\bigg)\bigg| \\
    \leq & \bigg| \sum_{s=1}^{t-1} \bigg( b'(\bg(\bx_{s};\btheta_{0})^\top(\btheta_{t,a} - \btheta_0) - f(\bx_s;\widetilde{\btheta}^{\dagger}_{t,a})) - r_s(\bg(\bx_{s};\btheta_{0})^\top(\btheta_{t,a} - \btheta_0) - f(\bx_s;\widetilde{\btheta}^{\dagger}_{t,a}))\bigg)\bigg| \nonumber \\
    & + U_b\sum_{s=1}^{t-1}  \bigg(\bg(\bx_{s};\btheta_{0})^\top(\btheta_{t,a} - \btheta_0) - f(\bx_s;\widetilde{\btheta}^{\dagger}_{t,a})\bigg)^2 \label{eq:index policy for NeuralRBMLE-GA eq-5} \\
    \leq & E_{\ref{lemma:index policy for NeuralRBMLE-GA},1}t\log{t} \bigg| \bg(\bx_{s};\btheta_{0})^\top(\btheta_{t,a} - \btheta_0) - f(\bx_s;\widetilde{\btheta}^{\dagger}_{t,a}) \bigg|  + U_b\sum_{s=1}^{t-1}  \bigg(\bg(\bx_{s};\btheta_{0})^\top(\btheta_{t,a} - \btheta_0) - f(\bx_s;\widetilde{\btheta}^{\dagger}_{t,a})\bigg)^2 \label{eq:index policy for NeuralRBMLE-GA eq-6},
\end{align}
where (\ref{eq:index policy for NeuralRBMLE-GA eq-5}) holds by Taylor's theorem and $b(\cdot)$ is a $U_b$-smooth function and (\ref{eq:index policy for NeuralRBMLE-GA eq-6}) holds by introducing $E_{\ref{lemma:index policy for NeuralRBMLE-GA},1} := 2\max\{U_b,E_{\ref{lemma:upper bound of r},2}\}$. Notice that for all vector $\bx$ and $a \in [K], t\in [T]$, we have
\begin{align}
    & \bigg| \bg(\bx;\btheta_{0})^\top(\btheta_{t,a} - \btheta_0) - f(\bx;\widetilde{\btheta}^{\dagger}_{t,a}) \bigg| \nonumber \\
    \leq & E_{\ref{appendix_lemma_1}} (E_{\ref{lemma:error between theta_tilde_dagger and theta},3}\sqrt{\frac{t\log{t}}{m\lambda^2}})^{\frac{4}{3}}L^3\sqrt{m\log{m}}  \nonumber \\
    & + E_{\ref{lemma:gradient of f}}\sqrt{mL}( \Eeleven ) \label{eq:index policy for NeuralRBMLE-GA eq-7},
\end{align}
where (\ref{eq:index policy for NeuralRBMLE-GA eq-7}) holds by Lemma \ref{lemma:gradient of f} and Lemma \ref{lemma:error between theta_tilde_dagger and theta}.
For $\frac{m\lambda}{2} \bigg| \left\Vert \btheta_{t,a} - \btheta_0 \right\Vert^2_2 -  \left\Vert \widetilde{\btheta}^{\dagger}_{t,a} - \btheta_0 \right\Vert^2_2 \bigg|$, we have
\begin{align}
    &\frac{m\lambda}{2} \bigg| \left\Vert \btheta_{t,a} - \btheta_0 \right\Vert^2_2 -  \left\Vert \widetilde{\btheta}^{\dagger}_{t,a} - \btheta_0 \right\Vert^2_2 \bigg| \nonumber \\
    = & \frac{m\lambda}{2}\bigg|(\left\Vert \btheta_{t,a} - \btheta_0 \right\Vert_2 +  \left\Vert \widetilde{\btheta}^{\dagger}_{t,a} - \btheta_0 \right\Vert_2)(\left\Vert \btheta_{t,a} - \btheta_0 \right\Vert_2 -  \left\Vert \widetilde{\btheta}^{\dagger}_{t,a} - \btheta_0 \right\Vert_2)\bigg| \label{eq:index policy for NeuralRBMLE-GA eq-8}\\
    \leq & \frac{m\lambda}{2}\bigg|(\left\Vert \btheta_{t,a} - \btheta_0 \right\Vert_2 +  \left\Vert \widetilde{\btheta}^{\dagger}_{t,a} - \btheta_0 \right\Vert_2)\left\Vert \btheta_{t,a} - \widetilde{\btheta}^{\dagger}_{t,a} \right\Vert_2\bigg| \label{eq:index policy for NeuralRBMLE-GA eq-9}\\
    \leq & E_{\ref{lemma:index policy for NeuralRBMLE-GA},2}\frac{m\lambda}{2} \sqrt{\frac{t\log{t}}{m\lambda^2}}\bigg( \Eeleven \bigg) \label{eq:index policy for NeuralRBMLE-GA eq-10}
\end{align}
where (\ref{eq:index policy for NeuralRBMLE-GA eq-8}) holds by $a^2 - b^2 = (a+b)(a-b)$, (\ref{eq:index policy for NeuralRBMLE-GA eq-9}) holds by triangle inequality, and (\ref{eq:index policy for NeuralRBMLE-GA eq-10}) holds by Lemma \ref{lemma:error between widetilde_theta_dagger and theta_0},  Lemma \ref{lemma:error between theta_tilde_dagger and theta} and $E_{\ref{lemma:index policy for NeuralRBMLE-GA},2} = 2\max\{E_{\ref{lemma:error between widetilde_theta_dagger and theta_0}},E_{\ref{lemma:error between theta_tilde_dagger and theta},3}\}$. Combining (\ref{eq:index policy for NeuralRBMLE-GA eq-6}), (\ref{eq:index policy for NeuralRBMLE-GA eq-7}) and (\ref{eq:index policy for NeuralRBMLE-GA eq-10}), we have
\begin{align}
   K_1 \leq & E_{\ref{lemma:index policy for NeuralRBMLE-GA},1}t\log{t}\bigg( E_{\ref{appendix_lemma_1}} (E_{\ref{lemma:error between theta_tilde_dagger and theta},3}\sqrt{\frac{t\log{t}}{m\lambda^2}})^{\frac{4}{3}}L^3\sqrt{m\log{m}}  \nonumber \\
    & \;\;\;\;+ E_{\ref{lemma:gradient of f}}\sqrt{mL}( \Eeleven ) \bigg) \nonumber \\
    & + U_bt\bigg( E_{\ref{appendix_lemma_1}} (E_{\ref{lemma:error between theta_tilde_dagger and theta},3}\sqrt{\frac{t\log{t}}{m\lambda^2}})^{\frac{4}{3}}L^3\sqrt{m\log{m}}  \nonumber \\
    & \;\;\;\;+ E_{\ref{lemma:gradient of f}}\sqrt{mL}( \Eeleven ) \bigg)^2 \nonumber \\
    & + E_{\ref{lemma:index policy for NeuralRBMLE-GA},2}\frac{m\lambda}{2} \sqrt{\frac{t\log{t}}{m\lambda^2}}\bigg( \Eeleven \bigg) \label{eq:index policy for NeuralRBMLE-GA eq-11}.
\end{align}
For $K_2$, by (\ref{eq:index policy for NeuralRBMLE-GA eq-7}) we have
\begin{align}
    K_2 = & \alpha(t)\zeta(t)\bigg| f(\bx_{t,a};\widetilde{\btheta}^{\dagger}_{t,a}) - \langle \bg(\bx_{t,a};\btheta_0), \btheta_{t,a} - \btheta_0 \rangle \bigg| \label{eq:index policy for NeuralRBMLE-GA eq-}\\
    \leq & \alpha(t)\zeta(t)\bigg(E_{\ref{appendix_lemma_1}} (E_{\ref{lemma:error between theta_tilde_dagger and theta},3}\sqrt{\frac{t\log{t}}{m\lambda^2}})^{\frac{4}{3}}L^3\sqrt{m\log{m}}  \nonumber \\
    & + E_{\ref{lemma:gradient of f}}\sqrt{mL}( \Eeleven ) \bigg) \label{eq:index policy for NeuralRBMLE-GA eq-12}
    .
\end{align}
Then, combining (\ref{eq:index policy for NeuralRBMLE-GA eq-11}) and (\ref{eq:index policy for NeuralRBMLE-GA eq-12}), we complete the proof.
\end{proof}
\end{lemma}

\begin{lemma}
\label{lemma:get the negetive of |g*|}
For all arm $i,j\in [K]$ and $t \in [T]$, there exists $\btheta' = \kappa\btheta_{t,i} + (1-\kappa)\btheta_{t,j}$ for some $\kappa \in (0,1)$, we have
\begin{align}
     0 = &  \Big(\bg(\bx_{t,i};\btheta_{0}) + \bg(\bx_{t,j};\btheta_{0})\Big)^\top (\btheta_{t,j} - \btheta_{t,i}) + \frac{\alpha(t)}{m} \Big( \left\Vert \bg(\bx_{t,i};\btheta_{0}) \right\Vert^2_{\bU^{-1}_t} - \left\Vert \bg(\bx_{t,j};\btheta_{0}) \right\Vert^2_{\bU^{-1}_t} \Big) \label{eq:get the negetive of |g*| eq-1},
\end{align}
where 
\begin{align}
    \bU_0 := \lambda\cdot\bI + \frac{1}{m}\sum_{s=1}^{t-1}b''\big( \langle \bg(\bx_s;\btheta_0), \btheta' - \btheta_0 \rangle \big)\bg(\bx_s;\btheta_0)\bg(\bx_s;\btheta_0)^\top .\label{eq:get the negetive of |g*| eq-1.5},
\end{align}
\begin{proof}
\normalfont
By the first-order necessary condition of $\btheta_{t,i}$ and $\btheta_{t,i}$ for any two arm $i,j \in [K]$, we have 
\begin{align}
    &\sum_{s=1}^{t-1}\Big( r_s\bg(\bx_s;\btheta_0) - b'\big(\langle \bg(\bx_s;\btheta_0), \btheta_{t,i} - \btheta_0 \rangle\big)\bg(\bx_s;\btheta_0)\Big) - m\lambda(\btheta_{t,i} - \btheta_0)+  \alpha(t)\bg(\bx_{t,i};\btheta_{0}) = 0 \label{eq:get the negetive of |g*| eq-2}, \\
    &\sum_{s=1}^{t-1}\Big( r_s\bg(\bx_s;\btheta_0) - b'\big(\langle \bg(\bx_s;\btheta_0), \btheta_{t,j} - \btheta_0 \rangle\big)\bg(\bx_s;\btheta_0)\Big) -  m\lambda(\btheta_{t,j} - \btheta_0) + \alpha(t)\bg(\bx_{t,j};\btheta_{0}) = 0 \label{eq:get the negetive of |g*| eq-3}.
\end{align}
After multiplying both sides of (\ref{eq:get the negetive of |g*| eq-2})-(\ref{eq:get the negetive of |g*| eq-3}) by $\big( \bg(\bx_{t,i};\btheta_{0}) + \bg(\bx_{t,j};\btheta_{0})\big)^\top \bU^{-1}_t$ and taking the difference between (\ref{eq:get the negetive of |g*| eq-2})-(\ref{eq:get the negetive of |g*| eq-3}), we have
\begin{align}
    0 = & \; \big(\bg(\bx_{t,i};\btheta_{0}) + \bg(\bx_{t,j};\btheta_{0})\big)^\top \bU^{-1}_t \cdot \bigg(  \alpha(t)\bg(\bx_{t,i};\btheta_{0}) - \alpha(t)\bg(\bx_{t,j};\btheta_{t,j}) \nonumber \\ 
    & + m\lambda(\btheta_{t,j} - \btheta_{t,i}) + \sum_{s=1}^{t-1}\Big( b'\big(\langle \bg(\bx_s;\btheta_0), \btheta_{t,j} - \btheta_0 \rangle\big) - b'\big(\langle \bg(\bx_s;\btheta_0), \btheta_{t,i} - \btheta_0 \rangle\big) \Big) \bg(\bx_s;\btheta_0) \bigg) \label{eq:get the negetive of |g*| eq-4} \\
    = & \; \frac{\alpha(t)}{m}\Big( \left\Vert \bg(\bx_{t,i};\btheta_{0}) \right\Vert^2_{\bU^{-1}_t} - \left\Vert \bg(\bx_{t,j};\btheta_{0}) \right\Vert^2_{\bU^{-1}_t} \Big) + \big(\bg(\bx_{t,i};\btheta_{0})  + \bg(\bx_{t,j};\btheta_{0})\big)^\top \bU^{-1}_t \nonumber \\
    & \cdot \Big( m\lambda\cdot\bI + \sum_{s=1}^{t-1}b''\big( \langle \bg(\bx_s;\btheta_0), \btheta' - \btheta_0 \rangle \big)\bg(\bx_s;\btheta_0)\bg(\bx_s;\btheta_0)^\top  \Big)(\btheta_{t,j} - \btheta_{t,i}) \label{eq:get the negetive of |g*| eq-5} \\
     = & \; \frac{\alpha(t)}{m} \Big( \left\Vert \bg(\bx_{t,i};\btheta_{0}) \right\Vert^2_{\bU^{-1}_t} - \left\Vert \bg(\bx_{t,j};\btheta_{0}) \right\Vert^2_{\bU^{-1}_t} \Big)  + \big(\bg(\bx_{t,i};\btheta_{0}) + \bg(\bx_{t,j};\btheta_{0})\big)^\top (\btheta_{t,j} - \btheta_{t,i}) \label{eq:get the negetive of |g*| eq-6},
\end{align}
where (\ref{eq:get the negetive of |g*| eq-5}) holds by mean-value theorem with $\btheta' = \kappa\btheta_{t,i} + (1-\kappa)\btheta_{t,j}$ for some $\kappa \in (0,1)$, and (\ref{eq:get the negetive of |g*| eq-6}) holds by the definition of $\bU_t$ in (\ref{eq:get the negetive of |g*| eq-1.5}).
\end{proof}
\end{lemma}

\begin{lemma}
\label{lemma:error between log likelihood}
For all arm $a \in [K]$, we have
\begin{align}
    \left\Vert \widehat{\btheta}_t-\btheta_{t,a} \right\Vert_{\Bar{\bZ}_t} \leq \frac{\alpha(t) }{mL_b}\left\Vert \bg(\bx_{t,a};\btheta_{0}) \right\Vert_{\Bar{\bZ}_t^{-1}}.
\end{align}
Moreover, for any pair of arms $i,j \in [K]$, we have
\begin{align}
    \ell_{\lambda}(\mathcal{F}_t;\btheta_{t,i}) - \ell_{\lambda}(\mathcal{F}_t;\btheta_{t,j}) &\leq \frac{U_b\alpha(t)^2}{mL_b^2}  \left\Vert \bg(\bx_{t,j};\btheta_{0}) \right\Vert^2_{\Bar{\bZ}_t^{-1}}.
\end{align}
\begin{proof}
Recall that for all $a \in [K]$, we have
\begin{align}
    \widehat{\btheta}_t &= \argmax_{\btheta} \{\ell_{\lambda}(\mathcal{F}_t;\btheta)\} \label{eq:error between log likelihood eq-3} \\
    \btheta_{t,a} &= \argmax_{\btheta} \{\ell_{\lambda}(\mathcal{F}_t;\btheta) + \alpha(t)f(\bx_{t,a};\btheta)\} \label{eq:error between log likelihood eq-4}.
\end{align}
By considering the first-order necessary condition of $\widehat{\btheta}_t$ and $\btheta_{t,a}$, we obtain
\begin{align}
     \alpha(t)\bg(\bx_{t,a};\btheta_{0}) &= \ell'_{\lambda}(\mathcal{F}_t;\widehat{\btheta}_t) - \ell'_{\lambda}(\mathcal{F}_t;\btheta_{t,a}) \label{eq:error between log likelihood eq-5}\\
    &= -m\underbrace{\sum_{s=1}^{t-1}\frac{1}{m}b''(\langle \bg(\bx_s;\btheta_0), \btheta'' - \btheta_0 \rangle)\bg(\bx_s;\btheta_0)\bg(\bx_s;\btheta_0)^\top}_{:= \widehat{\bU}_t}(\widehat{\btheta}_t-\btheta_{t,a}) \label{eq:error between log likelihood eq-6},
\end{align}
where (\ref{eq:error between log likelihood eq-6}) holds by mean-value theorem with $\btheta'' = \kappa'\widehat{\btheta}_t + (1-\kappa')\btheta_{t,a}$ for some $\kappa' \in (0,1)$. Multiplying both sides of (\ref{eq:error between log likelihood eq-5})-(\ref{eq:error between log likelihood eq-6}) by $(\widehat{\btheta}_t-\btheta_{t,a})^\top$, we have 
\begin{align}
    \alpha(t)(\widehat{\btheta}_t-\btheta_{t,a})^\top\bg(\bx_{t,a};\btheta_{0}) = -m(\widehat{\btheta}_t-\btheta_{t,a})^\top\widehat{\bU}_t(\widehat{\btheta}_t-\btheta_{t,a}). \label{eq:error between log likelihood eq-7}
\end{align}
By applying Cauchy–Schwarz inequality to the left hand side of (\ref{eq:error between log likelihood eq-7}), we have
\begin{align}
    \frac{\alpha(t)}{m}(\widehat{\btheta}_t-\btheta_{t,a})^\top\bg(\bx_{t,a};\btheta_{0}) \leq \frac{\alpha(t)}{m}\left\Vert \widehat{\btheta}_t-\btheta_{t,a} \right\Vert_{\Bar{\bZ}_t} \cdot \left\Vert \bg(\bx_{t,a};\btheta_{0}) \right\Vert_{\Bar{\bZ}_t^{-1}} \label{eq:error between log likelihood eq-8}.
\end{align}
For the right hand side of (\ref{eq:error between log likelihood eq-7}), by the fact that $L_{b}\Bar{\bZ}_t \preceq \widehat{\bU}_t \preceq U_{b}\Bar{\bZ}_t$ for all $\btheta$ and Cauchy–Schwarz inequality, we have
\begin{align}
    \frac{\alpha(t)}{m}\left\Vert(\widehat{\btheta}_t-\btheta_{t,a})^\top\bg(\bx_{t,a};\btheta_{0})\right\Vert \geq L_b \left\Vert \widehat{\btheta}_t-\btheta_{t,a} \right\Vert^2_{\Bar{\bZ}_t}\label{eq:error between log likelihood eq-9}.
\end{align}
Then, by combining (\ref{eq:error between log likelihood eq-8}) and (\ref{eq:error between log likelihood eq-9}), we have
\begin{align}
    \left\Vert \widehat{\btheta}_t-\btheta_{t,a} \right\Vert_{\Bar{\bZ}_t} \leq \frac{\alpha(t) }{mL_b}\left\Vert \bg(\bx_{t,a};\btheta_{0}) \right\Vert_{\Bar{\bZ}_t^{-1}} \label{eq:error between log likelihood eq-10}.
\end{align}
Moreover, we have
\begin{align}
    \ell_{\lambda}(\mathcal{F}_t;\btheta_{t,i}) - \ell_{\lambda}(\mathcal{F}_t;\btheta_{t,j}) &\leq \ell_{\lambda}(\mathcal{F}_t;\widehat{\btheta}_t) - \ell_{\lambda}(\mathcal{F}_t;\btheta_{t,j}) \label{eq:error between log likelihood eq-11}\\
    & = m( \widehat{\btheta}_t - \btheta_{t,j})^\top\widehat{\bU}_t(\widehat{\btheta}_t - \btheta_{t,j}) \label{eq:error between log likelihood eq-12} \\
    &\leq mU_b\left\Vert \widehat{\btheta}_t - \btheta_{t,j} \right\Vert_{\Bar{\bZ}_t}^2 \label{eq:error between log likelihood eq-13}\\
    &\leq \frac{U_b\alpha(t)^2}{mL_b^2}  \left\Vert \bg(\bx_{t,j};\btheta_{0}) \right\Vert^2_{\Bar{\bZ}_t^{-1}}, \label{eq:error between log likelihood eq-14}
\end{align}
where (\ref{eq:error between log likelihood eq-11}) holds due to $\widehat{\btheta}_t = \argmax_{\btheta} \{ \ell_{\lambda}(\mathcal{F}_t;\btheta\})$, (\ref{eq:error between log likelihood eq-12}) holds by mean value theorem with $\btheta'' = \kappa\widehat{\btheta}_t + (1-\kappa)\btheta_{t,j}$ for some $\kappa \in (0,1)$, (\ref{eq:error between log likelihood eq-13}) holds by the fact that $\widehat{\bU}_t \preceq U_{b}\Bar{\bZ}_t$, and (\ref{eq:error between log likelihood eq-14}) holds due to (\ref{eq:error between log likelihood eq-10}). Hence, the proof is complete.
\end{proof}
\end{lemma}

\subsection{Proof of Theorem \ref{theorem:regret of NeuralRBMLE-GA}}
\label{appendix:Regret Bound of NeuralRBMLE-GA subsection 1}
Define 
\begin{align}
    G_1 &= \left\{ \bigg|\sum_{s = 1}^{t} r_s\bigg| \leq E_{\ref{lemma:upper bound of r},1}t\sqrt{\log{t}}\right\}  \\
    G_2 &= \left\{\sum_{s=1}^{t} r_s^2 \leq E_{\ref{lemma:upper bound of r},2}t\log{t} \right\}
\end{align}
as the \emph{good} events. By Lemma \ref{lemma:upper bound of r}, we have $\mathbb{P}(G_1,G_2) = 1-\frac{1}{t^2}$. To begin with, we start by obtaining an upper bound for the immediate regret of NeuralRBMLE-GA described in Algorithm \ref{alg:NeuralRBMLE-GA}. Under events $G_1$ and $G_2$, the expected immediate regret can be derived as
\begin{align}
   \mathbb{E}[r_t^{*} - r_t] =& \;h(\bx^{*}_{t}) - h(\bx_{t}) \label{eq:regret bound of NeuralRBMLE_0 eq-1}\\
   =& \;\langle \bg(\bx^{*}_{t};\btheta_{0}),\btheta^{*}-\btheta_{0} \rangle - \langle \bg(\bx_{t};\btheta_{0}),\btheta^{*}-\btheta_{0} \rangle \label{eq:regret bound of NeuralRBMLE_0 eq-2}\\
   = & \;\langle \bg(\bx^{*}_{t};\btheta_{0}),\btheta^{*}-\btheta_{t} \rangle + \langle \bg(\bx^{*}_{t};\btheta_{0}),\btheta_{t} - \btheta_0 \rangle - \langle \bg(\bx_{t};\btheta_{0}),\btheta^{*}-\btheta_{0} \rangle \label{eq:regret bound of NeuralRBMLE_0 eq-2.5} \\
   = & \;\langle \bg(\bx^{*}_{t};\btheta_{0}),\btheta^{*}-\btheta_{t} \rangle - \langle \bg(\bx_{t};\btheta_{0}),\btheta^{*}-\btheta_{0} \rangle +  \langle \bg(\bx_{t};\btheta_{0}),\btheta_{t,a^{*}_t} - \btheta_{t} \rangle  \nonumber\\
   &\; + \langle \bg(\bx^{*}_{t};\btheta_{0}),\btheta_{t,a^{*}_t} - \btheta_0 \rangle + \frac{\alpha(t)}{m} \Big( \left\Vert \bg(\bx_{t};\btheta_{0}) \right\Vert^2_{\bU_{t}^{-1}} - \left\Vert \bg(\bx^{*}_{t};\btheta_{0}) \right\Vert^2_{\bU_{t}^{-1}} \Big),
   \label{eq:regret bound of NeuralRBMLE_0 eq-3}
\end{align}
where (\ref{eq:regret bound of NeuralRBMLE_0 eq-2}) holds by Lemma \ref{lemma:theta*}, (\ref{eq:regret bound of NeuralRBMLE_0 eq-2.5}) is obtained by adding and subtracting $\btheta_t$, and (\ref{eq:regret bound of NeuralRBMLE_0 eq-3}) holds due to Lemma \ref{lemma:get the negetive of |g*|}. Then, applying Lemma \ref{lemma:index policy for NeuralRBMLE-GA} to the index policy of Algorithm \ref{alg:NeuralRBMLE-GA}, we have 

\begin{align}
   \mathbb{E}[r_t^{*} - r_t] \leq &\; \langle \bg(\bx^{*}_{t};\btheta_{0}),\btheta^{*}-\btheta_{t} \rangle - \langle \bg(\bx_{t};\btheta_{0}),\btheta^{*}-\btheta_{0} \rangle \nonumber\\
   &\; + \langle \bg(\bx_{t};\btheta_{0}),\btheta_{t,a^{*}_t} - \btheta_{t} \rangle + \langle \bg(\bx_{t};\btheta_{0}),\btheta_{t} - \btheta_0 \rangle + \frac{\ell_{\lambda}(\mathcal{F}_t;\btheta_{t}) - \ell_{\lambda}(\mathcal{F}_t;\btheta_{t,a^{*}_t})}{\alpha(t)\zeta(t)} \nonumber \\
   &\; + \frac{\alpha(t)}{m} \Big( \left\Vert \bg(\bx_{t};\btheta_{0}) \right\Vert^2_{\bU_{t}^{-1}} - \left\Vert \bg(\bx^{*}_{t};\btheta_{0}) \right\Vert^2_{\bU_{t}^{-1}} \Big) + 2E_{\ref{lemma:index policy for NeuralRBMLE-GA}}(m,t)
   \label{eq:regret bound of NeuralRBMLE_0 eq-4} \\
   \leq &\; \langle \bg(\bx^{*}_{t};\btheta_{0}),\btheta^{*}-\btheta_{t} \rangle + \langle \bg(\bx_{t};\btheta_{0}),\btheta_t^{*} - \btheta^{*} \rangle + 2E_{\ref{lemma:index policy for NeuralRBMLE-GA}}(m,t) \nonumber\\
   &\; + \frac{U_b\alpha(t)}{mL_b^2\zeta(t)} \left\Vert \bg(\bx^{*}_{t};\btheta_{0}) \right\Vert^2_{\Bar{\bZ}^{-1}_t} + \frac{\alpha(t)}{m} \Big( \left\Vert \bg(\bx_{t};\btheta_{0}) \right\Vert^2_{\bU_{t}^{-1}} - \left\Vert \bg(\bx^{*}_{t};\btheta_{0}) \right\Vert^2_{\bU_{t}^{-1}} \Big),\label{eq:regret bound of NeuralRBMLE_0 eq-5} \\
   \leq &\; \langle \bg(\bx^{*}_{t};\btheta_{0}),\btheta^{*}-\btheta_{t} \rangle + \langle \bg(\bx_{t};\btheta_{0}),\btheta_t^{*} - \btheta^{*} \rangle + 2E_{\ref{lemma:index policy for NeuralRBMLE-GA}}(m,t) \nonumber\\
   &\; + \frac{U_b\alpha(t)}{L_b^2\zeta(t)} \left\Vert \frac{\bg(\bx^{*}_{t};\btheta_{0})}{\sqrt{m}} \right\Vert^2_{\Bar{\bZ}^{-1}_t} + \alpha(t) \Big( \frac{1}{U_b}\left\Vert \frac{\bg(\bx_{t};\btheta_{0})}{\sqrt{m}} \right\Vert^2_{\Bar{\bZ}^{-1}_t} - \frac{1}{L_b}\left\Vert \frac{\bg(\bx^{*}_{t};\btheta_{0})}{\sqrt{m}} \right\Vert^2_{\Bar{\bZ}^{-1}_t} \Big),\label{eq:regret bound of NeuralRBMLE_0 eq-6}
\end{align}
where (\ref{eq:regret bound of NeuralRBMLE_0 eq-5}) holds by Lemma \ref{lemma:error between log likelihood}, and (\ref{eq:regret bound of NeuralRBMLE_0 eq-6}) holds by the fact that $L_b\Bar{\bZ}_t \preceq \bU_t \preceq U_b\Bar{\bZ}_t$ and  
\begin{align}
     \left\Vert \bg(\bx_{t};\btheta_{0}) \right\Vert^2_{\bU_{t}^{-1}} - \left\Vert \bg(\bx^{*}_{t};\btheta_{0}) \right\Vert^2_{\bU_{t}^{-1}}\leq \frac{1}{U_b}\left\Vert \bg(\bx_{t};\btheta_{0}) \right\Vert^2_{\Bar{\bZ}_{t}^{-1}} - \frac{1}{L_b}\left\Vert \bg(\bx^{*}_{t};\btheta_{0}) \right\Vert^2_{\Bar{\bZ}_{t}^{-1}}, \label{eq:regret bound of NeuralRBMLE_0 eq-7}
\end{align}
Regarding the term $\langle \bg(\bx^{*}_{t};\btheta_{0}),\btheta^{*}-\btheta_{t} \rangle$, we have
\begin{align}
    \langle \bg(\bx^{*}_{t};\btheta_{0}),\btheta^{*}-\btheta_{t} \rangle \leq & \sqrt{m}\left\Vert \btheta^{*} - \btheta_t \right\Vert_{\Bar{\bZ}_t} \cdot \left\Vert \frac{\bg(\bx^{*}_{t};\btheta_{0})}{\sqrt{m}}  \right\Vert_{\Bar{\bZ}^{-1}_t} \label{eq:regret bound of NeuralRBMLE_0 eq-8}\\
    \leq & \bigg( \sqrt{m}\left\Vert \btheta^{*} - \widehat{\btheta}_t \right\Vert_{\Bar{\bZ}_t} + \sqrt{m}\left\Vert \widehat{\btheta}_t - \btheta_t \right\Vert_{\Bar{\bZ}_t} \bigg) \cdot \left\Vert \frac{\bg(\bx^{*}_{t};\btheta_{0})}{\sqrt{m}} \right\Vert_{\Bar{\bZ}^{-1}_t} \label{eq:regret bound of NeuralRBMLE_0 eq-9}\\
    \leq & \bigg( \sqrt{m}\left\Vert \btheta^{*} - \widehat{\btheta}_t \right\Vert_{\Bar{\bZ}_t} + \frac{\alpha(t)}{\sqrt{m}L_b}\left\Vert \bg(\bx_{t};\btheta_{0}) \right\Vert_{\Bar{\bZ}_t^{-1}}  \bigg) \cdot \left\Vert \frac{\bg(\bx^{*}_{t};\btheta_{0})}{\sqrt{m}} \right\Vert_{\Bar{\bZ}^{-1}_t} \label{eq:regret bound of NeuralRBMLE_0 eq-10},
\end{align}
where (\ref{eq:regret bound of NeuralRBMLE_0 eq-8}) holds by the Cauchy–Schwarz inequality, (\ref{eq:regret bound of NeuralRBMLE_0 eq-9}) holds by the triangle inequality, and (\ref{eq:regret bound of NeuralRBMLE_0 eq-10}) holds by (\ref{eq:error between log likelihood eq-10}). Similarly, for the term $\langle \bg(\bx_{t};\btheta_{0}),\btheta_t^{*} - \btheta^{*} \rangle$, we have
\begin{align}
    \langle \bg(\bx_{t};\btheta_{0}),\btheta_t^{*} - \btheta^{*} \rangle \leq \bigg( \sqrt{m}\left\Vert \btheta^{*} - \widehat{\btheta}_t \right\Vert_{\Bar{\bZ}_t} + \frac{\alpha(t)}{\sqrt{m}L_b}\left\Vert \frac{\bg(\bx^{*}_{t};\btheta_{0})}{\sqrt{m}} \right\Vert_{\Bar{\bZ}_t^{-1}}  \bigg) \cdot \left\Vert \frac{\bg(\bx_{t};\btheta_{0})}{\sqrt{m}} \right\Vert_{\Bar{\bZ}^{-1}_t} \label{eq:regret bound of NeuralRBMLE_0 eq-11}.
\end{align}
Plugging (\ref{eq:regret bound of NeuralRBMLE_0 eq-6}), (\ref{eq:regret bound of NeuralRBMLE_0 eq-10}) and (\ref{eq:regret bound of NeuralRBMLE_0 eq-11}) into (\ref{eq:regret bound of NeuralRBMLE_0 eq-6}), we obtain 
\begin{align}
    \mathbb{E}[r_t^{*} - r_t] \leq & \bigg( \sqrt{m}\left\Vert \btheta^{*} - \widehat{\btheta}_t \right\Vert_{\Bar{\bZ}_t} + \frac{\alpha(t)}{\sqrt{m}L_b}\left\Vert \frac{\bg(\bx_{t};\btheta_{0})}{\sqrt{m}} \right\Vert_{\Bar{\bZ}_t^{-1}}  \bigg) \cdot \left\Vert \frac{\bg(\bx^{*}_{t};\btheta_{0})}{\sqrt{m}} \right\Vert_{\Bar{\bZ}^{-1}_t} + 2E_{\ref{lemma:index policy for NeuralRBMLE-GA}}(m,t)\nonumber \\ 
    & + \bigg( \sqrt{m}\left\Vert \btheta^{*} - \widehat{\btheta}_t \right\Vert_{\Bar{\bZ}_t} + \frac{\alpha(t)}{\sqrt{m}L_b}\left\Vert \frac{\bg(\bx^{*}_{t};\btheta_{0})}{\sqrt{m}} \right\Vert_{\Bar{\bZ}_t^{-1}}  \bigg) \cdot \left\Vert \bg(\bx_{t};\btheta_{0}) \right\Vert_{\Bar{\bZ}^{-1}_t} \nonumber \\
    & + \frac{U_b\alpha(t)}{L_b^2\zeta(t)} \left\Vert \frac{\bg(\bx^{*}_{t};\btheta_{0})}{\sqrt{m}} \right\Vert^2_{\Bar{\bZ}_t^{-1}} + \alpha(t) \bigg( \frac{1}{U_b}\left\Vert \frac{\bg(\bx_{t};\btheta_{0})}{\sqrt{m}} \right\Vert^2_{\Bar{\bZ}_t^{-1}} - \frac{1}{L_b}\left\Vert \frac{\bg(\bx^{*}_{t};\btheta_{0})}{\sqrt{m}} \right\Vert^2_{\Bar{\bZ}_t^{-1}} \bigg) \label{eq:regret bound of NeuralRBMLE_0 eq-12}
    .
\end{align}
Then, we collect all the terms that contain $\left\Vert \frac{\bg(\bx^{*}_{t};\btheta_{0})}{\sqrt{m}} \right\Vert_{\Bar{\bZ}_t^{-1}}$ into  $\Psi\big(\left\Vert \frac{\bg(\bx^{*}_{t};\btheta^{*}_{0})}{\sqrt{m}} \right\Vert_{\Bar{\bZ}^{-1}_{t}}\big) $, where $\Psi(\cdot):\mathbb{R}\rightarrow\mathbb{R}$ is defined as \begin{align}
    \Psi(x) = \Big( \sqrt{m}\left\Vert \btheta^{*} - \widehat{\btheta}_t \right\Vert_{\Bar{\bZ}_t} + \frac{2\alpha(t)}{\sqrt{m}U_b}\left\Vert \frac{\bg(\bx_{t};\btheta_{0})}{\sqrt{m}} \right\Vert_{\Bar{\bZ}_t^{-1}}  \Big)x + \Big( \frac{U_b\alpha(t)}{L_b^2\zeta(t)} - \frac{\alpha(t)}{L_b}\Big)x^2 
    .
\end{align}
By completing the square, for any $t \in \{t' \in \mathbb{N}: \frac{U_b}{L_b\zeta(t')} \leq 1\}$, we have that for any $x\in\mathbb{R}$, 
\begin{align}
    \Psi(x) \leq & \;\frac{\Big( \sqrt{m}\left\Vert \btheta^{*} - \widehat{\btheta}_t \right\Vert_{\Bar{\bZ}_t} + \frac{2\alpha(t)}{\sqrt{m}U_b}\left\Vert \frac{\bg(\bx_{t};\btheta_{0})}{\sqrt{m}} \right\Vert_{\Bar{\bZ}_t^{-1}}  \Big)^2}{4\Big( \frac{\alpha(t)}{L_b} - \frac{U_b\alpha(t)}{L_b^2\zeta(t)}\Big)} \\
    \leq & \;\frac{\Big( \sqrt{m}\left\Vert \btheta^{*} - \widehat{\btheta}_t \right\Vert_{\Bar{\bZ}_t} + \frac{2\alpha(t)}{\sqrt{m}U_b}\left\Vert \frac{\bg(\bx_{t};\btheta_{0})}{\sqrt{m}} \right\Vert_{\Bar{\bZ}_t^{-1}}  \Big)^2}{\frac{4\alpha(t)}{L_b}} \\
    = & \;\frac{m L_b}{4\alpha(t)}\left\Vert \btheta^{*} - \widehat{\btheta}_t \right\Vert_{\Bar{\bZ}_t}^2 + \frac{L_b}{U_b}\left\Vert \btheta^{*} - \widehat{\btheta}_t \right\Vert_{\Bar{\bZ}_t} \cdot \left\Vert \frac{ \bg(\bx_{t};\btheta_{0})}{\sqrt{m}} \right\Vert_{\Bar{\bZ}_t^{-1}} + \frac{\alpha(t)L_b}{m U_b^2}\left\Vert \frac{\bg(\bx_{t};\btheta_{0})}{\sqrt{m}} \right\Vert_{\Bar{\bZ}_t^{-1}}^2. \label{eq:regret bound of NeuralRBMLE_0 eq-13}
\end{align}
Therefore, combining (\ref{eq:regret bound of NeuralRBMLE_0 eq-12}) and (\ref{eq:regret bound of NeuralRBMLE_0 eq-13}), we obtain
\begin{align}
    \mathbb{E}[r_t^{*} - r_t] \leq &  \sqrt{m}\left\Vert \btheta^{*} - \widehat{\btheta}_t \right\Vert_{\Bar{\bZ}_t} \cdot \left\Vert \frac{\bg(\bx_{t};\btheta_{0})}{\sqrt{m}} \right\Vert_{\Bar{\bZ}^{-1}_t} + \frac{\alpha(t)}{\sqrt{m} L_b} \left\Vert \frac{\bg(\bx_{t};\btheta_{0})}{\sqrt{m}} \right\Vert^2_{\Bar{\bZ}_t^{-1}}  + \Psi\left(\left\Vert \frac{\bg(\bx^{*}_{t};\btheta^{*}_{0})}{\sqrt{m}} \right\Vert_{\Bar{\bZ}^{-1}_{t}}\right) \label{eq:regret bound of NeuralRBMLE_0 eq-14}\\
    \leq & 2\sqrt{m}\left\Vert \btheta^{*} - \widehat{\btheta}_t \right\Vert_{\Bar{\bZ}_t} \cdot \left\Vert \frac{\bg(\bx_{t};\btheta_{0})}{\sqrt{m}} \right\Vert_{\Bar{\bZ}^{-1}_t} + \frac{2\alpha(t)}{m U_b} \left\Vert \frac{\bg(\bx_{t};\btheta_{0})}{\sqrt{m}} \right\Vert^2_{\Bar{\bZ}_t^{-1}} + \frac{m L_b}{4\alpha(t)}\left\Vert \btheta^{*} - \widehat{\btheta}_t \right\Vert_{\Bar{\bZ}_t}^2 + 2E_{\ref{lemma:index policy for NeuralRBMLE-GA}}(m,t). \label{eq:regret bound of NeuralRBMLE_0 eq-15}
\end{align}
Regarding the total regret $\mathcal{R}(T)$ defined in (\ref{def:regret}), we have
\begin{align}
    \mathcal{R}(T) = & \;\mathbb{E}\bigg[ \sum^T_{t=1} (r^*_t - r_{t})\bigg|G_1,G_2\bigg] + \mathbb{E}\left[ 1 | G_1^{^\complement} \vee G_2^{^\complement} \right]  + TE_{\ref{lemma:index policy for NeuralRBMLE-GA}}(m,t) \label{eq:regret bound of NeuralRBMLE_0 eq-16}\\
    \leq & \;2\sqrt{m}\sum^T_{t=1} \left\Vert \btheta^{*} - \widehat{\btheta}_t \right\Vert_{\Bar{\bZ}_t} \cdot \left\Vert \frac{\bg(\bx_{t};\btheta_{0})}{\sqrt{m}} \right\Vert_{\Bar{\bZ}^{-1}_t} + \sum^T_{t=1} \left\Vert \frac{2\alpha(t)}{m U_b}  \bg(\bx_{t};\btheta_{0}) \right\Vert^2_{\Bar{\bZ}_t^{-1}} \nonumber \\
    & + \; \sum^T_{t=1}\frac{4\alpha(t)}{m L_b}\left\Vert \btheta^{*} - \widehat{\btheta}_t \right\Vert_{\Bar{\bZ}_t}^2 + T\sum_{t = 1}^{T}\frac{1}{t^2} + 2TE_{\ref{lemma:index policy for NeuralRBMLE-GA}}(m,t). \label{eq:regret bound of NeuralRBMLE_0 eq-17}
\end{align}
Then, by \citep[Theorem 2]{abbasi2011improved}, for any $\delta \in (0,1)$, with probability at least $1-\delta$, we have
\begin{align}
    \sqrt{m}\left\Vert \btheta^{*} - \widehat{\btheta}_t \right\Vert_{\Bar{\bZ}_t} \leq & \; \nu\sqrt{\log{\frac{\det{\Bar{\bZ}_t}}{\det{\lambda\bI}}}-2\log{\delta}} + \lambda^{\frac{1}{2}}S \label{eq:regret bound of NeuralRBMLE_0 eq-18}\\
    \leq & \;\nu\sqrt{\widetilde{d}\log{\left(\frac{1+TK}{\lambda}\right)} + 1-2\log{\delta}} + \lambda^{\frac{1}{2}}S := \Bar{\gamma}_t, \label{eq:regret bound of NeuralRBMLE_0 eq-19}
\end{align}
where (\ref{eq:regret bound of NeuralRBMLE_0 eq-19}) holds by \citep[B.19]{zhou2020neural}. For $\sum^T_{t=1}\frac{2\alpha(t)}{m U_b}\left\Vert \bg(\bx_{t};\btheta_{0}) \right\Vert^2_{\Bar{\bZ}^{-1}_t}$, we have
\begin{align}
    \sum^T_{t=1}\left\Vert \frac{2\alpha(t)}{m U_b} \bg(\bx_{t};\btheta_{0}) \right\Vert^2_{\Bar{\bZ}^{-1}_t} \leq &\; \sum^T_{t=1}\min\left\{\frac{2\alpha(t)}{m U_b}\left\Vert \bg(\bx_{t};\btheta_{0}) \right\Vert^2_{\Bar{\bZ}^{-1}_t},1\right\} \label{eq:regret bound of NeuralRBMLE_0 eq-20}\\
    \leq &\; \sum^T_{t=1}\frac{2\alpha(t)}{U_b}\min\left\{\left\Vert \frac{\bg(\bx_{t};\btheta_{0})}{\sqrt{m}}  \right\Vert^2_{\Bar{\bZ}^{-1}_t},1\right\} \label{eq:regret bound of NeuralRBMLE_0 eq-21}\\
    \leq & \;\frac{4\alpha(T)}{U_b}\log{\frac{\det{\Bar{\bZ}_t}}{\det{\lambda\bI}}} \label{eq:regret bound of NeuralRBMLE_0 eq-22} \\
    \leq & \;\frac{4\alpha(T)}{U_b} \cdot \left( \widetilde{d}\log{\left(\frac{1+TK}{\lambda}\right)} + 1 \right) \label{eq:regret bound of NeuralRBMLE_0 eq-23}
\end{align}
where (\ref{eq:regret bound of NeuralRBMLE_0 eq-20}) holds by $\mathbb{E}[r_t^{*} - r_t] \leq 1$, and (\ref{eq:regret bound of NeuralRBMLE_0 eq-21}) holds for any $\{t'\in \mathbb{N}:\frac{4\alpha(t')}{U_b} \geq 1\}$, (\ref{eq:regret bound of NeuralRBMLE_0 eq-22}) holds by \citep[Lemma 11]{abbasi2011improved}, and (\ref{eq:regret bound of NeuralRBMLE_0 eq-23}) holds by \citep[B.19]{zhou2020neural}. Then, we have
\begin{align}
    \mathcal{R}(T) \leq & \;2\Bar{\gamma}_T\sqrt{T} \sqrt{\widetilde{d}\log\left(1+\frac{TK}{\lambda}\right) + 1} + \frac{2\alpha(T)}{U_b} \left(\widetilde{d}\log\left(1+\frac{TK}{\lambda}\right) + 1\right) + \frac{4\alpha(T)}{m L_b} \Bar{\gamma}_T^2 + \log{T} + 2TE_{\ref{lemma:index policy for NeuralRBMLE-GA}}(m,t) \label{eq:regret bound of NeuralRBMLE_0 eq-24}\\
    = & \;\mathcal{O}\left(\widetilde{d}\sqrt{T}\log{T}\right), \label{eq:regret bound of NeuralRBMLE_0 eq-25}
\end{align}
where (\ref{eq:regret bound of NeuralRBMLE_0 eq-25}) holds by the choice of $m$.
\section{Regret Bound of NeuralRBMLE-PC}
\label{appendix:Regret Bound of NeuralRBMLE-PC}
\subsection{Supporting Lemmas for the Proof of Theorem \ref{theorem:regret of NeuralRBMLE-PC}}

\begin{lemma}
\label{lemma:theta of ucb}
\citep[Lemma 5.2]{zhou2020neural}
There exist positive constant $\left\{E_{\ref{lemma:theta of ucb},i}\right\}_{i=1}^{5}$, such that for all $\delta \in (0,1)$, if $\eta \leq E_{\ref{lemma:theta of ucb},1}(TmL + m\lambda)^{-1}$ and
\begin{align}
    m & \geq E_{\ref{lemma:theta of ucb},2}\max\left\{T^7\lambda^{-7}L^{21}(\log m)^3,\lambda^{-\frac{1}{2}}L^{-\frac{3}{2}}\bigg(\log \frac{TKL^2}{\delta} \bigg)^{\frac{3}{2}}\right\} ,
\end{align}
with probability at least $1-\delta$, for all $t \in [T]$, we have $\left\Vert \widehat{\btheta}_{t} - \btheta_0 \right\Vert_2 \leq 2\sqrt{\frac{t}{m\lambda}}$ and 
$\left\Vert \btheta^{*} - \widehat{\btheta}_{t}\right\Vert_{\bZ_t} \leq \frac{\gamma(t)}{\sqrt{m}}$, where 
\begin{align}
    \gamma(t) = & \sqrt{1+ E_{\ref{lemma:theta of ucb},3}m^{-\frac{1}{6}}\sqrt{\log m}L^4t^{\frac{7}{6}}\lambda^{-\frac{7}{6}}} \nonumber \\
    & \cdot \bigg( \nu\sqrt{\log{\frac{\det{\bZ_t}}{\det{\lambda\bI}}}+ E_{\ref{lemma:theta of ucb},4}m^{-\frac{1}{6}}\sqrt{\log{m}}L^{4}t^{\frac{5}{3}}\lambda^{-\frac{1}{6}} - 2\log{\delta}} +  \sqrt{\lambda}S\bigg) \nonumber \\
    & + (\lambda + E_{\ref{lemma:theta of ucb},5}tL)\bigg[ (1-\eta m \lambda)^{\frac{J}{2}}\sqrt{\frac{t}{m\lambda}} + m^{-\frac{2}{3}}\sqrt{\log{m}}L^{\frac{7}{2}}t^{\frac{5}{3}}\lambda^{-\frac{5}{3}}(1+\sqrt{\frac{t}{\lambda}})  \bigg].
\end{align}
Notice that according to the proof of Lemma 5.4 in \citep{zhou2020neural}, we have
\begin{align}
    \log{\frac{\det{\bZ_t}}{\det{\lambda\bI}}} \leq \widetilde{d}\log(1+TK/\lambda) + E_{\ref{lemma:theta of ucb},4}m^{-\frac{1}{6}}\sqrt{\log{m}}L^{4}T^{\frac{5}{3}}\lambda^{-\frac{1}{6}}.
\end{align}
\end{lemma}

\begin{lemma}\label{lemma:Lemma 11 of Abbasi.2011}
\citep[Lemma 5.4]{zhou2020neural}
There exist positive constants $\{E_{\ref{lemma:Lemma 11 of Abbasi.2011},i}\}_{i=1}^{3}$ such that for all $t \in [T]$ and $\delta \in (0,1)$, if $\eta \leq E_{\ref{lemma:Lemma 11 of Abbasi.2011},1}(TmL + m\lambda)^{-1}$ and m satisfies that 
\begin{align}
    m \geq E_{\ref{lemma:Lemma 11 of Abbasi.2011},2}\max\left\{T^7\lambda^{-7}L^{21}(\log{m})^{3},T^6K^6L^6\left(\log\left(\frac{TKL^2}{\delta}\right)\right)^{\frac{3}{2}}\right\},
\end{align}
with probability at least $1-\delta$, we have
\begin{align}
    \sum_{t=1}^{T} \min\left\{\left\Vert \frac{\bg(\bx_t;\widehat{\btheta}_t)}{\sqrt{m}} \right\Vert^2_{\bZ_{t-1}^{-1}} , 1\right\} \leq 2 \widetilde{d}\log(1+TK/\lambda) + 2 +E_{\ref{lemma:Lemma 11 of Abbasi.2011},3}m^{-\frac{1}{6}}\sqrt{\log{m}}L^{4}T^{\frac{5}{3}}\lambda^{-\frac{1}{6}},
\end{align}
where $\widetilde{d}$ is the effective dimension defined in (\ref{def:effective dimension}).
\end{lemma}

\begin{lemma}
\label{lemma:range of theta_bar}
For all $a \in [K]$, $t \in [T]$, we have
\begin{align}
    \left\Vert \Bar{\btheta}_{t,a} - \btheta_0 \right\Vert_2 \leq 2\sqrt{\frac{t}{m\lambda}} + E_{\ref{lemma:range of theta_bar},1}\frac{\alpha(t)}{\sqrt{m\lambda}} := \Bar{\tau} \label{eq:range of theta_bar eq-1},
\end{align}
where $E_{\ref{lemma:range of theta_bar},1}$ is a positive constant. Moreover, by choosing $\alpha(t) = \sqrt{t}$ and defining 
$E_{\ref{lemma:range of theta_bar}} = 2\max\{2, E_{\ref{lemma:range of theta_bar},1} \}$, we have $\left\Vert \Bar{\btheta}_{t,a} - \btheta_0 \right\Vert_2 \leq E_{\ref{lemma:range of theta_bar}}\sqrt{\frac{t}{m\lambda}}$.
\begin{proof}
Recall that $\Bar{\btheta}_{t,a} = \widehat{\btheta}_{t} + \frac{\alpha(t)}{m}\cdot \bZ^{-1}_{t-1} \bg(\bx_{t,a};\widehat{\btheta}_{t})$. By triangle inequality, we have
\begin{align}
    \left\Vert \Bar{\btheta}_{t,a} - \btheta_0 \right\Vert_2 \leq & \left\Vert \widehat{\btheta}_{t} - \btheta_0 \right\Vert_2 + \frac{\alpha(t)}{m}\cdot\left\Vert \bZ^{-1}_{t-1} \bg(\bx_{t,a};\widehat{\btheta}_{t}) \right\Vert_2 \label{eq:range of theta_bar eq-2}.
\end{align}
Then, by Lemma \ref{lemma:theta of ucb}, we can verify that $\lVert \widehat{\btheta}_t -\btheta_0 \rVert_2$ satisfies the condition of Lemma \ref{lemma:gradient of f}. Therefore, we have
\begin{align}
    \left\Vert \Bar{\btheta}_{t,a} - \btheta_0 \right\Vert_2 \leq &\; 2\sqrt{\frac{t}{m\lambda}} + \frac{\alpha(t)}{m} \left\Vert  \bZ^{-1}_{t-1} \bg(\bx_{t,a};\widehat{\btheta}_{t}) \right\Vert_2 \label{eq:range of theta_bar eq-3} \\
    \leq &\; 2\sqrt{\frac{t}{m\lambda}} + E_{\ref{lemma:gradient of f}}\frac{\alpha(t)}{\sqrt{m}}\sqrt{L}\left\Vert  \bZ^{-1}_{t-1}\right\Vert_2 \label{eq:range of theta_bar eq-4} \\
    \leq &\; 2\sqrt{\frac{t}{m\lambda}} + E_{\ref{lemma:gradient of f}}\frac{\alpha(t)}{\sqrt{m}}\frac{\sqrt{L}}{\lambda} \label{eq:range of theta_bar eq-5},
\end{align}
where (\ref{eq:range of theta_bar eq-3}) holds by Lemma \ref{lemma:theta of ucb}, (\ref{eq:range of theta_bar eq-4}) holds by Lemma \ref{lemma:gradient of f} and Cauchy–Schwarz inequality, and (\ref{eq:range of theta_bar eq-5}) holds by the fact that $\left\Vert  \bZ^{-1}_{t}\right\Vert_2 \leq \frac{1}{\lambda}$, for all $t \in [T]$. By letting $E_{\ref{lemma:range of theta_bar},1} := E_{\ref{lemma:gradient of f}}\sqrt{\frac{L}{\lambda}}$, we complete the proof.
\end{proof}
\end{lemma}
\subsection{Proof of Theorem \ref{theorem:regret of NeuralRBMLE-PC}}
\label{appendix: Regret Bound of NeuralRBMLE-PC subsection 2}
By Lemmas \ref{lemma:theta of ucb}-\ref{lemma:range of theta_bar}, the choice of $m$, and choosing $\alpha(t) = \sqrt{t}$, we can verify that Lemmas \ref{appendix_lemma_1}-\ref{lemma:gradient of f} hold when $\btheta = \Bar{\btheta}_{t,a}$ and $\btheta = \widehat{\btheta}_t$, for all $t \in [T], a \in [K]$. Then, we start by the index comparison of NeuralRBMLE-PC in Algorithm \ref{alg:NeuralRBMLE-PC} as follows:
\begin{align}
    &f(\bx_t,\Bar{\btheta}_{t,a_t}) \geq f(\bx^*_t,\Bar{\btheta}_{t,a^*_t}) \\
    \implies& 2 E_{\ref{appendix_lemma_1}}\left( E_{\ref{lemma:range of theta_bar}}\sqrt{\frac{t}{m\lambda}}\right)^{\frac{4}{3}}L^3\sqrt{m\log{m}} + \langle\,\bg(\bx_{t};\btheta_0), \Bar{\btheta}_{t,a_t} - \btheta_0\rangle \geq \langle\,\bg(\bx^*_{t};\btheta_0), \Bar{\btheta}_{t,a^*_t} - \btheta_0\rangle \label{eq: NeuralRBMLE-PC index 1}\\
    \implies& 2 E_{\ref{appendix_lemma_1}}\left( E_{\ref{lemma:range of theta_bar}}\sqrt{\frac{t}{m\lambda}}\right)^{\frac{4}{3}}L^3\sqrt{m\log{m}} + \left\Vert \Bar{\btheta}_{t,a_t} - \btheta_0 \right\Vert_2 \cdot \left\Vert \bg(\bx_{t};,\btheta_0) - \bg(\bx_{t};\widehat{\btheta}_{t}) \right\Vert_2   \nonumber\\
    &+\left\Vert \Bar{\btheta}_{t,a^*_t} - \btheta_0 \right\Vert_2 \cdot  \left\Vert \bg(\bx^*_t;\btheta_0) - \bg(\bx^*_t;\widehat{\btheta}_{t}) \right\Vert_2 + \langle\,\bg(\bx_{t};\widehat{\btheta}_{t}) ,\Bar{\btheta}_{t,a_t} - \btheta_0\rangle \geq \langle\,\bg(\bx^*_{t};\widehat{\btheta}_{t}), \Bar{\btheta}_{t,a^*_t} - \btheta_0\rangle \label{eq: NeuralRBMLE-PC index 2}\\
    \implies& \underbrace{2 E_{\ref{appendix_lemma_1}}\left( E_{\ref{lemma:range of theta_bar}}\sqrt{\frac{t}{m\lambda}}\right)^{\frac{4}{3}}L^3\sqrt{m\log{m}} + 2 E_{\ref{appendix_lemma_2}}E_{\ref{lemma:gradient of f}}\sqrt{m\log m}\left( 2\sqrt{\frac{t}{m\lambda}}\right)^\frac{4}{3}L^\frac{7}{2}}_{=:D_1} \nonumber \\
    & + \langle\,\bg(\bx_{t};\widehat{\btheta}_{t}) ,\Bar{\btheta}_{t,a_t} - \btheta_0\rangle \geq \langle\,\bg(\bx^*_{t};\widehat{\btheta}_{t}), \Bar{\btheta}_{t,a^*_t} - \btheta_0\rangle, \label{eq: NeuralRBMLE-PC index 3}
\end{align}
where (\ref{eq: NeuralRBMLE-PC index 1}) follows from Lemma \ref{appendix_lemma_1}, (\ref{eq: NeuralRBMLE-PC index 2}) is obtained by adding and subtracting $\bg(\bx_{t};\widehat{\btheta}_{t})$ as well as $\bg(\bx_t;\Bar{\btheta}_{t,a^*_t})$ at the same time and then applying the Cauchy-Schwarz inequality, and (\ref{eq: NeuralRBMLE-PC index 3}) follows from Lemma \ref{appendix_lemma_2}, Lemma \ref{lemma:gradient of f} and Lemma \ref{lemma:range of theta_bar}. For ease of notation, we use $E_{\text{PC},1}$ to denote the pre-constant of $D_1$. By (\ref{eq: NeuralRBMLE-PC index 3}), we further have
\begin{align}
    &\langle\,\bg(\bx_{t};\widehat{\btheta}_{t}) ,\Bar{\btheta}_{t,a_t} - \btheta_0\rangle + \underbrace{E_{\text{PC},1}t^{\frac{2}{3}}m^{-\frac{1}{6}}\lambda^{-\frac{2}{3}}L^\frac{7}{2}\sqrt{\log{m}}}_{= D_1} \geq \langle\,\bg(\bx^*_{t};\widehat{\btheta}_{t}), \Bar{\btheta}_{t,a^*_t} - \btheta_0\rangle \label{eq: NeuralRBMLE-PC index 5}\\
    \implies &\langle\,\bg(\bx_{t};\widehat{\btheta}_{t}), \widehat{\btheta}_{t} + \frac{\alpha(t)}{m}\bZ^{-1}_{t} \bg(\bx_{t};\widehat{\btheta}_{t}) - \btheta_0\rangle  - \langle\,\bg(\bx^*_{t};\widehat{\btheta}_{t}),\widehat{\btheta}_{t}+  \frac{\alpha(t)}{m}\bZ^{-1}_{t} \bg(\bx^*_{t};\widehat{\btheta}_{t}) - \btheta_0\rangle + D_1 \geq 0\label{eq: NeuralRBMLE-PC index 6}\\
    \implies &\langle\,\bg(\bx_{t};\widehat{\btheta}_{t}), \widehat{\btheta}_t - \btheta_0\rangle + \frac{\alpha(t)}{m}\left\Vert \bg(\bx_{t};\widehat{\btheta}_{t}) \right\Vert^2_{\bZ^{-1}_t} -\Big( \langle\,\bg(\bx^*_{t};\widehat{\btheta}_{t}), \widehat{\btheta}_t - \btheta_0\rangle + \frac{\alpha(t)}{m}\left\Vert \bg(\bx^*_{t};\widehat{\btheta}_{t}) \right\Vert^2_{\bZ^{-1}_t}\Big)+D_1\geq 0.\label{eq: NeuralRBMLE-PC index 7}
\end{align}
Now we are ready to derive the regret bound of NeuralRBMLE-PC.
To begin with, we quantify the immediate regret at each step $t$ as
\begin{align}
    h(\bx^*_t) - h(\bx_{t})
    = & \; \langle\,\bg(\bx^*_t;\btheta_0),\btheta^* - \btheta_0\rangle - \langle\,\bg(\bx_{t};\btheta_0),\btheta^* - \btheta_0\rangle \label{eq: NeuralRBMLE-PC regret 2}\\
    \leq & \; \langle\,\bg(\bx^*_t;\widehat{\btheta}_t),\btheta^* - \btheta_0\rangle - \langle\,\bg(\bx_{t};\widehat{\btheta}_t),\btheta^* - \btheta_0\rangle \nonumber\\
    &+ \left\Vert \btheta^* - \btheta_0 \right\Vert_2 \cdot \left(\left\Vert \bg(\bx^*_t;\btheta_0) - \bg(\bx^*_t;\widehat{\btheta}_t)\right\Vert_2 + \left\Vert \bg(\bx_t;\btheta_0) - \bg(\bx_t;\widehat{\btheta}_t)\right\Vert_2\right), \label{eq: NeuralRBMLE-PC regret 3}\\
    \leq & \; \langle\,\bg(\bx^*_t;\widehat{\btheta}_t),\btheta^* - \btheta_0\rangle - \langle\,\bg(\bx_{t};\widehat{\btheta}_t),\btheta^* - \btheta_0\rangle + \underbrace{2Sm^{-\frac{1}{6}}\sqrt{\log{m}}t^{\frac{1}{6}}\lambda^{-\frac{1}{6}}L^{\frac{7}{2}}}_{=:D_2}\label{eq: NeuralRBMLE-PC regret 4}
\end{align}
where (\ref{eq: NeuralRBMLE-PC regret 2}) follows directly from Lemma \ref{lemma:theta*}, (\ref{eq: NeuralRBMLE-PC regret 3}) holds by adding and subtracting $\bg(\bx^*_t;\widehat{\btheta}_t)$ as well as $\bg(\bx_t;\widehat{\btheta}_t)$ at the same time and then applying the Cauchy-Schwarz inequality,
(\ref{eq: NeuralRBMLE-PC regret 4}) follows from Lemma \ref{lemma:theta*}, Lemma \ref{appendix_lemma_2}, Lemma \ref{lemma:gradient of f} and $\sqrt{2\textbf{h}^\top\textbf{H}^{-1}\textbf{h}} \leq S$. For $ \langle\,\bg(\bx^*_t;\widehat{\btheta}_t),\btheta^* - \btheta_0\rangle$, due to Lemma \ref{lemma:theta of ucb}, Theorem 2 in \cite{abbasi2011improved}, and the fact that 
\begin{align}
    \max_{\bm{x}:\left\Vert \bm{x}-\bm{b}\right\Vert_A \leq \bm{c}} \langle \bm{a}\,, \bm{x} \rangle = \langle \bm{a}\,, \bm{b} \rangle + \bm{c}\sqrt{\bm{a}^\top \bm{A}^{-1} \bm{a}},
\end{align}
we have
\begin{align}
    & \; \langle\,\bg(\bx^*_t;\widehat{\btheta}_t),\btheta^* - \btheta_0\rangle \nonumber\\
    \leq & \; \langle\,\bg(\bx^*_{t};\widehat{\btheta}_t), \widehat{\btheta}_t - \btheta_0\rangle  + \frac{\gamma(t)}{\sqrt{m}}\left\Vert \bg(\bx^*_{t};\widehat{\btheta}_t)\right\Vert_{\bZ^{-1}_t} \label{eq: NeuralRBMLE-PC regret 5}\\
    =& \; \langle\,\bg(\bx^*_{t};\widehat{\btheta}_t), \widehat{\btheta}_t - \btheta_0\rangle  + \frac{\gamma(t)}{\sqrt{m}}\left\Vert \bg(\bx^*_{t};\widehat{\btheta}_t)\right\Vert_{\bZ^{-1}_t} + \underbrace{\frac{\alpha(t)}{m}\left\Vert \bg(\bx^*_{t};\widehat{\btheta}_t) \right\Vert^2_{\bZ^{-1}_t} - \frac{\alpha(t)}{m}\left\Vert \bg(\bx^*_{t};\widehat{\btheta}_t) \right\Vert^2_{\bZ^{-1}_t}}_{=0} \label{eq: NeuralRBMLE-PC regret 6}.
\end{align}
Applying (\ref{eq: NeuralRBMLE-PC index 7}) to eliminate the term $\langle\,\bg(\bx^*_{t};\widehat{\btheta}_t), \widehat{\btheta}_t - \btheta_0\rangle$ and $\frac{\alpha(t)}{m}\left\Vert \bg(\bx^*_{t};\widehat{\btheta}_t) \right\Vert^2_{\bZ^{-1}_t}$ in (\ref{eq: NeuralRBMLE-PC regret 6}), we have 
\begin{align}
    \langle\,\bg(\bx^*_t;\widehat{\btheta}_t),\btheta^* - \btheta_0\rangle \leq & \; \langle\,\bg({x_{t}};\widehat{\btheta}_t), \widehat{\btheta}_t - \btheta_0\rangle  + \frac{\alpha(t)}{m}\left\Vert \bg({x_{t}};\widehat{\btheta}_t) \right\Vert^2_{\bZ^{-1}_t} \nonumber\\
    & + \frac{\gamma(t)}{\sqrt{m}}\left\Vert \bg(\bx^*_{t};\widehat{\btheta}_t)\right\Vert_{\bZ^{-1}_t} - \frac{\alpha(t)}{m}\left\Vert \bg(\bx^*_{t};\widehat{\btheta}_t) \right\Vert^2_{\bZ^{-1}_t} +D_1 \label{eq: NeuralRBMLE-PC regret 7}\\
   = & \; \langle\,\bg(\bx_{t};\widehat{\btheta}_t), \widehat{\btheta}_t - \btheta_0\rangle  + \frac{\alpha(t)}{m}\left\Vert \bg(\bx_{t};\widehat{\btheta}_t) \right\Vert^2_{\bZ^{-1}_t} \\
   &-\frac{\alpha(t)}{m}\left( \left\Vert \bg(\bx^*_{t};\widehat{\btheta}_t)\right\Vert_{\bZ^{-1}_t} - \frac{\gamma(t)\sqrt{m}}{2\alpha(t)} \right)^2 + \frac{\gamma(t)^2}{4\alpha(t)}+D_1\label{eq: NeuralRBMLE-PC regret 8}\\
    \leq & \; \langle\,\bg(\bx_{t};\widehat{\btheta}_t), \widehat{\btheta}_t - \btheta_0\rangle  + \frac{\alpha(t)}{m}\left\Vert \bg(\bx_{t};\widehat{\btheta}_t) \right\Vert^2_{\bZ^{-1}_t} + \frac{\gamma(t)^2}{4\alpha(t)} +D_1\label{eq: NeuralRBMLE-PC regret 9},
\end{align}
where (\ref{eq: NeuralRBMLE-PC regret 8}) holds by completing the square with respect to $\left\Vert \bg(\bx^*_{t};\widehat{\btheta}_t)\right\Vert_{\bZ^{-1}_t}$. Then, combining (\ref{eq: NeuralRBMLE-PC regret 4}) and (\ref{eq: NeuralRBMLE-PC regret 9}), we obtain 
\begin{align}
    h(\bx^*_t) - h(\bx_{t}) \leq \langle\,\bg(\bx_{t};\widehat{\btheta}_t), \widehat{\btheta}_t - \btheta_0\rangle  + \frac{\alpha(t)}{m}\left\Vert \bg(\bx_{t};\widehat{\btheta}_t) \right\Vert^2_{\bZ^{-1}_t} + \frac{\gamma(t)^2}{4\alpha(t)} - \langle\,\bg(\bx_{t};\widehat{\btheta}_t),\btheta^* - \btheta_0\rangle + D_1 + D_2. \label{eq: NeuralRBMLE-PC regret 9.5}
\end{align}
Moreover, we have
\begin{align}
    &\langle\,\bg(\bx_{t};\widehat{\btheta}_t), \widehat{\btheta}_t - \btheta_0\rangle  - \langle\,\bg(\bx_{t};\widehat{\btheta}_t),\btheta^* - \btheta_0\rangle \label{eq: NeuralRBMLE-PC regret 10}\\
    \leq &\max_{\btheta:\lVert \btheta - \widehat{\btheta}_t \rVert_{\bZ_t} \leq \frac{\gamma(t)}{\sqrt{m}} }\left\{\langle\,\bg(\bx_{t};\widehat{\btheta}_t), \btheta - \btheta_0\rangle  - \langle\,\bg(\bx_{t};\widehat{\btheta}_t),\btheta^* - \btheta_0\rangle\right\} \label{eq: NeuralRBMLE-PC regret 11}\\
    =&\max_{\btheta:\lVert \btheta - \widehat{\btheta}_t \rVert_{\bZ_t} \leq \frac{\gamma(t)}{\sqrt{m}} }\left\{\langle\,\bg(\bx_{t};\widehat{\btheta}_t), \btheta - \widehat{\btheta}_t\rangle  - \langle\,\bg(\bx_{t};\widehat{\btheta}_t),\btheta^* - \widehat{\btheta}_t\rangle \right\} \label{eq: NeuralRBMLE-PC regret 12}\\
    \leq &\max_{\btheta:\lVert \btheta - \widehat{\btheta}_t \rVert_{\bZ_t} \leq \frac{\gamma(t)}{\sqrt{m}} }\left\{\left\Vert\bg(\bx_{t};\widehat{\btheta}_t)\right\Vert_{\bZ^{-1}_t} \cdot\left\Vert\btheta - \widehat{\btheta}_t \right\Vert_{\bZ_t}  + \left\Vert \bg(\bx_{t};\widehat{\btheta}_t)\right\Vert_{\bZ^{-1}_t}  \cdot\left\Vert\btheta^* - \widehat{\btheta}_t\right\Vert_{\bZ_t}\right\}  \label{eq: NeuralRBMLE-PC regret 13}\\
    \leq &\; \frac{2\gamma(t)}{\sqrt{m}} \left\Vert \bg(\bx_{t};\widehat{\btheta}_t) \right\Vert_{\bZ^{-1}_t},\label{eq: NeuralRBMLE-PC regret 14}
\end{align}
where (\ref{eq: NeuralRBMLE-PC regret 13}) holds by the Cauchy-Schwarz inequality, and (\ref{eq: NeuralRBMLE-PC regret 14}) follows directly from the fact that $\lVert \btheta - \widehat{\btheta}_t \rVert_{\bZ_t} \leq \frac{\gamma(t)}{\sqrt{m}}$ and Lemma \ref{lemma:theta of ucb}. Therefore, plugging (\ref{eq: NeuralRBMLE-PC regret 14}) into (\ref{eq: NeuralRBMLE-PC regret 9.5}), with probability at least $1-\delta$, we have
\begin{align}
        h(\bx^*_t) - h(\bx_{t})&\leq \frac{2\gamma(t)}{\sqrt{m}} \left\Vert \bg(\bx_{t};\widehat{\btheta}_t) \right\Vert_{\bZ^{-1}_t} + \frac{\alpha(t)}{m}\left\Vert \bg(\bx_{t};\widehat{\btheta}_t) \right\Vert^2_{\bZ^{-1}_t}+ \frac{\gamma(t)^2}{4\alpha(t)}+D_1+D_2.  \label{eq: NeuralRBMLE-PC regret 15}
\end{align}
Notice that we cannot directly minimize (\ref{eq: NeuralRBMLE-PC regret 15}) by the inequality of arithmetic and geometric means to obtain the reward-bias term because $\{\Bar{\btheta}_{t,a}\}_{t\in[T],a\in[K]}$ need to satisfy the constrains of Lemma \ref{appendix_lemma_1}-\ref{lemma:gradient of f}. According to Lemma \ref{lemma:range of theta_bar} and the choice of $\alpha(t) = \sqrt{t}$, we can further bound the summation of the immediate regret over time step $t$ as follows:
\begin{align}
    &\sum_{t=1}^{T} h(\bx^*_t) - h(\bx_{t}) \nonumber\\
    \leq & \; 2\gamma(T) \sum_{t=1}^{T}\left\Vert \frac{\bg(\bx_{t};\widehat{\btheta}_t)}{\sqrt{m}} \right\Vert_{\bZ^{-1}_t} + \sqrt{T}\sum_{t=1}^{T}\left\Vert \frac{\bg(\bx_{t};\widehat{\btheta}_t)}{\sqrt{m}} \right\Vert^2_{\bZ^{-1}_t}+ \sum_{t=1}^{T}\frac{\gamma(t)^2}{4\sqrt{t}}+T(D_1+D_2). \label{eq: NeuralRBMLE-PC regret 16} \\
     \leq & \; 2\gamma(T) \sum_{t=1}^{T}\min\left\{\left\Vert \frac{\bg(\bx_{t};\widehat{\btheta}_t)}{\sqrt{m}} \right\Vert_{\bZ^{-1}_t},1\right\} + \sqrt{T}\sum_{t=1}^{T}\min\left\{\left\Vert \frac{\bg(\bx_{t};\widehat{\btheta}_t)}{\sqrt{m}} \right\Vert^2_{\bZ^{-1}_t},1\right\} + \sum_{t=1}^{T}\frac{\gamma(t)^2}{4\sqrt{t}}+T(D_1+D_2), \label{eq: NeuralRBMLE-PC regret 17}
\end{align}
where (\ref{eq: NeuralRBMLE-PC regret 17}) holds by the fact that $h(\bx^{*}_{t}) - h(\bx_t)$ is upper bounded by $1$.
For $2\gamma(T) \sum_{t=1}^{T}\min\left\{\left\Vert \frac{\bg(\bx_{t};\widehat{\btheta}_t)}{\sqrt{m}} \right\Vert_{\bZ^{-1}_t},1\right\}$, we have
\begin{align}
    2\gamma(T) \sum_{t=1}^{T}\min\left\{\left\Vert \frac{\bg(\bx_{t};\widehat{\btheta}_t)}{\sqrt{m}} \right\Vert_{\bZ^{-1}_t},1\right\} \leq &\; 2\gamma(T) \sqrt{T\cdot \sum_{t=1}^{T}\min\left\{\left\Vert \frac{\bg(\bx_{t};\widehat{\btheta}_t)}{\sqrt{m}} \right\Vert^2_{\bZ^{-1}_t},1\right\}} \label{eq: NeuralRBMLE-PC regret 18}\\
    \leq &\; 2\sqrt{T}\gamma(T)\sqrt{2 \widetilde{d}\log(1+TK/\lambda) + 2 +E_{\ref{lemma:Lemma 11 of Abbasi.2011},3}m^{-\frac{1}{6}}\sqrt{\log{m}}L^{4}T^{\frac{5}{3}}\lambda^{-\frac{1}{6}}}.\label{eq: NeuralRBMLE-PC regret 19}
\end{align}
where (\ref{eq: NeuralRBMLE-PC regret 18}) holds by Cauchy–Schwarz inequality, and (\ref{eq: NeuralRBMLE-PC regret 19}) holds by Lemma \ref{lemma:Lemma 11 of Abbasi.2011}. For $\sqrt{T}\sum_{t=1}^{T}\min\left\{\left\Vert \frac{\bg(\bx_{t};\widehat{\btheta}_t)}{\sqrt{m}} \right\Vert^2_{\bZ^{-1}_t},1\right\}$, we have 
\begin{align}
    \sqrt{T}\sum_{t=1}^{T}\min\left\{\left\Vert \frac{\bg(\bx_{t};\widehat{\btheta}_t)}{\sqrt{m}} \right\Vert^2_{\bZ^{-1}_t},1\right\} \leq  \; \sqrt{T}\left(2 \widetilde{d}\log(1+TK/\lambda) + 2 +E_{\ref{lemma:Lemma 11 of Abbasi.2011},3}m^{-\frac{1}{6}}\sqrt{\log{m}}L^{4}T^{\frac{5}{3}}\lambda^{-\frac{1}{6}}\right), \label{eq: NeuralRBMLE-PC regret 21} 
\end{align}
where (\ref{eq: NeuralRBMLE-PC regret 21}) holds by Lemma \ref{lemma:Lemma 11 of Abbasi.2011}. For $\sum_{t=1}^{T}\frac{\gamma(t)^2}{4\sqrt{t}}$, we have
\begin{align}
    \sum_{t=1}^{T}\frac{\gamma(t)^2}{4\sqrt{t}} \leq \sqrt{T}\gamma(T)^2. \label{eq: NeuralRBMLE-PC regret 22} 
\end{align}
Notice that when $m \geq T^{10}\lambda^{-7}L^{24}(\log{m})^3$, we have $\gamma(T) = \mathcal{O}\left(\sqrt{\widetilde{d}\log(1+TK/\lambda)}\right)$. For $T(D_1 + D_2)$, we have 
\begin{align}
    T(D_1+D_2) = & \; T \left( E_{\text{PC},1}t^{\frac{2}{3}}m^{-\frac{1}{6}}\lambda^{-\frac{2}{3}}L^\frac{7}{2}\sqrt{\log{m}} + 2Sm^{-\frac{1}{6}}\sqrt{\log{m}}t^{\frac{1}{6}}\lambda^{-\frac{1}{6}}L^{\frac{7}{2}}\right) \\
    \leq & \; E_{\text{PC},1}T^{\frac{5}{3}}m^{-\frac{1}{6}}\lambda^{-\frac{2}{3}}L^\frac{7}{2}\sqrt{\log{m}} + 2Sm^{-\frac{1}{6}}\sqrt{\log{m}}T^{\frac{7}{6}}\lambda^{-\frac{1}{6}}L^{\frac{7}{2}}
\end{align}
Similarly, $T(D_1+D_2)$ will not affect the regret bound under the choice of $m$. Then, substituting (\ref{eq: NeuralRBMLE-PC regret 19}), (\ref{eq: NeuralRBMLE-PC regret 21}) and (\ref{eq: NeuralRBMLE-PC regret 22}) into (\ref{eq: NeuralRBMLE-PC regret 17}), we have $\sum_{t=1}^{T} h(\bx^*_t) - h(\bx_{t}) = \mathcal{O}(\widetilde{d}\sqrt{T}\log{T})$. 
\section{Index Derivations of the NeuralRBMLE Algorithms}
\label{appendix:E}
\subsection{Equivalence Between NeuralRBMLE in (\ref{eq:original theta_t of NeuralRBMLE}) and the Index Strategy in (\ref{def:RBMLE arm-specific})-(\ref{eq:index})}
\label{appendix:justification}
Recall (\ref{eq:original theta_t of NeuralRBMLE}) that $\btheta^{\dagger}_t$ denotes a maximizer of the following problem: 
\begin{equation}
     \max_{\btheta}\Big\{ \ell^{\dagger}(\mathcal{F}_t;\btheta)+\alpha(t)\cdot \max_{1 \leq a \leq K} f(\bx_{t,a};\btheta)-\frac{m\lambda}{2}{\left\Vert \btheta - \btheta_0 \right\Vert}^2_2 \Big\}. \label{eq:tilde Theta t}
\end{equation}
Define
\begin{align}
    \bar{\mathcal{A}}_t&:=\argmax_{a}\hspace{2pt}f(\bx_{t,a};\btheta^{\dagger}_t),\label{eq:tilde A_t}\\
    \btheta^{\dagger}_{t,a}&:=\argmax_{\btheta} \Big\{ \ell^{\dagger}(\mathcal{F}_t;\btheta)+\alpha(t)\cdot f(\bx_{t,a};\btheta)-\frac{\lambda}{2}{\left\Vert \btheta -\btheta_0 \right\Vert}^2_2 \Big\}.\label{eq:tilde Theta t,a}
\end{align}
For each arm $a$, consider an estimator $\btheta^{\dagger}_{t,a}\in \Theta^{\dagger}_{t,a}$. Subsequently, define an index set
\begin{equation}
    \bar{\mathcal{A}}'_t:=\argmax_{1\leq a\leq K} \Big\{ \ell^{\dagger}(\mathcal{F}_t;\btheta^{\dagger}_{t,a})+\alpha(t)\cdot f(\bx_{t,a};\btheta^{\dagger}_{t,a}) -\frac{\lambda}{2}{\left\Vert\btheta^{\dagger}_{t,a} - \btheta_0 \right\Vert}^2_2\Big\}.\label{eq:tilde A_t'}
\end{equation}
\begin{lemma}\label{lemma:Index Derivations of NeuralRBMLE} For all $t \in [T]$, we have 
    \begin{align}
        \bar{\mathcal{A}}_t = \bar{\mathcal{A}}'_t. \label{eq:lemma:Index Derivations of NeuralRBMLE eq-1}
    \end{align}
\begin{proof}
Substituting $f(\bx_{t,a};\btheta^{\dagger}_t)$ for $\bar{\theta}_{t}^{\intercal}x_{t,a}$ and $f(\bx_{t,a};\btheta^{\dagger}_{t,a})$ for $\bar{\theta}_{t,a}^{\intercal}x_{t,a}$, we can obtain (\ref{eq:lemma:Index Derivations of NeuralRBMLE eq-1}) by reusing the same analysis as that of Theorem 3 in \cite{hung2020reward}, 
\end{proof}
\end{lemma}

\subsection{Derivation of the Surrogate Index of NeuralRBMLE-PC}
\label{appendix:index of NeuralRBMLE-PC}
Recall that $\widehat{\btheta}_t$ is the least-squares estimate.
\begin{lemma}
\label{lemma:Surrogate Index of NeuralRBMLE-PC}
For all $a\in [K]$, all $t\in[T]$, any $\btheta$ that satisfies $\lVert \btheta - \btheta_0 \rVert_2 \leq \tau$, and $m$ that satisfies \begin{align}
    m \geq 2(\max\{E_{\ref{appendix_lemma_1}},E_{\ref{appendix_lemma_2}}\})^6T^{16}L^{21}\lambda^4(\log{m})^3, \label{eq:Surrogate Index of NeuralRBMLE-PC eq-1}
\end{align}
we have
\begin{align}
    f(\bx_{t,a};\btheta) - \langle \bg(\bx_{t,a};\widehat{\btheta}_t),\btheta - \btheta_0\rangle \leq \frac{1}{T^2}.\label{eq:Surrogate Index of NeuralRBMLE-PC eq-2}
\end{align}
\begin{proof}
By Lemma \ref{appendix_lemma_1}, we have
\begin{align}
    f(\bx_{t,a};\btheta) - \langle \bg(\bx_{t,a};\widehat{\btheta}_t),\btheta - \btheta_0\rangle & \leq \langle \bg(\bx_{t,a};\btheta_0),\btheta - \btheta_0\rangle - \langle \bg(\bx_{t,a};\widehat{\btheta}_t),\btheta - \btheta_0\rangle + E_{\ref{appendix_lemma_1}}\tau^{\frac{4}{3}}L^3\sqrt{m\log(m)} \label{eq:Surrogate Index of NeuralRBMLE-PC eq-3}\\
    & \leq E_{\ref{appendix_lemma_2}}\tau^\frac{4}{3}L^{\frac{7}{2}}\sqrt{m\log m} + E_{\ref{appendix_lemma_1}}\tau^{\frac{4}{3}}L^3\sqrt{m\log{m}}.\label{eq:Surrogate Index of NeuralRBMLE-PC eq-4} \\
    & \leq \frac{1}{T^2}\label{eq:Surrogate Index of NeuralRBMLE-PC eq-5},
\end{align}
where (\ref{eq:Surrogate Index of NeuralRBMLE-PC eq-4}) holds by Lemma \ref{appendix_lemma_2} and Lemma \ref{lemma:gradient of f}, and (\ref{eq:Surrogate Index of NeuralRBMLE-PC eq-5}) holds by (\ref{eq:Surrogate Index of NeuralRBMLE-PC eq-1}). Then, we complete the proof.
\end{proof}
\end{lemma}

Lemma \ref{lemma:Surrogate Index of NeuralRBMLE-PC} shows that we can approximate the forward value of the neural network: $f(\bx;\btheta)$ by the inner product of $\btheta-\btheta_0$ and $\bg(\bx;\widehat{\btheta}_t)$, which is the gradient in the $\hat{\theta_t}$-induced NTK regime. Then, we take the following steps to approximate $\mathcal{I}^{\dagger}_{t,a}$, the true NeuralRBMLE index, by the surrogate index of NeuralRBMLE-PC. Define
\begin{align}
    \widehat{\btheta}'_{t}  := \btheta_0 + \bigg( \frac{1}{m} \sum_{s=1}^{t-1}\bg(\bx_s;\widehat{\btheta}_t)\bg(\bx_s;\widehat{\btheta}_t)^{\intercal} +\lambda\bI\bigg)^{-1}\cdot\bigg(\frac{1}{m}\sum_{s=1}^{t-1}r_s\bg(\bx_s;\widehat{\btheta}_t)\bigg) \label{eq:Surrogate Index of NeuralRBMLE-PC-2 eq-0}
\end{align}
as the approximation of $\widehat{\btheta}_t$ in the $\widehat{\btheta}_{t}$-induced NTK regime. Recall that the true index of NeuralRBMLE is
\begin{align}
    \mathcal{I}^{\dagger}_{t,a} = \ell_{\lambda}^{\dagger}(\mathcal{F}_t;\btheta_{t,a}^{\dagger}) + \alpha(t)f(\bx_{t,a};\btheta_{t,a}^{\dagger}). \label{eq:Surrogate Index of NeuralRBMLE-PC-2 eq-1}
\end{align}
By Lemma \ref{lemma:Surrogate Index of NeuralRBMLE-PC}, we can use $\langle \bg(\bx;\widehat{\btheta}_t), \btheta - \btheta_0 \rangle$ to replace $f(\bx;\btheta)$ in $\mathcal{I}^{\dagger}_{t,a}$. Then, we can obtain a surrogate of $\Bar{\btheta}_{t,a}$ defined in (\ref{def:theta_bar}) as
\begin{align}
    \Bar{\btheta}'_{t,a} = \argmax_{\btheta}\underbrace{\left\{-\frac{1}{2}\sum_{s=1}^{t-1}\Big( \langle \bg(\bx_s;\widehat{\btheta}_t), \btheta - \btheta_0 \rangle -r_s \Big)^2 - \frac{m\lambda}{2}\left\Vert \btheta - \btheta_0 \right\Vert_2^2 + \alpha(t)\langle \bg(\bx_{t,a};\widehat{\btheta}_t), \btheta - \btheta_0 \rangle \right\}}_{:= \widehat{\mathcal{I}}_{t,a}(\btheta)}. \label{eq:Surrogate Index of NeuralRBMLE-PC-2 eq-2}
\end{align}
By the first-order
necessary optimality condition of $\Bar{\btheta}'_{t,a}$, we have 
\begin{align}
    \Bar{\btheta}'_{t,a}  & = \btheta_0 + \bigg( \frac{1}{m} \sum_{s=1}^{t-1}\bg(\bx_s;\widehat{\btheta}_t)\bg(\bx_s;\widehat{\btheta}_t)^{\intercal} +\lambda\bI\bigg)^{-1}\cdot \bigg(\frac{1}{m}\sum_{s=1}^{t-1}r_s\bg(\bx_s;\widehat{\btheta}_t) + \frac{\alpha(t)}{m}\bg(\bx_{t,a};\widehat{\btheta}_t)\bigg)  \label{eq:Surrogate Index of NeuralRBMLE-PC-2 eq-3} \\
    & = \widehat{\btheta}'_t + \frac{\alpha(t)}{m}\bZ^{-1}_{t-1}\bg(\bx_{t,a};\widehat{\btheta}_t) \label{eq:Surrogate Index of NeuralRBMLE-PC-2 eq-4} \\
    & \approx \widehat{\btheta}_t + \frac{\alpha(t)}{m}\bZ^{-1}_{t-1}\bg(\bx_{t,a};\widehat{\btheta}_t) \label{eq:Surrogate Index of NeuralRBMLE-PC-2 eq-5} \\
    & = \bar{\btheta}_{t,a}, \label{eq:Surrogate Index of NeuralRBMLE-PC-2 eq-6}
\end{align}
where (\ref{eq:Surrogate Index of NeuralRBMLE-PC-2 eq-4}) holds by the definition of $\widehat{\btheta}'_t$ in (\ref{eq:Surrogate Index of NeuralRBMLE-PC-2 eq-0}), (\ref{eq:Surrogate Index of NeuralRBMLE-PC-2 eq-5}) holds by $\widehat{\btheta}'_t \approx \widehat{\btheta}_t$ in $\widehat{\btheta}_t$-induced NTK regime, and (\ref{eq:Surrogate Index of NeuralRBMLE-PC-2 eq-6}) holds due to (\ref{def:theta_bar}). Then, we have obtained the surrogate arm-specific 
RBMLE estimators $\Bar{\btheta}_{t,a}$ for Algorithm \ref{alg:NeuralRBMLE-PC}.

\section{Detailed Configuration of Experiments}
\label{appendix:hyp}

To construct bandit problems from these datasets, we follow the same procedure as in \cite{zhou2020neural,zhang2020neural} by converting classification problems to $K$-armed bandit problems with general reward functions.
Specifically, we first convert each input feature $\bx \in \mathbb{R}^{d'}$ into $K$ different context vectors, where $\bx_{t,i} = (\mathbf{0}^{d'i} ,\bx,\mathbf{0}^{d'(K-i-1)}) \in \mathbb{R}^{d'K} \equiv \mathbb{R}^d$ for $i \in [K]$. 
The learner receives a unit reward if the context is classified correctly, and receives zero reward otherwise. 
The benchmark methods considered in our experiments include the two state-of-the-art neural contextual bandit algorithms, namely NeuralTS \cite{zhang2020neural} and NeuralUCB \cite{zhou2020neural}.
Furthermore, we also provide the results of multiple benchmark methods that enforce exploration directly in the parameter space, including BootstrappedNN \citep{osband2016deep}, the neural variant of random exploration through perturbation (DeepFPL) in \citep{kveton2020randomized}, and RBMLE for linear bandits (LinRBMLE) in \citep{hung2020reward}.
For all the neural-network-based algorithms, we use the neural network with $1$ hidden layer and width $m = 100$, set $J = 100$ as the maximum number of steps for gradient decent, choose learning rate $\eta = 0.001$, and update the parameters at each time step.
Each trial continues for $T=15000$ steps except that we set $T = 8000$ in the experiments of Mushroom as the number of data samples in Mushroom is $8124$.
For $\bZ_{t}$ in NeuralRBMLE-PC, NeuralTS, NeuralUCB, we use the inverse of the diagonal elements of $\bZ_t$ to be the surrogate of $\bZ_t^{-1}$ to speed up the experiments. 
Notice that this procedure is also adopted by \citep{zhou2020neural,zhang2020neural}.
To ensure a fair comparison among the algorithms, the hyper-parameters of each algorithm are tuned as follows: 
For $\nu$ relative to exploration ratio, NeuralTS, NeuralUCB, NeuralRBMLE, DeepFPL and LinRBMLE, we do grid search on $\{1, 10^{-1},10^{-3},10^{-5}\}$. 
For BootstrappedNN, we refer to \cite{zhou2020neural} and then set the transition probability as $0.8$, and we use $10$ neural networks to estimate the unknown reward function.
For $\lambda$ used in ridge regression, we set $\lambda = 1$ for LinRBMLE as suggested by \cite{hung2020reward}, and set $\lambda = 0.001$ for NeuralTS, NeuralUCB, and NeuralRBMLE.
We choose $\alpha(t) = \nu\sqrt{t}$ in both NeuralRBMLE-GA, NeuralRBMLE-PC and LinRBMLE, where $\nu$ is tuned in the same way as other benchmarks.
For the benchmark methods, we leverage the open-source implementation provided by NeuralTS \citep{zhang2020neural} available at \url{https://openreview.net/forum?id=tkAtoZkcUnm}.

{\bf Computing Resources.} All the simulation results are run on a Linux server with (i) an
Intel Xeon Gold 6136 CPU operating at a maxinum clock rate of 3.7 GHz, (ii) a total of 256 GB
memory, and (iii) an RTX 3090 GPU.
\section{Related Work on Linear Contextual Bandits}
\label{appendix:related}
Linear stochastic contextual bandit problems have been extensively studied in the literature. 
For example, the celebrated Upper Confidence Bound (UCB) method and its variants enforce exploration through constructing confidence sets and have been applied to both linear bandits
\citep{auer2002using,dani2008stochastic,rusmevichientong2010linearly,abbasi2011improved,chu2011contextual} and generalized linear bandits \cite{filippi2010parametric,li2017provably,jun2017scalable} with provably optimal regret bounds.
Another popular line of exploration techniques is using randomized exploration.
For example, from a Bayesian perspective, Thompson sampling (TS) achieves exploration in linear bandits \cite{agrawal2013thompson,russo2016information,abeille2017linear} and generalized linear bandits \cite{dumitrascu2018pg,kveton2020randomized} by sampling the parameter from the posterior distribution.
\cite{kveton2020randomized} proposed a follow-the-perturbed-leader algorithm for generalized linear bandits to achieve efficient exploration through perturbed rewards.
Another popular randomized approach is information-directed sampling \cite{russo2018learning,kirschner2018information}, which determine the action sampling distribution by maximizing the ratio between the squared expected regret and the information gain.
The above list of works is by no means exhaustive and is only meant to provide an overview of the research progress in this domain. 
While the above prior studies offer useful insights into efficient exploration, they share the common limitation that the reward functions are required to satisfy the linear realizability assumption.
\section{Additional Experimental Results}
\label{section:additional experiments}
\begin{table*}[ht]
\caption{Mean final cumulative regret over 10 trials for Figure \ref{fig:real_regret} and Figure \ref{fig:real_regret_2}. The best of each column is highlighted. Algorithms within $1\%$ performance difference are regarded as equally competitive.}
\vspace{-3mm}
\begin{center}
\begin{tabular}{|c|c|c|c|c|c|c|}
\hline
\textbf{Mean Final Regret} & \textbf{Adult} & \textbf{Covertype} & \textbf{MagicTelescope} & \textbf{MNIST} & \textbf{Mushroom} & \textbf{Shuttle} \\ \hline
\textbf{NeuralRBMLE-GA} & \underline{\textbf{3125.2}} & \underline{\textbf{4328.1}} & \underline{\textbf{3342.9}} & \underline{\textbf{1557.8}} & \underline{\textbf{313.1}} & \underline{\textbf{335.2}} \\ \hline
\textbf{NeuralRBMLE-PC} & 3200.6 & 4824.9 & 3630.2 & 2197.7 & 580.4 & 817.7 \\ \hline
\textbf{NeuralUCB} & 3154.9 & 4697.7 & 3555.2 & 2079.5 & 583.8 & 555.0 \\ \hline
\textbf{NeuralTS} & 3179.0 & 4813.3 & 3567.1 & 2585.7 & 594.5 & 577.7 \\ \hline
\textbf{DeepFPL} & 3174.9 & 5207.3 & 3559.2 & 4473.0 & 600.1 & 1622.9 \\ \hline
\textbf{BoostrappedNN} & 3351.4 & 5265.3 & 3929.7 & 6437.1 & 805.1 & 1323.9 \\ \hline
\textbf{LinRBMLE} & \underline{\textbf{3121.9}} & 4929.8 & 3885.0 & 11851.9 & 656.0 & 3193.7 \\ \hline
\end{tabular}	
\label{table:mean regret}
\end{center}
\end{table*}

\label{appendix:exp}
\begin{table*}[ht]
\begin{center}
\caption{Standard deviation of the final cumulative regret under the real datasets in Figure \ref{fig:real_regret} and Figure \ref{fig:real_regret_2}. The best is highlighted.}
\begin{tabular}{|c|c|c|c|c|c|c|}
\hline
\textbf{Std.of Final Regret} & \textbf{Adult} & \textbf{Covertype} & \textbf{MagicTelescope} & \textbf{MNIST} & \textbf{Mushroom} & \textbf{Shuttle} \\ \hline
\textbf{NeuralRBMLE-GA} & 28.05 & \underline{\textbf{50.97}} & 32.91 & \underline{\textbf{29.80}} & \underline{\textbf{19.60}} & \underline{\textbf{11.50}} \\ \hline
\textbf{NeuralRBMLE-PC} & 115.96 & 88.91 & 137.96 & 82.28 & 57.91 & 48.37 \\ \hline
\textbf{NeuralUCB} & 42.21 & 62.67 & 97.44 & 86.48 & 40.74 & 79.00 \\ \hline
\textbf{NeuralTS} & 36.67 & 87.34 & 125.69 & 179.94 & 44.46 & 75.92 \\ \hline
\textbf{DeepFPL} & 39.58 & 536.48 & 93.77 & 1255.68 & 43.25 & 824.98 \\ \hline
\textbf{BoostrappedNN} & 133.52 & 126.23 & 55.25 & 630.02 & 64.31 & 311.51 \\ \hline
\textbf{LinRBMLE} & \underline{\textbf{24.97}} & 99.99 & \underline{\textbf{27.20}} & 279.36 & 20.23 & 71.27 \\ \hline
\end{tabular}
\label{table:std regret}
\end{center}
\end{table*}

\begin{figure*}[ht]
$\begin{array}{c c c}
    \multicolumn{1}{l}{\mbox{\bf }} & \multicolumn{1}{l}{\mbox{\bf }} & \multicolumn{1}{l}{\mbox{\bf }} \\ 
    \scalebox{0.33}{\includegraphics[width=\textwidth]{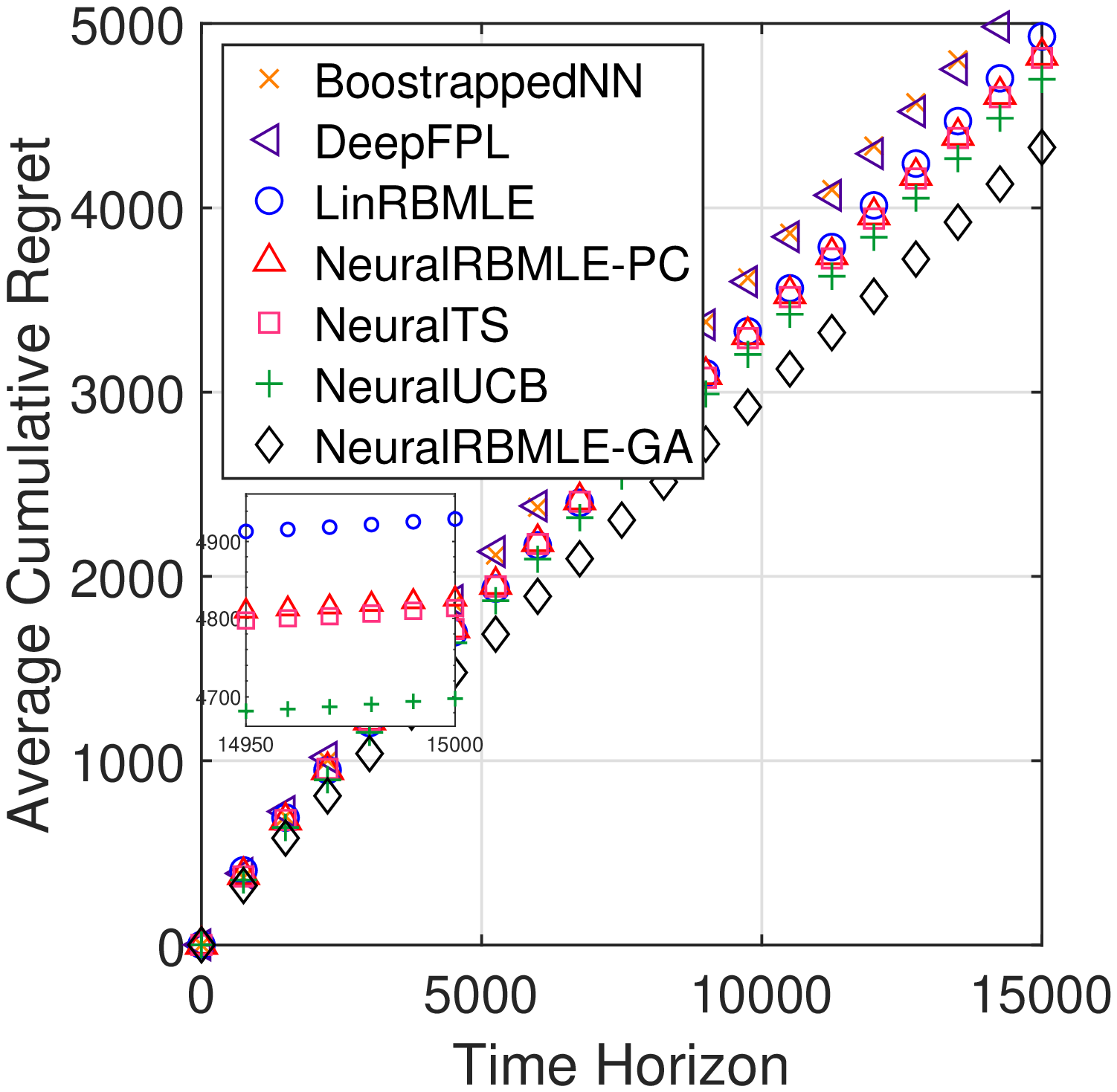}} & \hspace{-3mm} \scalebox{0.33}{\includegraphics[width=\textwidth]{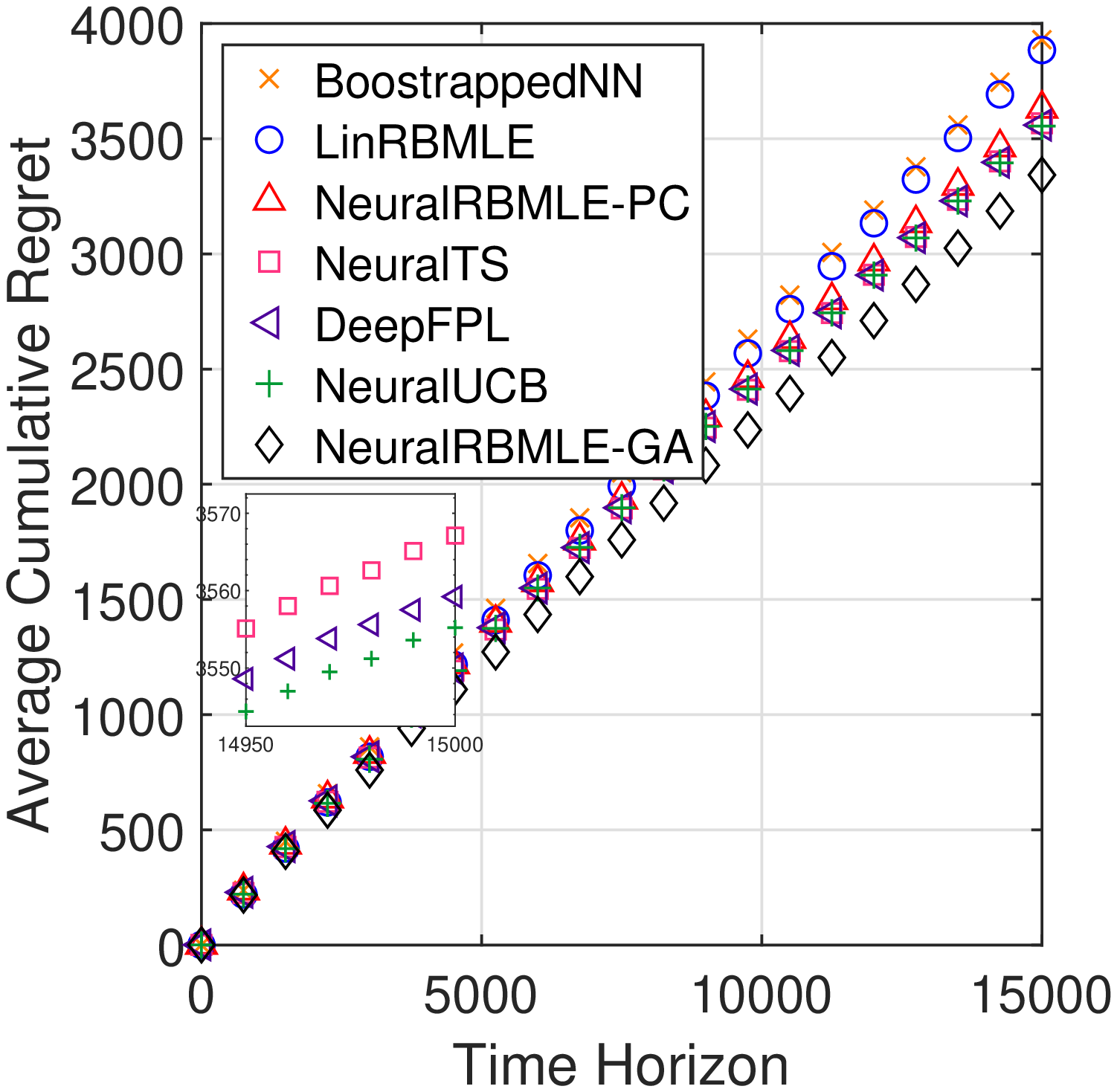}}  &  \scalebox{0.33}{\includegraphics[width=\textwidth]{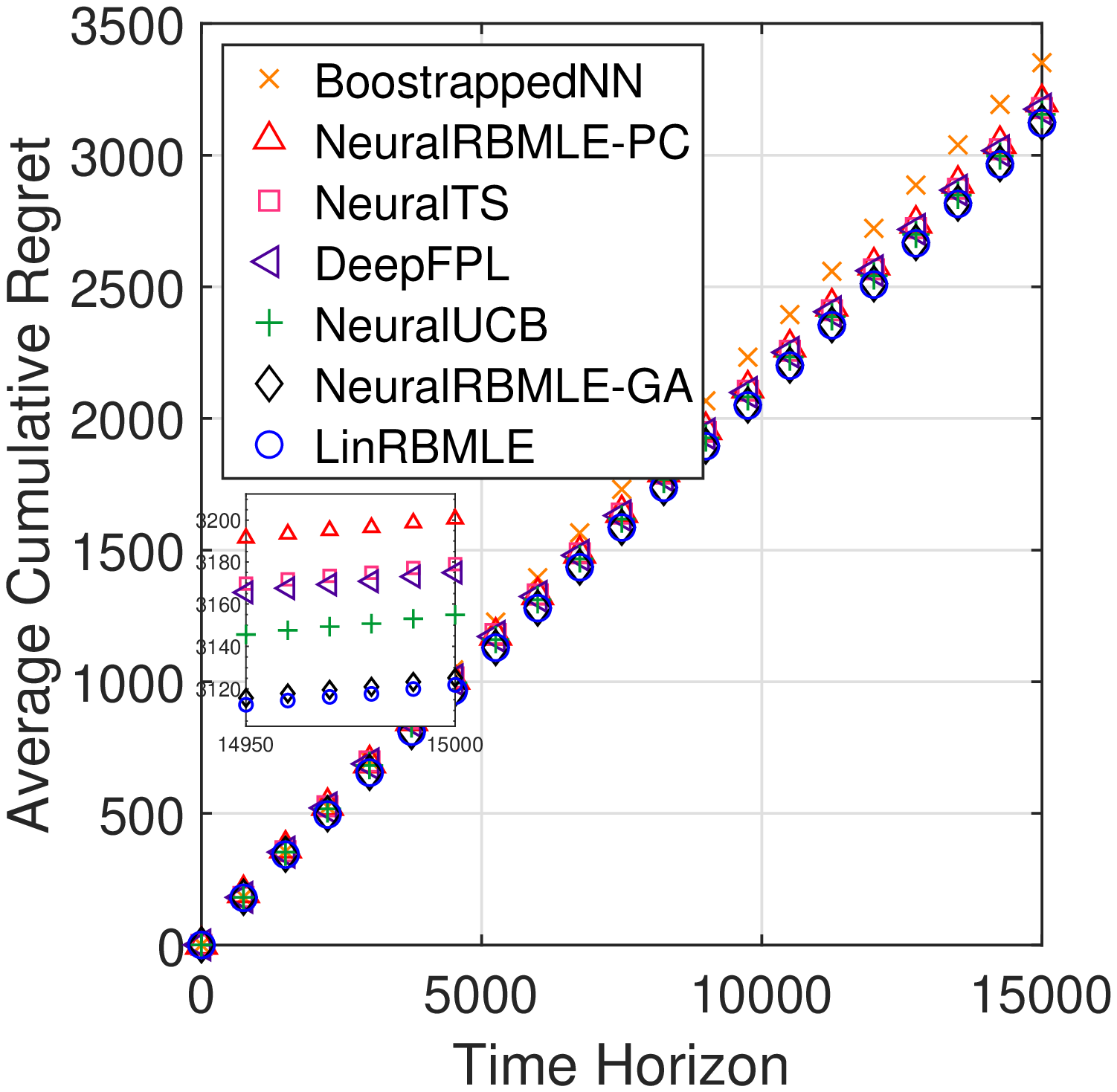}}  \\
    \mbox{(a) Covertype} &  \mbox{(b) MagicTelescope} &  \mbox{(c) Adult} 
\end{array}$
\caption{Cumulative regret averaged over $10$ trials with $T = 1.5\times 10^4$.}
\label{fig:real_regret_2}
\end{figure*}

\begin{table}[!ht]
\centering
\hungyh{
\caption{Computation time per step of NeuralRBMLE-GA, NeuralUCB and NeuralTS. We use $t_{\text{step}}$ to denote the total computation time of selecting action and traing, and $t_{\bZ}$ to denote the computation time of computing the inverse of $\bZ$.}
\begin{tabular}{cccc}
\hline
& NeuralRBMLE-GA  & NeuralUCB & NeuralTS \\ \hline
\begin{tabular}[c]{@{}l@{}} $t_{\text{step}}$ \end{tabular} &  \textbf{1.6688}s & 9.9027s   & 9.7390s   \\ \hline
$t_{\bZ}$ & N/A & 9.6974s    & 9.5406s   \\ \hline
\end{tabular}
\label{table:computation time}
}
\end{table}
\end{document}